\documentclass[10pt]{article}
\usepackage{enumerate}
\usepackage[OT1]{fontenc}
\usepackage{amsmath,amssymb}
\usepackage{natbib}
\usepackage[usenames,dvipsnames]{xcolor}
\usepackage{geometry}
\usepackage{dsfont}
\usepackage{thm-restate}
\usepackage{pgfplots}
\usepackage{smile_new}
\usepackage{multirow}
\usepackage{rotating}
\usepackage{enumerate}
\usepackage{esvect}
\usepackage{tikz}
\usetikzlibrary{patterns}
\usetikzlibrary{arrows}
\usepackage{hyperref}
\usepackage{algorithm}
\usepackage{algpseudocode}
\usepackage{enumitem}
\usepackage{acronym}
\newcommand{\ouralgfull}{\textsc{Posterior Gap Weighted Sampling with Abstention}}
\newcommand{\ouralg}{{\textnormal{\textsc{PGWS}}}}
\newcommand{\dd}{\mathrm{d}}

\usepackage[most,skins]{tcolorbox}
\definecolor{rliableolive}{HTML}{BBCC33}
\definecolor{rliableblue}{HTML}{77AADD}
\definecolor{rliablered}{HTML}{EE8866}
\definecolor{LightCyan}{rgb}{0.88,1,1}
\definecolor{darkblue}{HTML}{2878D9}
\definecolor{navyblue}{HTML}{0000FF}
\usepackage{mathtools}

\theoremstyle{plain}

\usepackage{mathrsfs}
\usepackage{fullpage}

\usepackage[protrusion=false,expansion=false]{microtype}
\usepackage{bbm}

\usepackage{makecell}
\usepackage{float}        
\usepackage{tabularx}     
\usepackage[table]{xcolor} 
\usepackage{multirow}
\usepackage{cleveref}
\usepackage{subcaption}
\usepackage{xurl}

\crefname{lemma}{lemma}{lemmas}
\Crefname{lemma}{Lemma}{Lemmas}
\crefname{definition}{definition}{definitions}
\Crefname{definition}{Definition}{Definitions}

\title{Bayesian Best-Arm Identification with Abstention: A Polynomial-to-Exponential Phase Transition}

\author{Yuqi Huang\textsuperscript{1}\quad
Yunlong Hou\textsuperscript{1}\quad
Vincent Y.~F.~Tan\textsuperscript{1,2} \\ \\
\textsuperscript{1}Department of Mathematics, National University of Singapore\\  \textsuperscript{2}Department of Electrical and Computer Engineering, National University of Singapore
}
\date{}
\begin{document}
\maketitle

\begin{abstract}
We study the Bayesian fixed-budget best-arm identification problem in which a learner can abstain from making a terminal recommendation. Subject to an abstention budget $\alpha$, we analyze the probability of undetected error--the risk of recommending a suboptimal arm without abstaining.
Our central finding is that abstention induces a phase transition: without abstention, the error probability decays polynomially in the sampling budget $T$; in contrast, introducing any small positive abstention budget shifts this to an exponential decay.
For Gaussian priors and rewards, in the regime $T\to\infty$ followed by $\alpha\downarrow0$, we establish exact matching information-theoretic lower bounds and algorithmic upper bounds on the optimal error exponent, which takes the form $\exp(-\frac{\alpha^{2}T}{8\kappa_{\nu}^{2}})$. The hardness parameter $\kappa_{\nu}$ represents the prior density of the top-two gap at zero, highlighting that nearly tied instances drive the fundamental error. We introduce an adaptive algorithm, \textsc{PGWS}, that successfully achieves this optimal exponent by expending its abstention budget on statistically ambiguous instances.
We further demonstrate that this polynomial-to-exponential improvement is exclusively a Bayesian phenomenon--in the frequentist setting, abstention only affects lower-order exponent terms. We also extend our results beyond the Gaussian model.
\end{abstract}

\section{Introduction}\label{sec:Intro}

\subsection{Problem Overview and Motivation}
\label{subsec:problem-overview}

\emph{Fixed-Budget Best-Arm Identification} (FB-BAI) \citep{audibert2010best,karnin2013almost, carpentier2016tight} is a canonical pure exploration problem in
sequential decision making. A learner is given a  sampling budget, adaptively
collects noisy observations from a set of  $K$ arms with unknown mean rewards, and
 finally   recommends the arm believed to have the largest mean. This formulation
captures the statistical core of ranking and selection, 
online experimentation, and treatment-selection problems, where samples are
costly and the objective is to make a reliable terminal recommendation rather
than to maximize reward during the experiment~\citep{audibert2010best,kaufmann2016complexity,russo2020simple}.

Typical FB-BAI formulations require a firm terminal recommendation after the
sampling budget is exhausted. In many applications, however, this  forced decision, in which the decision maker commits to one of the $K$ arms, may pose undue  risks. When the collected evidence is ambiguous, the decision maker may
defer the decision, run a follow-up experiment, escalate the case to an expert,
or choose a conservative default action. This motivates a FB-BAI model
with terminal {\em abstention} option: after sampling, the learner may either recommend an arm
or output an inconclusive decision.

Abstention changes the notion of error. A policy that always abstains
never makes an incorrect recommendation, but it is not useful. We therefore
treat abstention as a scarce resource: the learner must keep the probability of
abstention below a prescribed level \(\alpha\). Subject to this constraint, the
performance criterion is the probability of an \emph{undetected error}, namely the probability that the learner does not abstain but  recommends a suboptimal arm. Informally, let $a^\star$ denote the best arm, $?$ be the abstention decision, and
$\hat a\in [K]\cup\{?\}$ be the terminal output. The objective is to minimize
\begin{equation}
     \mathbb P\big(\hat a\notin\{a^\star,?\}\big),
     \quad
     \textrm{subject to} \quad \mathbb P(\hat a=?)\le \alpha .
\end{equation}

This paper studies the problem in a Bayesian fixed-budget setting, where the
unknown reward means are drawn from known priors. The Bayesian formulation is
natural when the same experimental protocol is repeatedly applied to related
instances, or when historical data provide a prior model for plausible reward
configurations. It is also the setting in which the effect of abstention is particularly
pronounced. Under a continuous prior, the Bayes error probability of forced-decision (i.e., without abstention) FB-BAI is
dominated by instances in which the two best arms are nearly tied. These instances
are rare, but they are also the hardest to resolve with a finite sampling budget.

The central observation of this paper is that a small abstention budget allows the learner to spend its inconclusive decisions precisely on these statistically ambiguous near-tie cases. We show that this fundamentally results in a phase transition of  the  behavior of Bayesian
BAI: even a small positive abstention budget  substantially reduces the forced-decision
{\em polynomially} small  Bayes error  probability into an {\em exponentially}  small  Bayes undetected error probability.
We derive this  exponent and show that it is governed by the amount of prior probability near top-two ties. 
The rest of the paper formalizes this principle, develops an adaptive algorithm
that attains the optimal exponent, and contrasts the Bayesian phenomenon with the frequentist setting.

\subsection{Problem Formulation and Top-Two Gap Complexity}
\label{subsec:model-objective}

We now introduce the formal model. For an integer $K\ge 2$, let $[K]:=\{1,\ldots,K\}$. The learner is given a fixed sampling budget $T$ and a collection of $K$ arms.
Arm $i\in[K]$ has an unknown mean reward $\mu_i$. We study a Bayesian model in which the mean vector $\mu=(\mu_1,\ldots,\mu_K)$ is drawn from a known product Gaussian prior
\[
    P_\nu(\dd\mu)
    :=
    \bigotimes_{i=1}^K
    \mathcal N(\nu_i,\sigma_0^2)(\dd\mu_i),
\]
where the prior means $\nu_1,\ldots,\nu_K\in\mathbb R$ and prior variance (which is common to all arms)
$\sigma_0^2>0$ are fixed and known. 
 We define \emph{the best arm} $a^\star=a^\star(\mu):=\argmax_{i\in[K]}\mu_i$ which is unique almost surely since the   prior $P_\nu$ is continuous.
Conditional on $\mu$, the rewards from the arms are independent (across arms and time) and normally distributed:
\[
    Y_{i,s}\mid \mu_i \sim \mathcal N(\mu_i,1),
    \qquad s \in \mathbb{N},~i \in [K].\]

\paragraph{Policies and Probability Laws.}

A {\em policy} $\pi$ consists of a sequence of sampling rules and a terminal
decision rule. At each round $t=1,\ldots, T$, the policy selects an arm $A_t \in [K]$. Conditional on  $A_t=i$ and the mean vector $\mu$, the learner observes a random sample $X_t=Y_{i, N_i(t)}$ where  $N_i(t):=\sum_{s=1}^t \mathbbm{1}\{A_s=i\}$
denotes the number of samples collected from arm $i$ up to and including time $t$. The choice of the arm to be pulled $A_t$ depends on the previous arm choices $\{A_s\}_{s=1}^{t-1}$ and observations $\{  X_s  \}_{s=1}^{t-1}$. At the terminal time $T$, the policy recommends an arm from $[K]$ as the estimated best arm or an abstention decision denoted as $?$.  The class of all such policies is denoted as $\Pi_T$.



For a fixed mean vector $\mu$, we write $\mathbb P_\mu^\pi$ and
$\mathbb E_\mu^\pi$ for the probability and expectation over the reward
realizations and the potential randomization of the policy. The joint Bayesian law is the
mixture
\[
    \mathbb P_\nu^\pi(\cdot)
    :=
    \int_{\mathbb R^K}
    \mathbb P_\mu^\pi(\cdot)\,P_\nu(\dd\mu),
\]
and $\mathbb E_\nu^\pi$ is defined analogously. When the policy (resp. probability law) is clear from context, we suppress the superscript $\pi$ (resp. subscript $\nu$).


\paragraph{Terminal Decisions with Abstention.}

At the end of the experiment, the learner outputs $\hat a_T\in[K]\cup\{?\}$,
where $?$ denotes abstention, an inconclusive decision. The learner makes an
\emph{undetected error} if it recommends an arm and the recommended arm is not 
the true best arm, i.e., the event
    $\big\{\hat a_T\notin\{a^\star,?\} \big\}
    =
     \big\{\hat a_T\in[K],\ \hat a_T\neq a^\star\big\}$ occurs.

For a policy $\pi\in\Pi_T$, 
define its {\em Bayes abstention probability} and {\em Bayes undetected error probability} respectively as 
\[
    \mathcal A_T(\pi)
    :=
    \mathbb P_\nu^\pi(\hat a_T=?)
    \quad\mbox{and}\quad
    \mathcal E_T(\pi)
    :=
    \mathbb P_\nu^\pi
    \bigl(\hat a_T\in[K],\ \hat a_T\neq a^\star(\mu)\bigr).
\]
Given an abstention budget $\alpha\in(0,1)$, let
\[
    \Pi_T(\alpha)
    :=
    \{\pi\in\Pi_T:\mathcal A_T(\pi)\le\alpha\}.
\]
The \emph{optimal Bayes undetected error probability} is
\begin{equation}
\label{eq:bayes-risk-abstention}
    \mathfrak E_{T}(\alpha)
    :=
    \inf_{\pi\in\Pi_T(\alpha)}
    \mathcal E_T(\pi).
\end{equation}
For later comparison, we also define the {\em optimal  forced-decision Bayes probability of
misidentification} (in which abstention is not allowed) as 
\begin{equation}
    \mathfrak E_{T}(0)
    :=
    \inf_{\pi\in\Pi_T:\,\hat a_T\in[K]\ {\rm a.s.}}
    \mathbb P_\nu^\pi(\hat a_T\neq a^\star)=  \inf_{\pi\in\Pi_T(0) }
    \mathbb P_\nu^\pi(\hat a_T\neq a^\star)\ . \label{eqn:CT0}
\end{equation}

\paragraph{Top-Two Gap Complexity.}
We define the {\em top-two gap}
\begin{equation}
    \Gamma:=\mu_{(1)}-\mu_{(2)},
\end{equation}
where $\mu_{(1)}>\mu_{(2)}>\cdots>\mu_{(K)}$ denote the  means of the arms arranged in decreasing order. Theorems~\ref{thm:pgws-upper} and~\ref{thm:bayesian-abstention-lower} identify the density of $\Gamma$ at zero as the fundamental hardness parameter for the Bayesian BAI problem with abstention.

For each arm $i$, we denote the prior density and distribution function of $\mu_i$ as
\[
    f_i(x)
    :=
    \frac{1}{\sigma_0}
    \phi\left(\frac{x-\nu_i}{\sigma_0}\right),
    \quad\mbox{and}\quad
    F_i(x)
    :=
    \Phi\left(\frac{x-\nu_i}{\sigma_0}\right),
\]
where $\phi(x)=\frac{1}{\sqrt{2\pi}}e^{-x^2/2}$ and $\Phi(x)=\int_{-\infty}^x \phi(u)\, \dd u$ are the standard normal density and cumulative distribution function, respectively.

For a  pair of distinct arms $(i,j) \in[K]^2$, define their {\em average midpoint variable} as
    $X_{ij}
    \sim
    \mathcal N\!\left(
        \frac{\nu_i+\nu_j}{2},
        \frac{\sigma_0^2}{2}
    \right)$,
and define the \emph{top-relevance weight}
    $w_{ij}
    :=
    \mathbb E\!\left[
        \prod_{k\neq i,j}
        \Phi\!\left(\frac{X_{ij}-\nu_k}{\sigma_0}\right)
    \right]$.
The quantity $w_{ij}$ is the prior probability, averaged over a local tie value
of arms $i$ and $j$, that this tied value lies above all remaining arms.
Define the \emph{hardness parameter} for Bayesian BAI with abstention problem as
\[
    \kappa_\nu
    =
    \sum_{i\neq j}
    \int_{\mathbb R}
    f_i(x)f_j(x)
    \prod_{k\neq i,j}F_k(x)\,\dd x 
    =
    \frac{1}{\sigma_0\sqrt\pi} \sum_{1\le i<j\le K}
    w_{ij}
    \exp\!\left(
        -\frac{(\nu_i-\nu_j)^2}{4\sigma_0^2}
    \right).
\]
In Lemma~\ref{lem:top-gap-density}, we show that the top-two gap $\Gamma$ is related to the   hardness parameter $\kappa_\nu$ as follows:
\begin{equation}
    \mathbb P_\nu(\Gamma\le \varepsilon) = \kappa_\nu\varepsilon + o(\varepsilon)\quad \mbox{as}\quad \varepsilon\downarrow 0. \label{eqn:kappa_nu}
\end{equation}
In view of \eqref{eqn:kappa_nu}, we see that $\kappa_\nu$ measures the amount of
prior probability assigned to locally ambiguous instances.
Specifically, pairs with similar prior means and a realistic chance of being the top two arms contribute more to $\kappa_\nu$, whereas pairs that are unlikely to be  optimal contribute little.

\subsection{Main Contributions and Results}
\label{subsec:main-results}

We formulate the fixed-budget BAI with abstention problem under the Bayesian setting.
We study this problem under the large-budget and small-abstention probability regime through the iterated limit $T\to\infty$ followed by $\alpha\downarrow0$ .

Our main result is a sharp characterization of the exponential rate of decay of the optimal Bayes
undetected error probability under abstention. In the Gaussian
prior and Gaussian reward model, we prove the information-theoretic lower bound
\[
    \liminf_{\alpha\downarrow0}
    \liminf_{T\to\infty}
    \frac{1}{\alpha^2T}
    \log \mathfrak E_{T}(\alpha)
    \ge
    - \frac{1}{8\kappa_\nu^2}.
\]

We also propose an adaptive algorithm, \ouralgfull{} (\ouralg($\alpha$)), whose
abstention probability is  calibrated exactly to $\alpha$ and whose
undetected error probability $\mathcal E_T(\ouralg(\alpha))$ satisfies
\[
    \limsup_{\alpha\downarrow0}
    \limsup_{T\to\infty}
    \frac{1}{\alpha^2T}
    \log \mathcal E_T(\ouralg(\alpha))
    \le
    - \frac{1}{8\kappa_\nu^2}.
\]
Combining these results, we see that the fundamental hardness of  the Bayesian BAI with abstention problem depends on the prior through $\kappa_\nu$, the density of top-two gap at zero. In other words, the  complexity is intimately related to the prior probability of instances in which the means of two best arms are nearly the same.

Subsequently, we compare the abstention problem with ordinary (forced-decision) Bayesian BAI, where the learner is compelled to recommend an arm. We establish that the optimal forced-decision Bayes probability of misidentification satisfies
\[
    \liminf_{T\to\infty}
    \sqrt T\,\mathfrak E_T(0)
    \ge
    \sqrt{\frac{2}{\pi}}\,\kappa_\nu .
\]
Therefore, allowing an $\alpha$ fraction of abstentions transits the Bayes probability of error from being on the
polynomial scale to the exponential scale.


We further demonstrate that this phase transition is inherently Bayesian. In contrast, for any fixed frequentist instance with a positive gap, the error probability of conventional forced-decision FB-BAI already exhibits exponential decay in $T$; see, e.g., \citet{karnin2013almost}. As a result, terminal abstention cannot improve the leading error exponent and affects only lower-order terms.

Finally, the analysis is extended beyond the Gaussian model. For regular
one-parameter reward models, the same exponent takes the form
\(1/(8\kappa^2)\), where \(\kappa\) is the top-two gap density evaluated in Fisher--Rao information coordinates; see Eqn.~\eqref{eqn:fisher_rao}.

We also provide numerical experiments illustrating the benefit of abstention and the practical advantage of our adaptive algorithm which we call \ouralg{}.

\subsection{Related work}
\label{subsec:related-work}

\paragraph{Best-arm Identification and Ranking and Selection.}
Best-arm identification is a central pure-exploration problem in multi-armed
bandits. In the fixed-confidence setting, the objective is to identify the best
arm with prescribed confidence using as few samples as possible 
\citep{banditAlg,even-dar2006action,kaufmann2016complexity}. In the fixed-budget setting,
which is our focus, the learner is given a sampling budget and aims to minimize
the probability of an incorrect terminal recommendation
\citep{audibert2010best,karnin2013almost,carpentier2016tight,
komiyama2022minimax,wang2023large}. For a fixed instance with a positive gap,
this probability typically decays exponentially in the budget, with the exponent
determined by instance-dependent complexity. Our work instead studies a Bayesian
average over instances and allows terminal abstention. The problem is also
related to ranking and selection and ordinal optimization, where the objective is
often to maximize the probability of correct selection under a simulation budget
\citep{chen2000simulation,glynn2004large,hong2021review}. In contrast to most of
that literature, our performance criterion is an abstention-constrained Bayes undetected error probability.

\paragraph{Bayesian Best Arm Identification.}
Bayesian BAI involves a prior distribution over the unknown reward means. Bayesian
algorithms such as Thompson sampling and top-two Thompson sampling have been
studied for posterior convergence and BAI
\citep{russo2020simple,shang2020fixed}. Recent work develops prior-sensitive
guarantees for Bayesian BAI: \citet{atsidakou2023bayesian} study Gaussian fixed-budget Bayesian BAI, while \citet{komiyama2024rate} characterize Bayesian simple-regret rates under continuity assumptions and identify near-tie instances as the leading source of error. Related work also studies prior-dependent
allocations in structured Bayesian fixed-budget BAI
\citep{nguyen2025prior} and fixed-confidence BAI in the Bayesian setting
\citep{jang2024fixed}. Our contribution is to add a
terminal abstention option and provide a full characterization of  the undetected error exponent
under an abstention budget.

\paragraph{Abstention and Reject-Option Decisions.}
Abstention has a long history in statistical decision theory and classification.
Classical reject-option rules trade prediction error against rejection \citep{chow1970optimum}, and selective prediction studies risk-coverage tradeoffs in supervised learning \citep{bartlett2008classification,cortes2016learning}.
These settings are typically passive, whereas our learner actively allocates a fixed sampling budget before deciding whether to recommend or abstain. Bandit models with abstention-like actions include selective sampling for online BAI \citep{camilleri2021selective} and regret minimization with per-round abstention \citep{yang2026abstention}. In contrast, our model permits abstention only at the terminal recommendation stage.

\paragraph{Erasure Decoding and Undetected Error.}
An essential motivation for considering this setup arises from erasure decoding in information theory. In
channel coding with an erasure option, the decoder may either output a message or declare an erasure; performance is then measured by the tradeoff between the total   and undetected error probabilities. 
This line of work dates back to Forney's exponential bounds for erasure, list, and decision-feedback decoding \citep{forney1968exponential}, with later refinements for erasure/list exponents,
exact random-coding exponents, and asymmetric erasure-error regimes
\citep{merhav2008error,somekhbaruch2011exact,hayashi2015asymmetric}. Our
abstention option plays an analogous role to erasure, but the BAI setting is an
active sequential experiment, in contrast to the non-sequential nature of channel coding (without feedback). 

\section{Algorithm and Error Analysis}

\subsection{The Algorithm  \ouralg{}}
\label{subsec:pgws-algorithm}

We now present \ouralgfull{} (\ouralg{}). This policy has two components, a sampling rule and a terminal decision rule. During the sampling phase,
it maintains   Gaussian posterior distributions of the arm means and 
samples arms according to posterior mean gaps. At the terminal time, it either
recommends the posterior mean leader or abstains, with the abstention rule
calibrated to the prescribed budget \(\alpha\).

Let \(\mathcal F_t:=\sigma(A_1,X_1,\ldots,A_t,X_t)\) denote the history up to time \(t\), with $\mathcal F_0:=\{\varnothing,\Omega\}$, augmented to include any external randomization used by the policy. Under the Gaussian prior and
Gaussian reward model, the posterior distributions after \(t\) samples are
independent across arms:
\[
    \mu_i\mid \mathcal F_t
    \sim
    \mathcal N\big(M_i(t),V_i(t)\big),
    \qquad i\in[K],
\]
where the posterior variance $V_i(t)$ and mean $M_i(t)$ are, respectively,
\begin{equation}
\label{eq:pgws-posterior-update}
    V_i(t)
    =
    \frac{1}{\sigma_0^{-2}+N_i(t)},
    \quad\mbox{and}\quad
    M_i(t)
    =
    V_i(t)
    \left(
        \frac{\nu_i}{\sigma_0^2}
        +
        \sum_{s=1}^{N_i(t)}Y_{i,s}
    \right).
\end{equation}

At each time \(t\), define the \emph{posterior mean leader}
    $\widehat b_t
    :=
    \arg\max_{i\in[K]} M_i(t)$.
For \(i\neq \widehat b_t\), define the
\emph{posterior mean gap}
    $\widehat\Delta_i(t)
    :=
    M_{\widehat b_t}(t)-M_i(t),$
     with $\widehat\Delta_{\widehat b_t}(t)
    :=
    \min_{j\neq \widehat b_t}
    \widehat\Delta_j(t)$.

\begin{algorithm}[t]
\caption{\ouralgfull{} (\ouralg{}$(\alpha)$)}
\label{alg:pgws}
\begin{algorithmic}[1]
\Require Budget \(T\), abstention level \(\alpha\), prior parameters
\((\nu_1,\ldots,\nu_K)\) and \(\sigma_0^2\).
\State Initialize \(N_i(0)=0\), \(M_i(0)=\nu_i\), and
\(V_i(0)=\sigma_0^2\) for all \(i\in[K]\).
\For{\(t=0,1,\ldots,T-1\)}
    \If{there exists \(i\in[K]\) such that \(N_i(t)<\sqrt{t+1}\)}
        \State Sample one such arm, using deterministic tie-breaking.
    \Else
        \State Compute \(\widehat b_t\), the gaps
        \(\{\widehat\Delta_i(t)\}_{i\in[K]}\), and the sampling probabilities
        \(\{p_i(t)\}_{i\in[K]}\) in \eqref{eq:pgws-sampling-prob}.
        \State Sample arm \(i\) with probability \(p_i(t)\).
    \EndIf
    \State Observe the reward and update \(N_i(t+1)\), \(M_i(t+1)\), and
    \(V_i(t+1)\) according to \eqref{eq:pgws-posterior-update}.
\EndFor
\State Set \(\widehat b_T=\arg\max_{i\in[K]}M_i(T)\).
\State Compute the terminal statistic \(R_T\) in
\eqref{eq:pgws-terminal-statistic}.
\State Output \(\hat a_T\) according to the abstention rule
\eqref{eq:pgws-terminal-rule}.
\end{algorithmic}
\end{algorithm}

The forced-exploration step in Lines 3--4 ensures that every arm is 
 sufficiently sampled to yield an accurate estimate of the mean. 
On all
remaining rounds, 
\ouralg{} samples arm \(i\) with
probability
\begin{equation}
\label{eq:pgws-sampling-prob}
    p_i(t)
    =
    \frac{\widehat\Delta_i(t)^{-2}}
    {\sum_{\ell=1}^K \widehat\Delta_\ell(t)^{-2}},
    \qquad i\in[K].
\end{equation}
Consequently, the leader and the closest
challenger receive equal sampling probability, while arms with larger
posterior mean gaps receive smaller probabilities.

At the end of the sampling phase, \ouralg{} measures the strength of the
posterior evidence in favor of the final leader \(\widehat b_T\). Define
\begin{equation}
\label{eq:pgws-terminal-statistic}
    R_T
    :=
    \min_{j\neq \widehat b_T}
    \frac{
        \bigl(M_{\widehat b_T}(T)-M_j(T)\bigr)^2
    }{
        2\bigl(V_{\widehat b_T}(T)+V_j(T)\bigr)
    },
\end{equation}
which represents the smallest  pairwise posterior Gaussian exponent
separating the final leader from its challengers. A small value of \(R_T\)
indicates that at least one challenger remains statistically difficult
to be distinguished from the leader.

To calibrate abstention, let \(r_{T,\alpha}\) be the lower
\(\alpha\)-quantile of \(R_T\) under the joint Bayesian law induced by the
sampling rule, i.e., 
\begin{equation}
\label{equ:terminal_threshold}
    r_{T,\alpha}
    :=
    \inf\bigl\{
        r\ge 0:
        \mathbb P(R_T\le r)\ge \alpha
    \bigr\}.
\end{equation}
If the distribution of \(R_T\) has an atom at the
threshold, define
\begin{equation}
\label{eq:pgws-boundary-randomization}
    \theta_{T,\alpha}
    :=
    \begin{cases}
    \displaystyle
    \frac{
        \alpha-\mathbb P(R_T<r_{T,\alpha})
    }{
        \mathbb P(R_T=r_{T,\alpha})
    },
    & \mathbb P(R_T=r_{T,\alpha})>0,\\ 
    0,
    & \mathbb P(R_T=r_{T,\alpha})=0.
    \end{cases}
\end{equation}
The terminal decision is
\begin{equation}
\label{eq:pgws-terminal-rule}
    \hat a_T
    =
    \begin{cases}
        ?, & R_T<r_{T,\alpha},\\
        ?, & R_T=r_{T,\alpha}
             \text{ with probability } \theta_{T,\alpha},\\
        \widehat b_T, & \text{otherwise}.
    \end{cases}
\end{equation}
Thus \ouralg{} abstains on the least confident posterior histories, up to boundary randomization which is used only to obtain
{\em exact} calibration when \(R_T\) has an atom at its \(\alpha\)-quantile.

\begin{remark}[Computing the quantile]
\label{rem:pgws-threshold}
The quantile \(r_{T,\alpha}\) is deterministic once the prior, the sampling
rule, \(T\), and \(\alpha\) are fixed. For the purpose of implementation, it can be estimated
offline by Monte Carlo simulation from the prior. A    
large-\(T\), small-\(\alpha\) approximation is
$
    r_{T,\alpha}^{\mathrm{asy}}
    =
    \frac{1}{8\kappa_\nu^2}\alpha^2T .
$
\end{remark}

\subsection{Upper Bound Analysis}
\label{subsec:pgws-upper-analysis}

We now derive theoretical guarantees for \ouralg{}. In
this subsection, all probabilities are under the joint Bayesian law induced by
Algorithm~\ref{alg:pgws}.

\begin{theorem}[Upper Bound for \ouralg{}]
\label{thm:pgws-upper}
For any $T$ and any $\alpha\in(0,1)$, \ouralg{}($\alpha$) with the terminal rule \eqref{eq:pgws-terminal-rule} satisfies
$\mathcal A_T(\ouralg(\alpha))=\alpha$.
Furthermore, asymptotically,
\begin{equation}
    \limsup_{\alpha\downarrow0}
    \limsup_{T\to\infty}
    \frac{1}{\alpha^2T}\log \mathcal E_T(\ouralg(\alpha))
    \le
    - \frac{1}{8\kappa_\nu^2}.\label{eqn:ub}
\end{equation}
\end{theorem}
The proof uses two key lemmas. In particular, the first lemma gives a finite-sample posterior
certification bound: conditional on the sampled arms and observed rewards up to time step $T$, the statistic $R_T$
controls the posterior probability that the final leader is not the true best arm.

\begin{lemma}[Posterior Certification Error Bound]
\label{lem:pgws-certification}
For every $T$ and every $\alpha\in(0,1)$,
\[
    \mathcal E_T(\ouralg(\alpha))
    \le
    (K-1)e^{-r_{T,\alpha}}.
\]
\end{lemma}

\begin{lemma}[Asymptotic Lower Bound on the Calibration Threshold]
\label{lem:pgws-threshold-lower}
For every $\varepsilon\in(0,1)$, there exists $\alpha_0>0$ such that for every
$0<\alpha\le \alpha_0$, there exists $T_0(\alpha,\varepsilon)$ such that for all
$T\ge T_0(\alpha,\varepsilon)$,
\[
    r_{T,\alpha}
    \ge
    (1-\varepsilon) \frac{1}{8\kappa_\nu^2} \alpha^2T.
\]
\end{lemma}
The proofs of both Lemmas are deferred to
Appendix~\ref{app:pgws-analysis}.

\begin{proof}[Proof of Theorem~\ref{thm:pgws-upper}]
By the definition of the lower $\alpha$-quantile,
$ 
    \mathbb P(R_T<r_{T,\alpha})
    \le
    \alpha
    \le
    \mathbb P(R_T\le r_{T,\alpha}).
$
The boundary randomization in \eqref{eq:pgws-terminal-rule} is chosen so that
$ 
    \mathbb P(\hat a_T=?)
    =
    \mathbb P(R_T<r_{T,\alpha})
    +
    \theta_{T,\alpha}\mathbb P(R_T=r_{T,\alpha})
    =
    \alpha.
$ 
Thus $\mathcal A_T(\ouralg(\alpha))=\alpha$.

For the error probability, Lemma~\ref{lem:pgws-certification} gives
$ 
    \mathcal E_T(\ouralg(\alpha))
    \le
    (K-1)e^{-r_{T,\alpha}}.
$
By Lemma~\ref{lem:pgws-threshold-lower}, for every $\varepsilon\in(0,1)$, all
sufficiently small $\alpha$, and all sufficiently large $T$,
$ 
    r_{T,\alpha}
    \ge
    (1-\varepsilon) \frac{1}{8\kappa_\nu^2}\alpha^2T.
$
Therefore
\[
    \mathcal E_T(\ouralg(\alpha))
    \le
    (K-1)
    \exp \left\{
        -(1-\varepsilon) \frac{1}{8\kappa_\nu^2}\alpha^2T
    \right\}.
\]
This then yields the final result in \eqref{eqn:ub}.  

\end{proof}


\section{Lower Bounds}
\label{sec:lower-bounds}

This section presents three lower bounds. The first is the main impossibility result for
Bayesian BAI with abstention and matches the upper bound presented in
Section~\ref{subsec:pgws-upper-analysis}. The second
shows that  forced-decision Bayesian BAI has only polynomial decay in the error probability, which refines the lower bound established by~\citet{atsidakou2023bayesian}. The third considers the
frequentist (non-Bayesian) formulation and shows that, when the instance has a positive gap,
abstention improves only lower-order exponential terms. Together, these results
identify both the value of abstention and the role of the Bayesian formulation.

\subsection{Bayesian BAI with Abstention}
\label{subsec:lower-bound-bayes-abstention}

\begin{theorem}[Lower Bound for Bayesian BAI with Abstention]
\label{thm:bayesian-abstention-lower}
In the Gaussian prior and Gaussian reward model,
\[
    \liminf_{\alpha\downarrow0}
    \liminf_{T\to\infty}
    \frac{1}{\alpha^2T}
    \log \mathfrak E_{T}(\alpha)
    \ge
    -\frac{1}{8\kappa_\nu^2}.
\]
\end{theorem}

Theorem~\ref{thm:bayesian-abstention-lower} applies to every adaptive sampling
and terminal decision rule satisfying the abstention constraint
$\mathcal A_T(\pi)\le\alpha$.
The proof, which is  deferred to Appendix \ref{pf_bayes_abs_lb}, proceeds by decomposing the parameter space into disjoint regions $\Omega_{ij}$ where arms~$i$ and~$j$ are the  two arms with the largest means. Since the undetected error is dominated by near-tie instances, we focus on small neighborhoods of the tie surfaces inside these regions and introduce a locally flat subprior that is a submeasure of the original prior. This reduction transforms the problem into a two-arm hypothesis testing problem under a convenient approximately uniform prior. Combining a Neyman--Pearson-like argument with the aggregation over all pairs yields the exponent $1/(8\kappa_\nu^2)$ that matches that of the upper bound.

Combining Theorem~\ref{thm:bayesian-abstention-lower} with the achievability
guarantee for \ouralg{} gives the tight  asymptotics on the exponential scale:
\begin{corollary}[Tight Asymptotics and Optimality of \ouralg{}]
\label{cor:tightness}
The following holds:
    \begin{equation}\label{equ:tightness}
    \limsup_{\alpha\downarrow0}
    \limsup_{T\to\infty}
    \frac{1}{\alpha^2T}
    \log \mathcal E_{T}(\ouralg(\alpha))
    =\liminf_{\alpha\downarrow0}
    \liminf_{T\to\infty}
    \frac{1}{\alpha^2T}
    \log \mathfrak E_{T}(\alpha)
    =
    - \frac{1}{8\kappa_\nu^2}.
\end{equation}
\end{corollary}

Corollary~\ref{cor:tightness} identifies \(\kappa_\nu\), the density at zero of the prior top-two gap $\Gamma$,
as the 
prior-dependent hardness parameter on the leading exponential scale. 


\subsection{Bayesian BAI without Abstention}
\label{subsec:lower-bound-no-abstention}
To characterize the role of the abstention formulation, we next consider the forced-decision  Bayesian BAI problem \citep{atsidakou2023bayesian}, in which the learner must recommend
an arm at the end of the experiment. Recall that we defined the forced-decision Bayes error probability as in \eqref{eqn:CT0}.
The following theorem shows that, without abstention, the Bayes probability of
misidentification cannot decay exponentially fast.

\begin{theorem}[No-Abstention Lower Bound]
\label{thm:no-abstention-lower}
For the Gaussian prior and Gaussian reward model,
\[
    \liminf_{T\to\infty}
    \sqrt T\,\mathfrak E_{T}(0)
    \ge
    \sqrt{\frac{2}{\pi}}\,\kappa_\nu.
\]
\end{theorem}

The proof is deferred to Appendix~\ref{app:no-abstention-lower}.
Theorem~\ref{thm:no-abstention-lower} refines the two-arm Bayesian
fixed-budget BAI lower bound of~\citet{atsidakou2023bayesian}. They show that,
for \(K=2\) Gaussian arms with Gaussian priors, the Bayes probability of
misidentification is \( \Omega((T\log T)^{-1/2})\), with the multiplicative factor depending on the  prior as $\exp\big\{\!-\!\frac{(\nu_1-\nu_2)^2}{4\sigma_0^2} \big\}$.
For \(K=2\), our hardness parameter is
    $\kappa_\nu
    =
    2\int_{\mathbb R}f_1(x)f_2(x)\,\dd x
    =
    \frac{1}{\sigma_0\sqrt{\pi}}
    \exp \big\{
        \!-\!\frac{(\nu_1-\nu_2)^2}{4\sigma_0^2}
    \big\}$.
Thus Theorem~\ref{thm:no-abstention-lower} improves the dependence on $T$, recovers the same prior-separation
dependence through $\kappa_\nu$, and extends the
result from two arms to general \(K\ge 2\).
For \(K>2\), \(\kappa_\nu\) represents a weighted sum (with weights $\{w_{ij}\}$) over all distinct pairs of arms $(i,j)$ of the likelihood of $i$ and $j$ being equal. 
Additionally, this theorem is also materially different from the lower bound analysis of the Bayesian simple regret in~\citet{komiyama2024rate}.
 That work studies the \emph{simple regret}, rather than the probability of error, and establishes a decay rate of $\Theta(T^{-1})$.
 

Theorem~\ref{thm:no-abstention-lower} identifies the mechanism behind the
polynomial Bayes error probability in standard forced-decision Bayesian FB-BAI. The dominant error comes
from instances whose top-two gap, $\Gamma=\mu_{(1)}-\mu_{(2)}$, is of order \(T^{-1/2}\). The prior probability of these instances is of order \(\kappa_\nu/\sqrt T\), and on these instances, the two leading arms cannot be reliably separated with $T$ samples. Thus the same local near-tie density \(\kappa_\nu\) that governs
the abstention exponent  also determines the constant   in
the forced-decision problem.

This contrasts with the results for Bayesian BAI with abstention in Theorem~\ref{thm:bayesian-abstention-lower}. By allowing an \(\alpha\)-fraction of inconclusive decisions on the most ambiguous posterior
histories, the learner removes the near-tie region that dominates the
forced-decision error probability. 
Consequently, the error probability exhibits a fundamental phase transition from 
polynomial decay (i.e., $\Omega(T^{-1/2})$)
to exponential decay (i.e., $\exp( - \Theta( \alpha^2T) )$).

\subsection{Frequentist BAI with Abstention}
\label{subsec:frequentist-abstention}

The preceding results concern the Bayesian formulation, where the error probability is
averaged over a  prior. 
We now contrast this behavior with the frequentist (fixed-instance) setting, and we show that the polynomial-to-exponential improvement is specific to the Bayesian formulation. 
In particular, for a fixed positive-gap instance,  the error probability of ordinary frequentist FB-BAI decays exponentially with \(T\) \citep{audibert2010best,carpentier2016tight, kaufmann2016complexity}. 
We show that the terminal abstention can improve only lower-order (i.e., sublinear in $T$)
terms in the exponent.

We exemplify this point with a symmetric pair of two-arm Gaussian instances which we call the \emph{mirror
experiments}. Pulling arm \(i\in\{1,2\}\) produces an independent reward from
\(\mathcal N(\mu_i,1)\). A strategy \(\pi\) adaptively samples the arms for \(T\) rounds
and then outputs $\hat a_T\in\{1,2,?\}$,
where \(?\) denotes abstention. For an instance \(\mu\), let
\(a^\star(\mu)=\arg\max_{i\in\{1,2\}}\mu_i\) be the optimal arm, and
define
\[
    \mathcal E_{T,\mu}(\pi)
    :=
    \mathbb P_\mu^\pi\!\left(
        \hat a_T\notin \{a^\star(\mu),?\}
    \right),
    \quad\mbox{and}\quad
    \mathcal A_{T,\mu}(\pi)
    :=
    \mathbb P_\mu^\pi(\hat a_T=?).
\]
For \(\alpha\in(0,1)\), let
\[
    \Pi_T^{\rm freq}(\alpha)
    :=
    \left\{
        \pi:
        \sup_{\mu\in\mathbb R^2}
        \mathcal A_{T,\mu}(\pi)\le \alpha
    \right\}
\]
be the class of strategies with  abstention probability uniformly bounded above by $\alpha$.
Let $z_\alpha
    :=
    \Phi^{-1}\!\left(\frac{1+\alpha}{2}\right)$,
so that \(\mathbb P(|Z|\le z_\alpha)=\alpha\) for
\(Z\sim\mathcal N(0,1)\).

Fix \(x\in\mathbb R\) and \(\Delta>0\), and consider the following two \emph{mirror instances}
\[
    \mu^+
    :=
    \left(
        x+\frac{\Delta}{2},
        x-\frac{\Delta}{2}
    \right),
    \quad\mbox{and}\quad
    \mu^-
    :=
    \left(
        x-\frac{\Delta}{2},
        x+\frac{\Delta}{2}
    \right).
\]
Under \(\mu^+\), arm \(1\) is optimal; under \(\mu^-\), arm \(2\) is optimal.

\begin{theorem}[Two-Arm Frequentist Lower Bound with Abstention]
\label{thm:two-arm-frequentist-abstention}
Given \(T\in\mathbb N\), \(\alpha\in(0,1)\), \(x\in\mathbb R\),
\(\Delta>0\), and any strategy
\(\pi\in\Pi_T^{\rm freq}(\alpha)\), we have
\[
    \max\left\{
        \mathcal E_{T,\mu^+}(\pi),
        \mathcal E_{T,\mu^-}(\pi)
    \right\}
    \ge
    \Phi\!\left(
        -\frac{\Delta\sqrt T}{2}-z_\alpha
    \right).
\]
\end{theorem}

The proof is deferred to Appendix~\ref{app:frequentist-abstention}. The theorem is a   minimax statement over the mirror instances \((\mu^+,\mu^-)\): any
strategy whose abstention probability is uniformly bounded by \(\alpha\) 
necessarily incurs the stated undetected error on at least one of the two instances. 

This bound is sharp for the two-arm Gaussian mirror experiment: when \(T\) is even, sampling
each arm \(T/2\) times and abstaining when $|\hat\mu_1-\hat\mu_2|\le \frac{2z_\alpha}{\sqrt T}$
attains the lower bound.

Theorem~\ref{thm:two-arm-frequentist-abstention} shows that abstention does not
change the leading fixed-gap exponential rate. Without abstention, the policy that pulls each arm an equal number of times results in the error probability 
\[
    \Phi\!\left(-\frac{\Delta\sqrt T}{2}\right)
    =
    \exp\!\left\{
        -\frac{\Delta^2}{8}T+O(\log T)
    \right\}.
\]
With an abstention budget \(\alpha\), the undetected error probability becomes
\[
    \Phi\!\left(
        -\frac{\Delta\sqrt T}{2}-z_\alpha
    \right)
    =
    \exp\!\left\{
        -\frac{\Delta^2}{8}T
        -\frac{\Delta z_\alpha}{2}\sqrt T
        -\frac{z_\alpha^2}{2}
        +O(\log T)
    \right\}.
\]
Hence, abstention only improves the exponent by a lower order term. For small \(\alpha\),
\(z_\alpha=\sqrt{\pi/2}\,\alpha+O(\alpha^3)\), so the logarithmic improvement is
of order \(O(\alpha\sqrt T)\). 

In contrast to the   Bayesian formulation,  the error probability under the frequentist setup is already decaying exponentially without abstention, and 
 the additional option to abstain yields only a lower-order improvement in the error probability, characterized by a moderate-deviation term.
 

\section{Extensions Beyond the Gaussian Model}
\label{sec:generalization}

In the preceding sections, we established the polynomial-to-exponential phase transition under the Gaussian  model. In this section, we show that this transition extends beyond the Gaussian setting. 
Specifically, the same exponent arises for regular one-parameter reward models after measuring
gaps in the   Fisher--Rao information scale.




Let \(\mathcal I\subseteq\mathbb R\) be an open interval, called the \emph{quality parameter space}. At the beginning of each experiment, the \emph{quality parameters} \(\mu_1,\ldots,\mu_K\in\mathcal I\) are drawn from their prior distributions and remain fixed throughout the experiment. The best arm is $a^\star:=\arg\max_{i\in[K]}\mu_i$.
Conditional on \(\mu_1, \ldots, \mu_K \), each pull of arm \(i\) produces an independent reward with distribution \(P_{\mu_i}\). We assume that each $P_{\mu_i}$ comes from a regular one-parameter exponential family, defined as follows. 


\begin{restatable}[Regular One-Parameter Exponential Family]{assumption}{AssRegularExpfam}
\label{ass:regular-expfam}
The family \(\{P_\mu: \mu\in\mathcal I\}\) is a regular full
one-parameter exponential family, possibly after a smooth monotone
reparameterization (which we call $\theta(\cdot)$); see, e.g., \cite[Sec.~34.3.1]{banditAlg}. Specifically,
there exists an open natural parameter space
\(\Theta\subseteq\mathbb R\), a dominating measure \(\lambda\), a natural
statistic \(t(\cdot)\), a base distribution \(h(\cdot)\), a \(C^4\) log-partition function
\(A:\Theta\to\mathbb R\) with $A''(\theta)>0$, and a \(C^4\) map
\(\theta:\mathcal I\to\Theta\) onto $\Theta$ with $\theta'(\mu)>0$ such that \(\theta(\mu)\) is the natural
parameter corresponding to the quality parameter \(\mu\), and
\[
    \frac{\dd P_\mu}{\dd \lambda}(w)
    =
    h(w)\exp\big\{
        \theta(\mu)t(w)-A(\theta(\mu))
    \big\},
    \qquad \mu\in\mathcal I .
\]
Moreover,
the image of every compact subset of $\mathcal I$ under
$\theta$ is compact in \(\Theta\).
\end{restatable}
Under Assumption~\ref{ass:regular-expfam}, the Fisher information for the
quality parameter \(\mu\) is
\[
    I(\mu)
    :=
    A''(\theta(\mu))\big(\theta'(\mu)\big)^2 .
\]
We then define the \emph{Fisher--Rao information coordinate}
\begin{equation}
    s(\mu)
    :=
    \int_{\mu_\circ}^{\mu}\sqrt{I(v)}\,\dd v, \label{eqn:fisher_rao}    
\end{equation}
where \(\mu_\circ\in\mathcal I\) is arbitrary and all quantities and results below are independent of the choice of \(\mu_\circ\), as changing $\mu_\circ$ shifts $s(\cdot)$ by an additive constant, leaving $\kappa$ and gap-based results unchanged.
This is the arc-length coordinate induced by the
Fisher--Rao metric~\citep{Rao1992,amari2000methods}.
 Since \(s(\cdot)\) is strictly
increasing, the best arm is unchanged by the transformation
\(\mu_i\mapsto \eta_i:=s(\mu_i)\).

\begin{restatable}[Regular Prior in Information Coordinates]{assumption}{AssRegularInfoPrior}
\label{ass:regular-info-prior}
The arms' quality parameters \(\mu_1,\ldots,\mu_K\) are independent under the prior. In the information coordinate \(\eta_i=s(\mu_i)\), the induced prior distribution of arm \(i\) has density \(\bar f_i\) and distribution function \(\bar F_i\) on \(\mathcal H:=s(\mathcal I)\). For each \(i\in[K]\),
\[
    \bar f_i \text{ is continuous and bounded on }\mathcal H,
    \qquad
    \bar f_i(\eta)>0 \quad \forall\, \eta\in\mathcal H .
\]
\end{restatable}
Define the \emph{information-scaled top-two tie density}
\[
    \kappa
    :=
    \sum_{i\neq j}
    \int_{\mathcal H}
        \bar f_i(u)\bar f_j(u)
        \prod_{k\neq i,j}\bar F_k(u)
    \,\dd u .
\]
Equivalently, if \(f_i\) and \(F_i\) respectively denote the prior density and distribution function
of \(\mu_i\) in the original coordinate system, then
\[
    \kappa
    =
    \sum_{i\neq j}
    \int_{\mathcal I}
        \frac{
            f_i(\mu)f_j(\mu)\prod_{k\neq i,j}F_k(\mu)
        }{\sqrt{I(\mu)}}
    \,\dd\mu .
\]
Thus \(\kappa\) is the density at zero of the top-two gap after 
transforming the quality parameters to Fisher--Rao coordinates via~\eqref{eqn:fisher_rao}.
In the Gaussian   model with unit variance, $\theta(\mu)\equiv\mu$ and 
\(I(\mu)\equiv1\), so this reduces to the top-two gap density used in the preceding
sections.

For the achievability result we also assume a global quadratic lower bound on
the Kullback--Leibler (KL) divergence. Let
$ 
    D(\mu,z):=D(P_\mu\,\|\,P_z) .
$

\begin{restatable}[Quadratic KL Growth in the Fisher--Rao Information Coordinate]{assumption}{AssQuadraticKL} \label{ass:quadratic-kl}
There exists \(c_{\rm KL}>0\) such that, for all \(\mu,z\in\mathcal I\),
\begin{equation}
    D(\mu,z)
    \ge
    c_{\rm KL}\big(s(\mu)-s(z)\big)^2 . \label{eqn:quad_growth}
\end{equation}
\end{restatable}
Assumption \ref{ass:quadratic-kl} holds for several one-parameter exponential families when $s(\cdot)$ is  the Fisher--Rao information coordinate. In particular,   Gaussians with unit variance, Bernoulli,  and Poisson  satisfy~\eqref{eqn:quad_growth} with $c_{\rm KL}=1/4$.  Assumption \ref{ass:quadratic-kl} can be viewed as a global version of the local quadratic approximation $D(\mu,z)\approx\frac{1}{2}\big(s(\mu)-s( z)\big)^2$ where $|s(\mu)-s( z)|$ is the Fisher--Rao distance~\citep{amari2000methods}.

\begin{theorem}[Extension to General Distributions]
\label{thm:general-transfer}
Suppose Assumptions~\ref{ass:regular-expfam}--\ref{ass:quadratic-kl} hold, then
\[
    \limsup_{\alpha\downarrow0}
    \limsup_{T\to\infty}
    \frac{1}{\alpha^2T}
    \log \mathcal E_{T}(\text{\textsc{IC-PGWS}}(\alpha))
    =\liminf_{\alpha\downarrow0}
    \liminf_{T\to\infty}
    \frac{1}{\alpha^2T}
    \log \mathfrak E_{T}(\alpha)
    =
    -\frac{1}{8\kappa^2}.
\]
The lower bound holds for all adaptive policies satisfying the abstention
constraint, and the upper bound is attained by an information-coordinate version
of \ouralg{}, namely, \textsc{IC-PGWS}.
\end{theorem}
The proof is postponed to Appendix~\ref{app:pf_thm_general}. We now provide a concrete example of applying Theorem~\ref{thm:general-transfer} to Beta-Bernoulli arms.

\begin{example}[Beta--Bernoulli arms]
\label{ex:beta-bernoulli}
Consider Bernoulli rewards with unknown success probabilities
\(\mu_i\in(0,1)\) that follow  Beta distributions, i.e., 
\[
    Y_{i,s}\mid \mu_i \sim {\rm Bernoulli}(\mu_i),
    \quad\mbox{and}\quad
    \mu_i\sim {\rm Beta}(a_i,b_i),
\]
independently across arms and time, where  \(a_i,b_i\ge 1/2\) for \(i \in [K]\).

 We show that the Beta--Bernoulli arms satisfy Assumptions~\ref{ass:regular-expfam}--\ref{ass:quadratic-kl}.
Firstly,
the Bernoulli family is a regular exponential family with
\[
    \theta(\mu)=\log\frac{\mu}{1-\mu},
    \quad
    A(\theta)=\log(1+e^\theta),
    \quad\mbox{and}\quad
    t(w)=w .
\]
Its Fisher information  is $I(\mu)=\frac{1}{\mu(1-\mu)}$. Hence, $s(\mu)= \int_0^\mu\sqrt{I(v)}\, \dd v= 2\arcsin\sqrt{\mu}$, and  $\mathcal H=(0,\pi)$.

Secondly, if \(\eta=s(\mu)\), then \(\mu=\sin^2(\eta/2)\), and the prior density of \(\eta_i\) is
\[
    \bar f_i(\eta)
    =
    \frac{
        \sin^{2a_i-1}(\eta/2)\cos^{2b_i-1}(\eta/2)
    }{
        {\rm Beta}(a_i,b_i)
    },
    \qquad 0<\eta<\pi .
\]
Because \(a_i,b_i\ge1/2\), this density is bounded, continuous, and strictly
positive on \((0,\pi)\). Hence, Assumption~\ref{ass:regular-info-prior} holds.

Lastly, a routine calculation shows that the Bernoulli KL divergence satisfies 
\[
    D(\mu,z)
    \ge
    \frac{1}{4}\big(s(\mu)-s(z)\big)^2,
    \qquad \mu,z\in(0,1),
\]
so Assumption~\ref{ass:quadratic-kl} also holds with $c_{\rm KL}=1/4$.

Therefore, the conditions of  Theorem~\ref{thm:general-transfer} hold and the theorem applies. Define
\[
    \kappa_{\rm BB}
    :=
    \sum_{i\neq j}
    \int_0^\pi
        \bar f_i(u)\bar f_j(u)
        \prod_{k\neq i,j}\bar F_k(u)
    \,\dd u 
    =
    \sum_{i\neq j}
    \int_0^1
        f_i(\mu)f_j(\mu)
        \prod_{k\neq i,j}F_k(\mu)
        \sqrt{\mu(1-\mu)}
    \,\dd\mu
\]
where \(\bar F_i\) is the cumulative distribution function  of \(2\arcsin\sqrt{\mu_i}\), and \(f_i\) and \(F_i\) are, respectively, the Beta\((a_i,b_i)\) density and CDF. So by Theorem~\ref{thm:general-transfer}, we have
\[
    \limsup_{\alpha\downarrow0}
    \limsup_{T\to\infty}
    \frac{1}{\alpha^2T}
    \log \mathcal E_{T}^{\rm BB}(\text{\textsc{IC-PGWS}}(\alpha))
    =\liminf_{\alpha\downarrow0}
    \liminf_{T\to\infty}
    \frac{1}{\alpha^2T}
    \log \mathfrak E_{T}^{\rm BB}(\alpha)
    =
    -\frac{1}{8\kappa^2_{\rm {BB}}}.
\]
in the Beta-Bernoulli case.
\end{example}

\section{Numerical Experiments}
\label{sec:experiments}

We conduct numerical experiments to illustrate  the benefit brought by the abstention option and the superior design of the sampling rule in \ouralg{}. 
First,
we evaluate the practical benefit of abstention by comparing the undetected error
probability for several abstention levels. Second, we isolate the value of the
\ouralg{} allocation by comparing it with uniform allocation
 under the same posterior abstention threshold. 

\subsection{Experimental Setup}
\label{subsec:experiment-setup}

We adopt the Gaussian prior and reward model from
Section~\ref{subsec:model-objective}. The prior is $\mu_i\sim \mathcal N(\nu_i,\sigma_0^2)$ for $i\in[K]$, independently across arms, and rewards satisfy $Y_{i,s}\mid \mu_i\sim \mathcal N(\mu_i,1)$. We consider the arm prior with $K=5$, $\sigma_0=1$, and $ \nu=[5,5,3,3,2]$.
%
The following methods are involved in the comparison: 
\begin{itemize}
    \item \ouralg$(\alpha)$: \ouralg{} with the
    quantile-calibrated posterior abstention threshold $\widehat r_{T,\alpha}$, so that the abstention probability is $\alpha$.

    \item \textsc{Unif}$(\alpha)$: Uniform allocation, followed by recommending the posterior leader. Use the same terminal statistic and calibration procedure as \ouralg.

    \item \textsc{BayesElim}~\citep{atsidakou2023bayesian}: Bayesian elimination without abstention, used as a
    forced-decision Bayesian BAI baseline. 
\end{itemize}

For \ouralg$(\alpha)$ and \textsc{Unif}$(\alpha)$, the abstention threshold is
calibrated separately for each sampling rule, budget $T$, so that the abstention probability is upper bounded by $\alpha$. 
Specifically, the abstention threshold is the lower $\alpha$-quantile $r_{T,\alpha}$ defined in~\eqref{equ:terminal_threshold}.
For \ouralg{}$(\alpha)$, for each given budget $T$ and abstention level $\alpha$, we apply the given sampling rule and generate terminal statistic samples $\{R_{T}^{(m)}\}_{m\in [M_{\rm cal}]}$ of independent $M_{\rm cal}=10^5$ simulations from the prior. We approximate the abstention threshold $r_{T,\alpha}$ by its empirical estimate
\[
    \widehat r_{T,\alpha}
    :=
    \min\left\{
        r:
        \frac1{M_{\rm cal}}
        \sum_{m=1}^{M_{\rm cal}}
        \mathbbm{1}\{R_T^{(m)}\le r\}
        \ge \alpha
    \right\}.
\]
For $\alpha=0$, there is no abstention  and the algorithm always
recommends the posterior mean leader.

The performance is evaluated on an independent test sample of
\(M_{\rm test}=10^6\) prior-generated experiments. For each method, budget \(T\),
and abstention level \(\alpha\), we 
compute the empirical abstention probability and empirical
undetected error probability as follows:
\[
    \widehat{\mathcal A}_T
    :=
    \frac1{M_{\rm test}}
    \sum_{m=1}^{M_{\rm test}}
    \mathbbm{1}\big\{\hat a_T^{(m)}=?\big\} ,\quad
    \mbox{and}
    \quad
     \widehat{\mathcal E}_T
    :=
    \frac1{M_{\rm test}}
    \sum_{m=1}^{M_{\rm test}}
    \mathbbm{1}\big\{
        \hat a_T^{(m)}\notin\{a^{\star,(m)},?\}
    \big\},
\]
where \(a^{\star,(m)}\) denotes the true best arm in the \(m\)th prior-generated
experiment, and \(\hat a_T^{(m)}\in[K]\cup\{?\}\) is the terminal output of the
algorithm.


\subsection{Effect of Abstention}
\label{subsec:experiment-abstention-benefit}
 In the first experiment, we aim to demonstrate the value of  abstention. 
We run
\ouralg$(\alpha)$ for
$\alpha\in\{0,0.01,0.03,0.05\}$
over a range of budgets $T$, 
and report the empirical undetected error probability 
$\widehat{\mathcal E}_T$. 
The empirical abstention rate $\widehat{\mathcal A}_T$ is reported separately to verify the correctness of threshold calibration. 
The results are shown in Figure \ref{fig:effect_of_abstention}:
\begin{itemize}
    \item \textbf{Significant reduction in undetected error probability} (Figure \ref{fig:effect_of_abstention}(a)): The forced-decision baselines (\ouralg$(0)$ and \textsc{BayesElim}~\citep{atsidakou2023bayesian}) exhibit a significantly slower decay in the error probability compared to the case where abstention is allowed, i.e., \ouralg$(\alpha)$ for positive $\alpha$. 
    This corroborates  our main result in Theorem~\ref{thm:pgws-upper}: spending a small fraction of inconclusive decisions on the most ambiguous instances yields an exponential reduction in terminal misidentifications.  In addition,  Figure~\ref{fig:effect_of_abstention}(a) also shows that the curves for \ouralg{}$(\alpha)$ when  $\alpha$  is  positive are approximately affine, demonstrating exponential decay in the undetected error probability. 
    \item \textbf{Abstention calibration} (Figure \ref{fig:effect_of_abstention}(b)): Across all budgets $T$, \ouralg$(\alpha)$ adheres to the prescribed abstention targets without being overly conservative. This validates the practical reliability of our calibrated abstention threshold.
\end{itemize}


\begin{figure}[t]
    \centering
    \begin{subfigure}{0.48\textwidth}
        \centering
        \includegraphics[width=\textwidth]{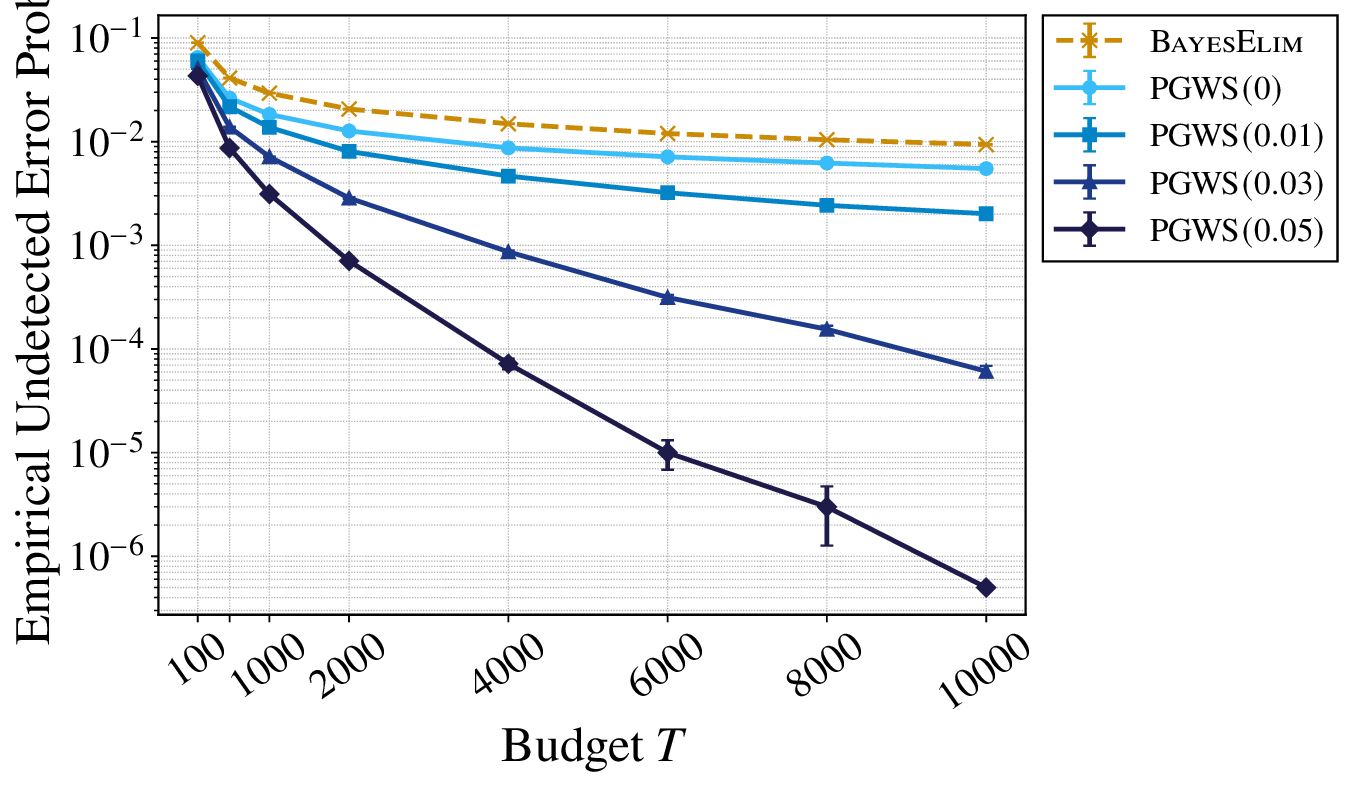}
        \caption{$\widehat{\mathcal E}_T$ against $T$ }
        \label{fig:sub1}
    \end{subfigure}
    \hfill 
    \begin{subfigure}{0.48\textwidth}
        \centering
        \includegraphics[width=\textwidth]{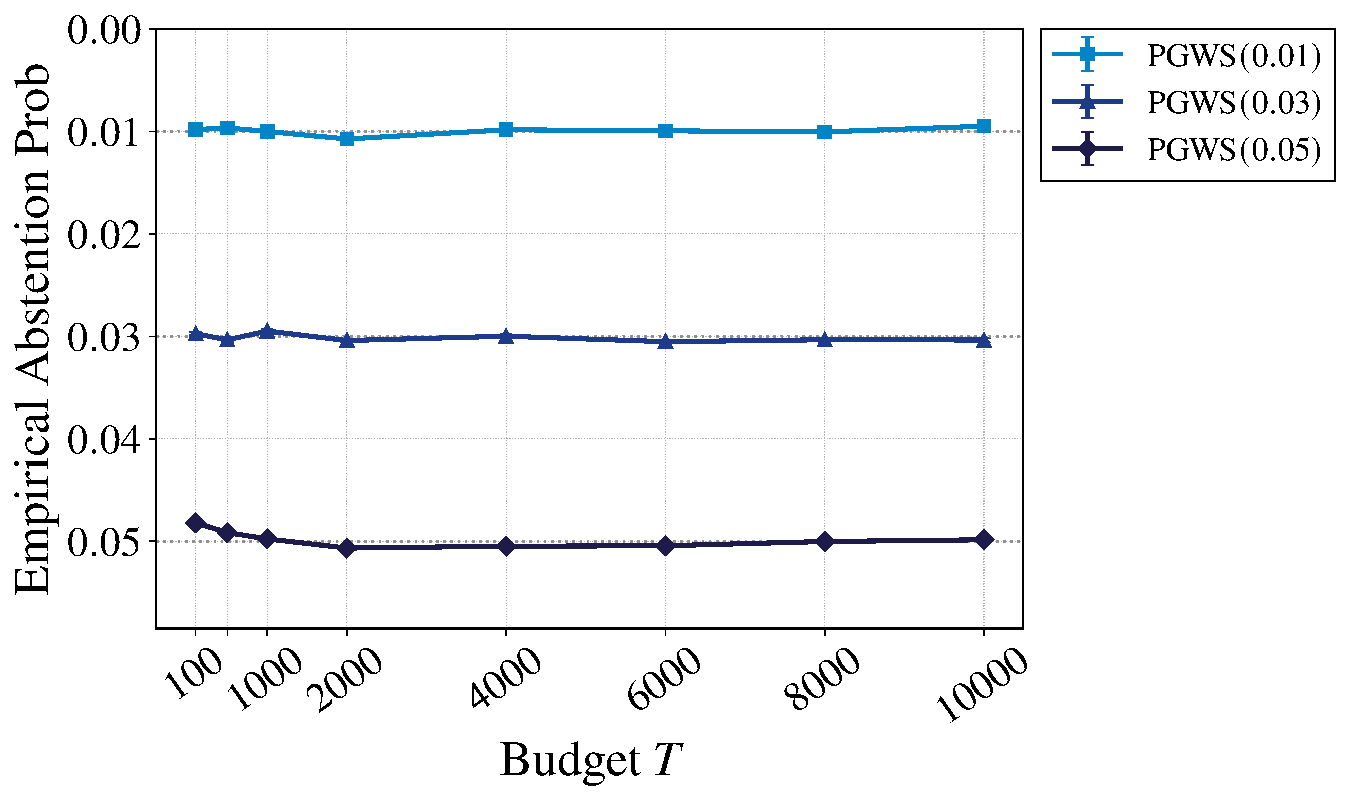}
        \caption{$\widehat{\mathcal A}_T$ against $T$}
        \label{fig:sub2}
    \end{subfigure}
    
    \caption{Plots of the undetected and abstention  probabilities as functions of  the budget $T$} 
    \label{fig:effect_of_abstention}
\end{figure}

\subsection{Effect of Adaptive Sampling in \ouralg{}}
\label{subsec:experiment-allocation}

The second experiment illustrates the effectiveness of the sampling strategy of \ouralg{} as in~\eqref{eq:pgws-sampling-prob}, which adapts to the instance. We compare
\ouralg$(\alpha)$ with uniform allocation \textsc{Unif}$(\alpha)$  for the abstention
levels $\alpha\in\{0,0.01,0.05\}$, and the results are shown in Figure \ref{fig:pgws_unif}. 
We observe that for every abstention level,  \ouralg$(\alpha)$ achieves a significantly lower undetected error probability than \textsc{Unif}$(\alpha)$ across all budgets.
The performance gap highlights the practical superiority of our algorithm's adaptive sampling design. By dynamically concentrating the sampling budget on the posterior mean leader and its most relevant challengers, rather than wasting pulls uniformly across all arms, the algorithm more efficiently resolves the ambiguous near-tie instances that dominate the Bayes error probability.

\begin{figure}[t]
    \centering
    \includegraphics[width=0.7\textwidth]{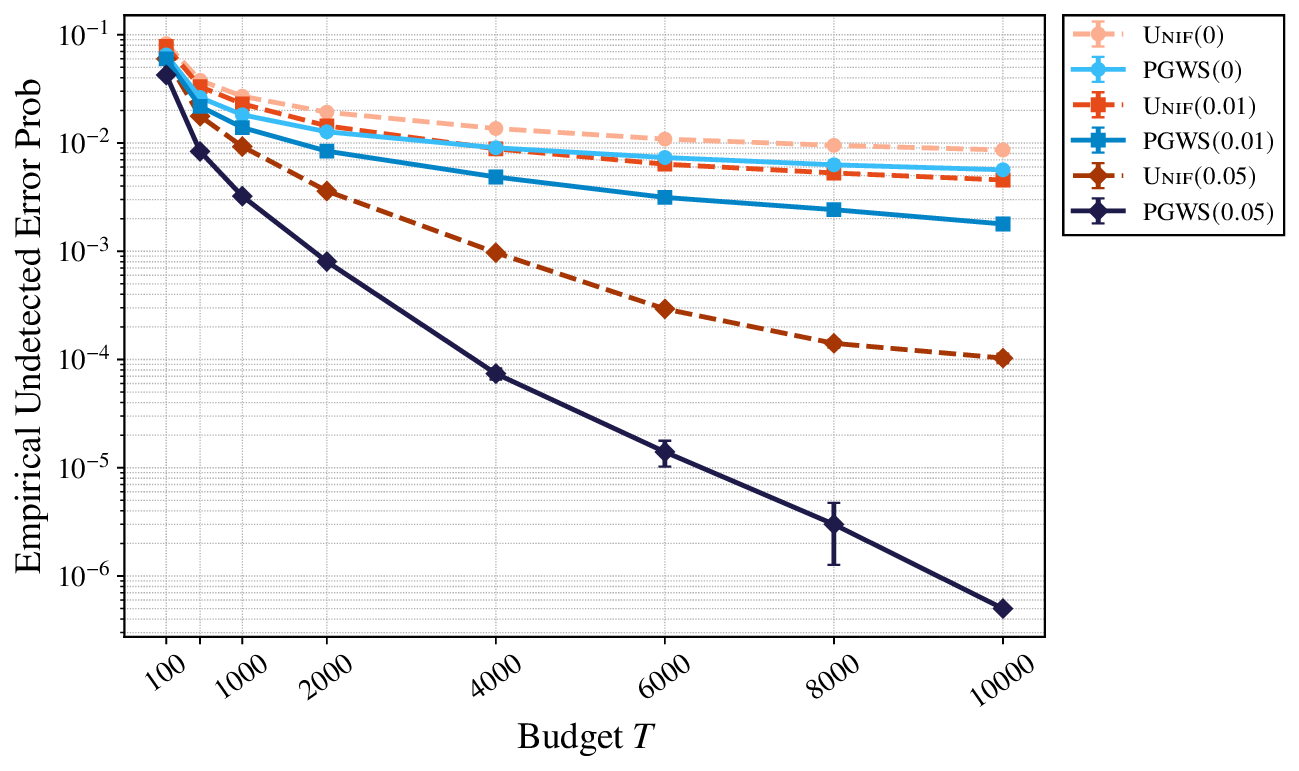}
    \caption{Comparison of  \ouralg{} to the algorithm which samples arms uniformly at random. }
    \label{fig:pgws_unif}
\end{figure}

\section{Conclusion}\label{sec:Conclusion}

In this paper, we studied Bayesian FB-BAI  with an abstention option. We characterized
the sharp small-abstention, large-budget exponent:
\[
    \mathfrak E_{T}(\alpha)
    =
    \exp\!\left\{
        -\frac{\alpha^2T}{8\kappa_\nu^2}
        +o(\alpha^2T)
    \right\}.
\]
The exponent  is governed by the density $\kappa_\nu$ of the prior top-two gap at
zero.

We proposed \ouralg{}, an adaptive algorithm that
allocates samples to the posterior mean leader and its  challengers.
 \ouralg{} samples arms according to the empirical allocation distribution which is proportional to the reciprocal of the squared posterior mean gaps.
It achieves the same   exponent as the information-theoretic lower bound. Hence it is asymptotically optimal on the $\alpha^2T$ logarithmic scale. Numerical experiments support the practicality of our theoretical findings.

Our results clarify the benefits of abstention. Without abstention, Bayesian BAI has only a polynomial error decay, driven by prior probability on instances with top-two gap of order $T^{-1/2}$.
With abstention, the learner can spend its abstention budget on these near-tie instances, leading to exponentially decaying undetected error. In contrast, in the frequentist fixed-gap setting, ordinary BAI is already exponentially accurate, and abstention improves only lower-order
terms in the exponent.

Finally, we showed that the same mechanism extends beyond the Gaussian model
under continuous prior, one parameter exponential family reward, and quadratic KL growth conditions, with the Gaussian top-two gap density replaced by an information-scaled top-two gap density.

\begingroup \parindent 0pt \parskip 4ex
\def\enotesize{\normalsize} 
\endgroup

%
%
%


\newpage
\appendix
\noindent
{\Large{\textbf{Appendix}}}
\section{Prior Density of the Top-Two Gap at Zero} \label{pf:lem:top-gap-density}

\begin{lemma}[Prior Density of the Top-Two Gap at Zero]
\label{lem:top-gap-density}
The hardness parameter $\kappa_\nu$ and the top-two gap $\Gamma$ are related as follows:
\[
    \kappa_\nu
    =
    \lim_{\varepsilon\downarrow0}
    \frac{\mathbb P_\nu(\Gamma\le \varepsilon)}{\varepsilon}.
\]
\end{lemma}

\begin{proof}
For each arm $i \in [K]$, the event that arm $i$ is the best and the top-two gap is at most $\varepsilon$ is
$ 
\left\{
\mu_i=\mu_{(1)},\ 
\mu_{(1)}-\mu_{(2)}\le \varepsilon
\right\}.
$
Conditional on $\{\mu_i=x\}$, this event says that the means of  all the other arms  are at most $x$ and at least one other arm is in $[x-\varepsilon,x]$. Hence
$$
\begin{aligned}
    \mathbb P(\Gamma\le \varepsilon)
&= \sum_{i=1}^K \mathbb P(\mu_i=\mu_{(1)},\ 
\mu_{(1)}-\mu_{(2)}\le \varepsilon) =\sum_{i=1}^K
\int_{\mathbb R}
f_i(x)
\left[
\prod_{k\neq i}F_k(x)
-
\prod_{k\neq i}F_k(x-\varepsilon)
\right]\dd x
\end{aligned}
$$
Let
$G_i(x):=\prod_{k\neq i}F_k(x)$, then $G_i'(x)
=
\sum_{j\neq i} \big(
f_j(x)\prod_{k\neq i,j}F_k(x) \big)$.
By the fundamental theorem of calculus,
$ 
G_i(x)-G_i(x-\varepsilon)
=
\int_{x-\varepsilon}^x G_i'(u)\,\dd u.
$
Additionally, by the dominated convergence theorem, as $\varepsilon \to 0$,
$$
\begin{aligned}
\frac{\mathbb P(\Gamma\le \varepsilon)}{\varepsilon} 
&= \sum_{i=1}^K \int_{\mathbb R} f_i(x) \left[ \frac{G_i(x) - G_i(x-\varepsilon)}{\varepsilon} \right]  \, \dd x\to
\sum^K_{i=1} \sum_{j\neq i}
\int_{\mathbb R}
f_i(x)f_j(x)\prod_{k\neq i,j}F_k(x)\,\dd x.
\end{aligned}
$$
It remains to simplify the product of the Gaussian densities. For $i<j$,
$$
\begin{aligned}
f_i(x)f_j(x) 
&= \frac{1}{2\pi\sigma_0^2} \exp\!\left( -\frac{(x-\nu_i)^2 + (x-\nu_j)^2}{2\sigma_0^2} \right) 
=
\frac{1}{2\sqrt\pi\,\sigma_0}
\exp\left(
-\frac{(\nu_i-\nu_j)^2}{4\sigma_0^2}
\right)
g_{ij}(x),
\end{aligned}
$$
where $g_{ij}$ is the density of
$\mathcal N\!\left(
\frac{\nu_i+\nu_j}{2},
\frac{\sigma_0^2}{2}
\right)$.
Therefore
$$
\begin{aligned}
2\int_{\mathbb R}
f_i(x)f_j(x)\prod_{k\neq i,j}F_k(x)\,\dd x
&=
\frac{1}{\sigma_0\sqrt\pi}
\exp\!\left(
-\frac{(\nu_i-\nu_j)^2}{4\sigma_0^2}
\right)
\int_{\mathbb R}
g_{ij}(x)\prod_{k\neq i,j}F_k(x)\,\dd x \\
&=
\frac{1}{\sigma_0\sqrt\pi}
\exp\!\left(
-\frac{(\nu_i-\nu_j)^2}{4\sigma_0^2}
\right)
w_{ij}.
\end{aligned}
$$
Summing over \(i<j\) proves the claim.

\end{proof}

\section{Proof of Theorem \ref{thm:pgws-upper}} \label{app:pgws-analysis}

\subsection{Error Guarantee}

\begin{proof}[Proof of Lemma \ref{lem:pgws-certification}]
If the algorithm recommends $\widehat b_T$, then conditional on $\mathcal F_T$,
\[
\mathbb P(a^\star\neq \widehat b_T\mid \mathcal F_T)
\le
\sum_{j\neq \widehat b_T}
\mathbb P(\mu_j\ge \mu_{\widehat b_T}\mid \mathcal F_T).
\]

As usual, let \(\mu=(\mu_1,\dots,\mu_K)\). By the independent Gaussian prior and the arm-wise reward model, the posterior factorizes across arms even under adaptive sampling:
\[
p(\mu\mid\mathcal F_T)
\propto
\prod_{i=1}^K
\left[
p(\mu_i)
\prod_{s=1}^{N_i(T)}
p(Y_{i,s}\mid \mu_i)
\right].
\]
Consequently, the law of the means is
\[
\mathcal L(\mu_1,\dots,\mu_K\mid\mathcal F_T)
=
\bigotimes_{i=1}^K \mathcal N(M_i(T),V_i(T)),
\]
thus the arm means \(\mu_1,\dots,\mu_K\) are mutually independent conditional on
\(\mathcal F_T\).

Furthermore, \(M_i(T)\) and \(V_i(T)\) are \(\mathcal F_T\)-measurable, and therefore
$\widehat b_T=\arg\max_i M_i(T)$
is also \(\mathcal F_T\)-measurable. Hence,  conditional on \(\mathcal F_T\), the random index
\(\widehat b_T\) is fixed. Thus, for each \(j\neq \widehat b_T\),
\[
\mu_{\widehat b_T}-\mu_j\mid\mathcal F_T
\sim
\mathcal N\!\left(
M_{\widehat b_T}(T)-M_j(T),
V_{\widehat b_T}(T)+V_j(T)
\right).
\]
Therefore, using $\Phi(-x)\le e^{-x^2/2}$ for $x\ge0$,
\[
\mathbb P(\mu_j\ge\mu_{\widehat b_T}\mid\mathcal F_T)
=
\Phi\left(
-\frac{M_{\widehat b_T}(T)-M_j(T)}
{\sqrt{V_{\widehat b_T}(T)+V_j(T)}}
\right)
\le
\exp\left(
-\frac{
\bigl(M_{\widehat b_T}(T)-M_j(T)\bigr)^2
}{
2\bigl(V_{\widehat b_T}(T)+V_j(T)\bigr)
}
\right).
\]
By definition of $R_T$, on the event that the algorithm does not abstain we have
$R_T\ge r_{T,\alpha}$. Hence, on this event,
$
\mathbb P(a^\star\neq \widehat b_T\mid\mathcal F_T)
\le
(K-1)e^{-R_T}
\le
(K-1)e^{-r_{T,\alpha}}.
$

Taking expectation gives
\[
\mathcal E_T
=
\mathbb E\!\left[
\mathbbm{1}\{\hat a_T\neq ?\}
\mathbb P(a^\star\neq \widehat b_T\mid\mathcal F_T)
\right]
\le
(K-1)e^{-r_{T,\alpha}}.  
\]
\end{proof}

\subsection{Asymptotic Behavior of the Statistic $R_T$}

To analyze the (asymptotic) behavior of $r_{T,\alpha}$, we first analyze $R_T$.

For ease of notation, we introduce auxiliary unnormalized weights.  At time \(t\), define \(\widehat b_t\) and \(\widehat\Delta_i(t)\) from the
posterior means as in Algorithm~\ref{alg:pgws}.  We also  define $\widehat{\Gamma}_t = \widehat\Delta_{\widehat b_t}(t)$. Set  $w_{\widehat b_t}(t)=1$, $w_{j}(t)=\big(\frac{\widehat{\Gamma}_t}{\widehat\Delta_j(t)}\big)^2$ for $j \neq \widehat b_t$,
and $W(t):=\sum_{\ell=1}^K w_\ell(t)$. Then, on every non-forced round, the sampling rule in~\eqref{eq:pgws-sampling-prob} in
Algorithm~\ref{alg:pgws} can be written equivalently as $p_i(t)=\frac{w_i(t)}{W(t)}$ for $i\in[K]$.

Let
$b^\star:=\arg\max_{i\in[K]}\mu_i$
be the true best arm which is  unique almost surely.
For $j\neq b^\star$, define the true gap
$\Delta_j:=\mu_{b^\star}-\mu_j>0$.
Recall  that the true top-two gap is
$\Gamma:=\min_{j\neq b^\star}\Delta_j
=
\mu_{(1)}-\mu_{(2)}$.

Define the (deterministic) unnormalized limiting sampling weights
\[
w_{b^\star}^\dagger(\mu):=1,
\quad\mbox{and}\quad
w_j^\dagger(\mu)
:=
\left(
\frac{\Gamma}{\Delta_j}
\right)^2,
\qquad j\neq b^\star,
\]
Also define their sum, which is the normalization constant for the sampling distribution
\begin{equation}
W^\dagger(\mu)
:=
\sum_{\ell=1}^K w_\ell^\dagger(\mu)
=
1+\sum_{j\neq b^\star}
\left(
\frac{\Gamma}{\Delta_j}
\right)^2. \label{eqn:W_dagger}
\end{equation}
Finally, define the limiting sampling proportions
\[
p_i^\dagger(\mu)
:= \frac{w_i^\dagger(\mu)}
{W^\dagger(\mu)} =
\frac{w_i^\dagger(\mu)}
{\sum_{\ell=1}^K w_\ell^\dagger(\mu)},
\]
as well as the worst-case pairwise error exponent
\[
C^\dagger(\mu)
:=
\min_{j\neq b^\star}
\frac{
\Delta_j^2
}{
2\left(
1/p_{b^\star}^\dagger(\mu)+1/p_j^\dagger(\mu)
\right)
}.
\]

\noindent Next we simplify $C^\dagger(\mu)$.

\begin{lemma}[Explicit Form of the Limiting Exponent]\label{lem:Cdagger-explicit}
Almost surely, $$C^\dagger(\mu)
=
\frac{\Gamma^2}{4W^\dagger(\mu)}.$$
In particular,
$$\frac{\Gamma^2}{4K}
\le
C^\dagger(\mu)
\le
\frac{\Gamma^2}{8}.$$
\end{lemma}

\begin{proof}
By definition, the normalized sampling proportions are:
$$p_{b^\star}^\dagger=\frac{1}{W^\dagger(\mu)},
\quad\mbox{and}\quad
p_j^\dagger
=
\frac{\Gamma^2/\Delta_j^2}{W^\dagger(\mu)}.$$
We substitute these proportions into the pairwise error exponent for a specific challenger $j$:
$$\frac{
\Delta_j^2
}{
2\left(
1/p_{b^\star}^\dagger+1/p_j^\dagger
\right)
}
=
\frac{
\Delta_j^2
}{
2W^\dagger(\mu)\left(1+\Delta_j^2/\Gamma^2\right)
}.$$
To isolate the effect of the gap, we factor out $\frac{\Gamma^2}{2W^\dagger(\mu)}$ and express the remainder as a function of the relative gap  $r_j := \Delta_j/\Gamma$: $$\frac{
\Delta_j^2
}{
2\left(
1/p_{b^\star}^\dagger+1/p_j^\dagger
\right)
}
=
\frac{\Gamma^2}{2W^\dagger(\mu)}
\left[
\frac{(\Delta_j/\Gamma)^2}
{1+(\Delta_j/\Gamma)^2}
\right]
=
\frac{\Gamma^2}{2W^\dagger(\mu)}
\left(
\frac{r_j^2}{1+r_j^2}
\right).$$

Now, recall the limiting exponent $C^\dagger(\mu)$ is defined as the minimum of this expression over all challengers $j \neq b^\star$. Observe that the function $f(r) = \frac{r^2}{1+r^2}$ is increasing on $[1,\infty)$. Since $\Delta_j/\Gamma\ge1$ for all $j\neq b^\star$, the minimum is attained by
a second-best arm, for which $\Delta_j=\Gamma$:
$$C^\dagger(\mu)
=
\frac{\Gamma^2}{2W^\dagger(\mu)}
\left(
\frac{1^2}{1+1^2}
\right)
=
\frac{\Gamma^2}{4W^\dagger(\mu)}.$$
Also, because $w_{b^\star}^\dagger=1$, the second-best arm has weight $1$, and all other
weights are at most $1$, we have 
$2\le W^\dagger(\mu)\le K$.
Substituting back gives $$\frac{\Gamma^2}{4K}
\le
C^\dagger(\mu)
\le
\frac{\Gamma^2}{8}.  $$ 
\end{proof}

Next, we prove that the forced-exploration safeguard collects enough data to guarantee the posterior mean of every arm converges to its true mean.

\begin{lemma}[Posterior Consistency]\label{lem:pgws-posterior-consistency}
For every arm $i$,
\begin{equation}
M_i(t)\to \mu_i    \qquad\text{almost surely}. \label{eqn:Mit_as}
\end{equation}
\end{lemma}

\begin{proof}

First, recall that the algorithm's forced-exploration samples arm $i$ whenever $N_i(t) < \sqrt{t+1}$, this guarantees that as $t \to \infty$, $N_i(t)\to\infty$ for every $i$.

Second, conditional on a specific realization of the true mean $\mu_i$, the reward sequence $(Y_{i,s})_{s\ge1}$ consists of mutually independent draws from $\mathcal N(\mu_i,1)$. Therefore, by the Strong Law of Large Numbers, the conditional probability of the empirical average converging to $\mu_i$ is $1$. Integrating this over the prior of $\mu_i$, we have 
$$\frac1{N_i(t)}
\sum_{s=1}^{N_i(t)}Y_{i,s}
\to
\mu_i
\qquad\text{almost surely}.$$ 

Finally, recall our algorithm's update for the posterior mean:
$$M_i(t)
=
\frac{
\sigma_0^{-2}\nu_i+\sum_{s=1}^{N_i(t)}Y_{i,s}
}{
\sigma_0^{-2}+N_i(t)
}.$$
Combining the previous two displayed equations yields~\eqref{eqn:Mit_as}.  
\end{proof}

Next, we show that the empirical allocation of arm pulls converges to the  sampling proportions. 

\begin{lemma}[Convergence of Empirical Sampling Proportions]
\label{lem:pgws-sampling-proportions}
For every arm $i$,
\begin{equation}
\frac{N_i(T)}{T}
\to
p_i^\dagger(\mu)
\qquad\text{almost surely}.    \label{eqn:allocations}
\end{equation}
\end{lemma}

\begin{proof}
Fix an arm $i$. We proceed in four steps:

\textbf{Step 1: Forced exploration is negligible.}
Let $F_i(T)$ be the number of forced-exploration pulls of arm $i$ up to time $T$.
If the $m$-th forced pull of arm $i$ occurs before time $T$, then immediately before that pull,
$N_i(t)\ge m-1$.
But a forced pull of arm $i$ can occur only if
$N_i(t)<\sqrt{t+1}\le \sqrt{T+1}$.
Hence
$m-1<\sqrt{T+1}$,
and therefore
$F_i(T)\le \lceil \sqrt{T+1}\rceil+1$.
Consequently,
\[
\frac{F_i(T)}{T}\to0.
\]
Let $F(T)$ be the total number of forced-exploration rounds up to time $T$. Since
$F(T)=\sum_{\ell=1}^K F_\ell(T)$,
we also have
\[
\frac{F(T)}{T}\le
\frac{K(\lceil\sqrt{T+1}\rceil+1)}{T}
\to0.
\]

\textbf{Step 2: Martingale concentration on non-forced rounds.}
Let $\mathcal F_t$ be the history before round $t+1$, and define
$B_t
:=
\mathbbm{1}\{\text{round }t+1\text{ is non-forced}\}$.
The event $\{B_t=1\}$ is $\mathcal F_t$-measurable. On non-forced rounds, the algorithm samples arm $i$ with conditional probability $p_i(t)$.

Define
$D_{i,t+1}
:=
B_t\bigl(\mathbbm{1}\{A_{t+1}=i\}-p_i(t)\bigr)$.
Then
$\mathbb E[D_{i,t+1}\mid \mathcal F_t]=0$,
so $(D_{i,t+1})_{t\ge0}$ is a martingale difference sequence. Also,
$|D_{i,t+1}|\le1$.

Therefore, by Azuma's inequality, for every $\varepsilon>0$,
\[
\mathbb P\left(
\bigg|
\frac1T\sum_{t=0}^{T-1}D_{i,t+1}
\bigg|
\ge \varepsilon
\right)
\le
2\exp\left(-\frac{\varepsilon^2T}{2}\right).
\]
The right-hand side is summable in $T$. Hence, by the Borel--Cantelli lemma,
\[
\frac1T\sum_{t=0}^{T-1}D_{i,t+1}\to0
\qquad\text{almost surely}.
\]

\textbf{Step 3: Decomposition of the empirical pull count.}
The total number of pulls of arm $i$ can be decomposed as
\[
N_i(T)
=
F_i(T)
+
\sum_{t=0}^{T-1}
B_t\mathbbm{1}\{A_{t+1}=i\}.
\]
Using the definition of $D_{i,t+1}$, we  have
$B_t\mathbbm{1}\{A_{t+1}=i\}
=
D_{i,t+1}+B_t p_i(t)$.
Thus
\[
\frac{N_i(T)}{T}
=
\frac{F_i(T)}{T}
+
\frac1T\sum_{t=0}^{T-1}D_{i,t+1}
+
\frac1T\sum_{t=0}^{T-1}B_t p_i(t).
\]
By Steps 1 and 2,
\[
\frac{N_i(T)}{T}
-
\frac1T\sum_{t=0}^{T-1}B_t p_i(t)
\to0
\qquad\text{almost surely}.
\]

\textbf{Step 4: Ces\`aro convergence of the target probabilities.}
By Lemma~\ref{lem:pgws-posterior-consistency}, $M(t)\to \mu$ almost surely. 
Since the prior is continuous, the true best arm and true second-best arm are unique almost surely. On this probability-one event, the map $m\mapsto p(m)$ is continuous at $m=\mu$. Therefore, as $t \to \infty$,
\[
p_i(t)\to p_i^\dagger(\mu)
\qquad\text{almost surely}.
\]
Hence, the sequence $\{p_i(t)\}$ is also  Ces\`aro convergent, i.e., 
\[
\frac1T\sum_{t=0}^{T-1}p_i(t)
\to
p_i^\dagger(\mu)
\qquad\text{almost surely}.
\]

Now write
\[
\frac1T\sum_{t=0}^{T-1}B_t p_i(t)
=
\frac1T\sum_{t=0}^{T-1}p_i(t)
-
\frac1T\sum_{t=0}^{T-1}(1-B_t)p_i(t).
\]
Since $0\le p_i(t)\le1$,
\[
0\le
\frac1T\sum_{t=0}^{T-1}(1-B_t)p_i(t)
\le
\frac1T\sum_{t=0}^{T-1}(1-B_t)
=
\frac{F(T)}{T}
\to0.
\]
Hence
\[
\frac1T\sum_{t=0}^{T-1}B_t p_i(t)
\to
p_i^\dagger(\mu)
\qquad\text{almost surely}.
\]
Combining this with Step 3 gives \eqref{eqn:allocations}.  
\end{proof}
Finally, we show that the statistic $R_T/T$ converges to $ C^\dagger(\mu)$.

\begin{lemma}[Convergence of the Terminal Statistic]
\label{lem:RT-converges}
Under the joint Bayesian law,
\begin{equation}
    \frac{R_T}{T}
\to
C^\dagger(\mu)
\qquad\text{almost surely}. \label{eqn:terminal_stat}
\end{equation} 
\end{lemma}

\begin{proof}
By Lemma~\ref{lem:pgws-posterior-consistency}, $M_i(T)\to \mu_i$ for every $i$.
Since the true best arm is unique almost surely, $\widehat b_T=b^\star$ for $T$ large enough.
So for every $j\neq b^\star$,
$M_{b^\star}(T)-M_j(T)\to \Delta_j$.

Recall $V_i(T)=\frac{1}{\sigma_0^{-2}+N_i(T)}$.
Furthermore, by Lemma~\ref{lem:pgws-sampling-proportions},
$\frac{N_i(T)}{T}\to p_i^\dagger(\mu)$.

Hence
\[
T\ V_i(T)
=
\frac{T}{\sigma_0^{-2}+N_i(T)}
\to
\frac{1}{p_i^\dagger(\mu)}.
\]
Therefore, for every $j\neq b^\star$,
\[
\frac1T
\frac{
\bigl(M_{b^\star}(T)-M_j(T)\bigr)^2
}{
2\bigl(V_{b^\star}(T)+V_j(T)\bigr)
}
\to
\frac{
\Delta_j^2
}{
2\left(
1/p_{b^\star}^\dagger(\mu)+1/p_j^\dagger(\mu)
\right)
}.
\]
Taking the minimum over $j\neq b^\star$ proves \eqref{eqn:terminal_stat} as desired.  
\end{proof}

\subsection{Lower Tail of the  Exponent}

We now analyze the  lower tail of $C^\dagger(\mu)$. We first bound the event that the means of three or more arms are $\varepsilon$-close. 
\begin{lemma}[Triple Near-Tie Probability]
\label{lem:triple-near-tie}
Let $\Gamma_3:=\mu_{(1)}-\mu_{(3)}$ when $K\ge3$. If $K=2$, set $\Gamma_3:=+\infty$. We have that as $\varepsilon\downarrow0$,
\[
\mathbb P(\Gamma_3\le \varepsilon)=O(\varepsilon^2).
\]
\end{lemma}

\begin{proof}
For $K=2$ the statement is trivial. Assume $K\ge3$.
If $\Gamma_3\le \varepsilon$, then there exist three distinct arms $i,j,k$
such that all three means lie within an interval of length $\varepsilon$ around
the largest of the three. By a union bound over ordered triples,
\[
\mathbb P(\Gamma_3\le \varepsilon)
\le
\sum_{i\neq j\neq k}
\int_{\mathbb R}
f_i(x)
\left[
\int_{x-\varepsilon}^x f_j(y)\,\dd y
\right]
\left[
\int_{x-\varepsilon}^x f_k(z)\,\dd z
\right]\dd x.
\]
Gaussian densities are bounded uniformly by $(\sigma_0\sqrt{2\pi})^{-1}$, so each
inner integral is $O(\varepsilon)$ uniformly in $x$. Since $\int f_i(x)\,\dd x=1$,
the whole expression is $O(\varepsilon^2)$.  
\end{proof}

\begin{remark}
    This lemma reveals why asymptotically the leading-order complexity of the abstention problem depends only on pairwise gaps. While a two-arm near-tie forms the primary statistical bottleneck (occurring with probability $O(\varepsilon)$), the event that three or more arms are simultaneously tied is one order rarer (namely, $O(\varepsilon^2)$). Consequently, in the small-abstention regime, the asymptotic error exponent is perfectly captured by the closest challenger alone; the statistical cost of separating third-best (and worse) arms vanishes. 
\end{remark}

Next, we show that $W^\dagger(\mu)$, defined in \eqref{eqn:W_dagger}, is approximately $2$ when $C^\dagger(\mu)$ is small, because it is unlikely (by  Lemma~\ref{lem:triple-near-tie}) that many other arms are close to the optimal. Consequently, $\Gamma^2/8$ and $C^\dagger(\mu)$ behave similarly.

\begin{lemma}[Lower Tail of $C^\dagger(\mu)$]
\label{lem:Cdagger-tail}
As $x\downarrow0$,
\begin{equation}
    \mathbb P(C^\dagger(\mu)\le x)
=
\kappa_\nu\sqrt{8x}
+
o(\sqrt x). \label{eqn:Cdagger}
\end{equation}
\end{lemma}

\begin{proof}
By Lemma~\ref{lem:Cdagger-explicit},
$C^\dagger(\mu)
=
\frac{\Gamma^2}{4W^\dagger(\mu)}
\le
\frac{\Gamma^2}{8}$.
Therefore
$\{\Gamma\le \sqrt{8x}\}
\subseteq
\{C^\dagger(\mu)\le x\}$.

Using Lemma~\ref{lem:top-gap-density},
\[
\mathbb P(C^\dagger(\mu)\le x)
\ge
\mathbb P(\Gamma\le \sqrt{8x})
=
\kappa_\nu\sqrt{8x}+o(\sqrt x).
\] 

\noindent On the other hand, to derive an upper bound, set
$\delta_x:=x^{3/8}$.
On the event $\{C^\dagger(\mu)\le x\}$, Lemma~\ref{lem:Cdagger-explicit} yields
$\Gamma^2\le 4Kx$,
because $W^\dagger(\mu)\le K$.
On the additional event $\{\Gamma_3>\delta_x\}$, every arm other than the top two
has gap at least $\delta_x$ from the best arm. Hence, for every such arm $k$,
\[
\left(\frac{\Gamma}{\Delta_k}\right)^2
\le
\frac{4Kx}{\delta_x^2}
=
4Kx^{1/4}
=o(1)\quad \text{ as } x \downarrow 0.
\]
Thus, on
$\{C^\dagger(\mu)\le x,\ \Gamma_3>\delta_x\}$,
we have
\[
W^\dagger(\mu)
=\sum_{\ell=1}^K w_\ell^\dagger(\mu)
=
1+\sum_{j\neq b^\star}
\left(
\frac{\Gamma}{\Delta_j}
\right)^2=
2+o(1),
\]
uniformly as $x\downarrow0$. Therefore, for every fixed $\eta>0$ and all
sufficiently small $x$,
\[
C^\dagger(\mu)\le x,\ \Gamma_3>\delta_x
\quad\Longrightarrow\quad
\Gamma = 2 \sqrt{C^\dagger(\mu) W^\dagger(\mu)}
\le
(1+\eta)\sqrt{8x}.
\]
Consequently,
\[
\mathbb P(C^\dagger(\mu)\le x)
\le
\mathbb P\bigl(\Gamma\le(1+\eta)\sqrt{8x}\bigr)
+
\mathbb P(\Gamma_3\le\delta_x).
\]
By Lemma~\ref{lem:top-gap-density},
\[
\mathbb P\bigl(\Gamma\le(1+\eta)\sqrt{8x}\bigr)
=
(1+\eta)\kappa_\nu\sqrt{8x}+o(\sqrt x).
\]
By Lemma~\ref{lem:triple-near-tie},
\[
\mathbb P(\Gamma_3\le\delta_x)
=
O(\delta_x^2)
=
O(x^{3/4})
=
o(\sqrt x).
\]
Thus
\[
\limsup_{x\downarrow0}
\frac{\mathbb P(C^\dagger(\mu)\le x)}{\sqrt x}
\le
(1+\eta)\kappa_\nu\sqrt8.
\]
Letting $\eta\downarrow0$ gives the matching upper bound, and hence we obtain \eqref{eqn:Cdagger}.  
\end{proof}

\subsection{Lower Bound on the Calibration Threshold $r_{T,\alpha}$}

\begin{proof}[Proof of Lemma \ref{lem:pgws-threshold-lower}]

Fix $\varepsilon\in(0,1)$ and set
$x_{\alpha,\varepsilon}
:=
(1-\varepsilon) \frac{1}{8\kappa_\nu^2} \alpha^2$. By Lemma~\ref{lem:Cdagger-tail},
\[
\mathbb P(C^\dagger(\mu)\le x_{\alpha,\varepsilon})
=
\kappa_\nu
\sqrt{8(1-\varepsilon) \frac{1}{8\kappa_\nu^2} }\,\alpha
+ o(\alpha)
=
\sqrt{1-\varepsilon}\,\alpha
+
o(\alpha).
\]
Hence, for sufficiently small $\alpha$,
$\mathbb P(C^\dagger(\mu)\le x_{\alpha,\varepsilon})
<
\alpha$.

Additionally, by Lemma~\ref{lem:RT-converges}, $\frac{R_T}{T}\to C^\dagger(\mu)$ almost surely.
Therefore we also have convergence in distribution. By the Portmanteau theorem applied to the closed set
\((-\infty,x_{\alpha,\varepsilon}]\),
\[
\limsup_{T\to\infty}
\mathbb P\left(\frac{R_T}{T}\le x_{\alpha,\varepsilon}\right)
\le
\mathbb P(C^\dagger(\mu)\le x_{\alpha,\varepsilon}) < \alpha.
\]
Hence for all sufficiently large $T$,
\[
\mathbb P(R_T\le x_{\alpha,\varepsilon}T)<\alpha.
\]
By the definition of $r_{T,\alpha}$ as the lower $\alpha$-quantile of $R_T$, this implies
\[
r_{T,\alpha}
\ge
x_{\alpha,\varepsilon}T
=
(1-\varepsilon) \frac{1}{8\kappa_\nu^2} \alpha^2T.  
\]

\end{proof}

\section{Proof of Theorem \ref{thm:bayesian-abstention-lower}} \label{pf_bayes_abs_lb}

We aim to lower bound the Bayes undetected error probability subject to an upper bound on the Bayes abstention probability. We begin by expressing these global probabilities over the full parameter (mean of the arms) space $\mathbb{R}^K$, i.e., 
$$\mathcal{A}_T(\pi) = \int_{\mathbb{R}^K} \mathbb{P}_\mu^\pi(\hat{a}=?)\,P_\nu(\dd\mu),\quad \mbox{and}\quad \mathcal{E}_T(\pi) = \int_{\mathbb{R}^K} \mathbb{P}_\mu^\pi\bigl(\hat{a}_T \notin \{a^\star(\mu), ?\}\bigr)\,P_\nu(\dd\mu).$$

Inspired by the form of the prior complexity $\kappa_\nu$, it is natural to partition the parameter space $\mathbb{R}^K$ according to the top pair of arms. For any pair $i < j$, define the region $\Omega_{ij}$ where arms $i$ and $j$ strictly dominate all the other arms, i.e., 
$$\Omega_{ij} := \left\{ \mu \in \mathbb{R}^K : \min\{\mu_i, \mu_j\}> \max_{k \neq i,j} \mu_k \right\}.$$

Because the boundaries between these regions correspond to tie events and hence have $P_\nu$-measure zero, the sets $\Omega_{ij}$ form a   partition of the parameter space up to \(P_\nu\)-null sets.  Thus, we can decompose the abstention and undetected error probability as:
$$\mathcal{A}_T(\pi) = \sum_{i<j} \int_{\Omega_{ij}} \mathbb{P}_\mu^\pi(\hat{a}=?)\,P_\nu(\dd\mu),\quad \mbox{and}\quad
\mathcal{E}_T(\pi) = \sum_{i<j} \int_{\Omega_{ij}} \mathbb{P}_\mu^\pi\bigl(\hat{a}_T \notin \{a^\star(\mu), ?\}\bigr)\,P_\nu(\dd\mu).$$

Ultimately, we want to lower bound the undetected error  $\mathcal{E}_T(\pi)$ using a Neyman--Pearson lemma (cf.\  \citet{testhypo2005}) style rearrangement. However, working directly with the restricted true prior $\mathbbm{1}_{\Omega_{ij}}(\mu)P_\nu(\dd\mu)$ presents significant technical hurdles.  This is overcome by the introduction of a submeasure,  which we call a {\em  flat local subprior}.  

\subsection{A Submeasure: Flat Local Subprior}

To resolve the difficulties of integrating over $\Omega_{ij}$, we construct a mathematically cleaner  measure $Q_\rho$ that is a submeasure of the true Gaussian prior $P_\nu$, where $\rho \in (0,1)$ is a fixed slack parameter. 

\begin{definition}[Local Change of Variable]
For a fixed pair of arms $i < j$, we reparameterize the mean vector $\mu \in \mathbb{R}^K$ into $(i,j,x,\delta,h,\zeta_{-ij})$, a ``local'' coordinate system defined by the top two arms, their gap, and the other arms. Specifically,
\begin{itemize}
    \item $x \in \mathbb{R}$: the midpoint of the top two arms' means.
    \item $\delta \ge 0$: the absolute gap between the top two arms.
    \item $h \in \{+1, -1\}$: the sign indicating which arm is   better ($h=+1$ for arm $i$, $h=-1$ for arm~$j$).
    \item $\zeta_{-ij} = (\zeta_k)_{k \neq i,j}$: the means of the remaining $K-2$ arms.
\end{itemize}
We consolidate the background variables, $i, j, x, \zeta_{-ij}$, into one tuple $\xi = (i, j, x, \zeta_{-ij})$. Given $\xi$, $\delta$, and $h$, the inverse mapping $\mu(\xi, \delta, h) \in \mathbb{R}^K$ is defined as: 
\begin{equation}
     \mu_i = x + h\frac{\delta}{2}, \quad \mu_j = x - h\frac{\delta}{2}, \quad \mbox{and}\quad\mu_k = \zeta_k \text{ for } k \neq i,j.\label{eqn:local_coord}
\end{equation}
Note the absolute value of the determinant of the Jacobian of the map $(x, \delta) \mapsto (\mu_i, \mu_j)$ is~$1$.
\end{definition}


\begin{remark} To see why the restricted true prior is difficult to work with, let $\psi: \mathbb{R}^K \to \mathbb{R}_{\ge 0}$ be any non-negative measurable function. Applying the local coordinate transformation in~\eqref{eqn:local_coord} to the true prior over the partitioned space yields:
$$
\begin{aligned}
&\int_{\mathbb{R}^K} \psi(\mu) P_\nu(\dd\mu)
= \sum_{i<j} \int_{\Omega_{ij}} \psi(\mu) p_\nu(\mu) \,\dd\mu \\
&= \sum_{i<j} \sum_{h \in \{+1, -1\}} \int_0^\infty \int_{\mathcal{D}_{ij}(\delta)} \psi\bigl(\mu(\xi, \delta, h)\bigr) \bigg[ f_i\Big(x + h\frac{\delta}{2}\Big) f_j\Big(x - h\frac{\delta}{2}\Big) \prod_{k \neq i,j} f_k(\zeta_k) \bigg] \dd\xi \, \dd\delta, 
\end{aligned}
$$
where $\dd\xi = \dd x \, \dd\zeta_{-ij}$, and the   domain of integration 
$$\mathcal{D}_{ij}(\delta) := \left\{ (x, \zeta_{-ij}) \in \mathbb{R}^{K-1} : \max_{k \neq i,j} \zeta_k < x - \frac{\delta}{2} \right\}$$
depends on the gap $\delta$.
The expansion of $\int_{\mathbb{R}^K} \psi(\mu) P_\nu(\dd\mu)$ above reveals the technical hurdles: the domain of integration $\mathcal{D}_{ij}(\delta)$ shrinks as the gap~$\delta$ grows, and the densities $f_i$ and $f_j$ fluctuate with both the gap $\delta$ and the sign $h$. To overcome these hurdles, we construct $Q_\rho$ to eliminate these dependencies by restricting $(x,\zeta_{-ij})$ to a smaller  set, and by lower bounding $f_i\left(x + h\frac{\delta}{2}\right) f_j \left(x - h\frac{\delta}{2}\right)$ with $(1-\rho)f_i(x)f_j(x)$, which  depends on neither $\delta$ nor $h$.  
\end{remark}

For $\ell > 0$ and $M > 0$, and for each pair $1\leq i < j \leq K$, define the {\em $(\ell,M)$-restricted set}
$$ \mathcal{C}_{ij}(\ell, M) := \left\{ (x, \zeta_{-ij}) : |x| \le M,\ |\zeta_k| \le M,\ \zeta_k \le x - \ell \text{ for all } k \neq i,j \right\}. $$
Let $\mathcal{C}_\rho = \cup_{i<j}\mathcal{C}_{ij}(\ell, M)$ denote the disjoint union of $\mathcal{C}_{ij}(\ell, M)$ over all arm pairs $i < j$. Recall that we defined
$$\kappa_\nu = \sum_{i\neq j}\int_{\mathbb R}f_i(x)f_j(x)\prod_{k\neq i,j}F_k(x) \ \dd x 
=2\sum_{i<j}\int_{-\infty}^{\infty}\left[\int_{-\infty}^{x}\cdots\int_{-\infty}^{x}f_i(x)f_j(x)\prod_{k\neq i,j}f_k(\zeta_k) \ \dd\zeta_{-ij}\right] \ \dd x.$$
By the monotone convergence theorem, for any  $\rho >0$, we can choose $\ell > 0$ sufficiently small and $M > 0$ sufficiently large such that:
$$ 2\sum_{i<j} \int_{\mathcal{C}_{ij}(\ell,M)} f_i(x)f_j(x)\prod_{k \neq i,j} f_k(\zeta_k) \,\dd\zeta_{-ij}\,\dd x \ge (1-\rho)\kappa_\nu. $$
We now define a   measure $\Lambda_\rho$ whose support is the set $\mathcal{C}_\rho$. Since $\mathcal{C}_\rho$ is the disjoint union of $\mathcal{C}_{ij}(\ell,M)$, we only need to specify $\Lambda_\rho$ for each pair $i<j$.  For the pair $i < j$, $\Lambda_\rho$ takes the form
\begin{equation} \label{eq:Lambda-rho}
\Lambda_\rho(\dd\xi) := (1-\rho) f_i(x)f_j(x) \prod_{k \neq i,j} f_k(\zeta_k) \,\dd x\,\dd\zeta_{-ij}
\end{equation}
Let $c_\rho := 2\Lambda_\rho(\mathcal{C}_\rho) \ge (1-\rho)^2\kappa_\nu$.

\begin{lemma}[Uniform   Lower Bound]
\label{lem:uniform-density}
There exists a gap $\delta_\rho \in (0,  2\ell)$ such that, for all $\delta \in [0, \delta_\rho]$,   $\xi \in \mathcal{C}_{ij}(\ell, M)$, and    $h \in \{+1, -1\}$, the original densities satisfy the multiplicative bound:
$$ (1-\rho)f_i(x)f_j(x)
\le
f_i\!\left(x + h\frac{\delta}{2}\right) f_j\!\left(x - h\frac{\delta}{2}\right). $$
Furthermore, for any $\delta \in [0, \delta_\rho]$, arms $i$ and $j$ are the  top two arms because $x-\frac{\delta}{2}>x-\ell\ge \zeta_k$. 
\end{lemma}
\begin{proof}
    For each $(i,j)$, because $\mathcal{C}_{ij}(\ell, M)$ is compact, the Gaussian densities defined on $\mathcal{C}_{ij}(\ell, M)$  are (strictly) positive and uniformly continuous. Because $\ell > 0$ separates the top pair of arms $i,j$ from the other arms $k \ne i,j$, we can choose a sufficiently small $\delta_\rho < 2\ell$ such that the continuity of $f_i$ and $f_j$ guarantees the multiplicative bound over the interval $[0,\delta_\rho]$. Since there are finitely many pairs, taking minimum of the pair-specific values of $\delta_\rho$ yields the desired result.  
\end{proof}

We now define a convenient submeasure to work with later. 

\begin{definition}[Flat Local Subprior]
\label{def:flat-local-subprior}
Define the {\em  flat local subprior} $Q_\rho$ as
\begin{equation}
\label{eq:Q-rho-definition}
Q_\rho(E)
:=
\sum_{i<j}\sum_{h\in\{+1,-1\}}
\int_{\mathcal C_{ij}(\ell,M)}
\int_0^{\delta_\rho}
\mathbbm 1_E\!\bigl(\mu(\xi,\delta,h)\bigr)
\,\dd\delta\,\Lambda_\rho(\dd\xi)\qquad\forall\, E\in\mathcal{B}(\mathbb{R}^K).
\end{equation}
\end{definition}


\begin{lemma}[Domination]
\label{lem:subprior-domination}
The constructed measure $Q_\rho$ in \eqref{eq:Q-rho-definition} is a submeasure of the true Gaussian prior $P_\nu$, i.e., $Q_\rho \le P_\nu$ everywhere.
\end{lemma}

\begin{proof}
For each $i<j$, let
$
\mathcal S_{ij}
:=
\left\{
\mu(\xi,\delta,h):
\xi\in\mathcal C_{ij}(\ell,M),\
\delta\in[0,\delta_\rho],\
h\in\{+1,-1\}
\right\},$
and let $\mathcal{S}:=\bigcup_{i<j}\mathcal S_{ij}$. For any Borel set $E\subseteq\mathbb R^K$, substituting the definition of
$\Lambda_\rho$ in~\eqref{eq:Lambda-rho} into \eqref{eq:Q-rho-definition} gives
\[
Q_\rho(E)
=
\sum_{i<j}\sum_{h\in\{+1,-1\}}
\int_{\mathcal C_{ij}(\ell,M)}
\int_0^{\delta_\rho}
\mathbbm 1_E\!\bigl(\mu(\xi,\delta,h)\bigr)
(1-\rho)f_i(x)f_j(x)
\prod_{k\neq i,j}f_k(\zeta_k)
\,\dd\delta\,\dd x\,\dd\zeta_{-ij}.
\]
For each fixed $i<j$ and $h$, the Jacobian determinant of the coordinate transformation in~\eqref{eqn:local_coord} has absolute value $1$.
Moreover, except on the null set $\{\delta=0\}$, the sign $h$ is uniquely
determined by $\mu_i-\mu_j$, and the top-two pair $\{i,j\}$ is uniquely
determined on $\mathcal{S}$. Hence $Q_\rho$ has Lebesgue density
\[
q_\rho\bigl(\mu(\xi,\delta,h)\bigr)
=
(1-\rho)f_i(x)f_j(x)
\prod_{k\neq i,j}f_k(\zeta_k)
\]
for Lebesgue-almost every
$\mu(\xi,\delta,h)\in\mathcal S_{ij}$. Per the definition in
\eqref{eq:Q-rho-definition}, $Q_\rho$ is supported on $\mathcal{S}$, hence we also have $q_\rho(\mu)=0$ for Lebesgue-almost every $\mu\notin\mathcal{S}$.

For $\mu=\mu(\xi,\delta,h)\in\mathcal S_{ij}$, the true prior density is
\[
p_\nu(\mu)
=
f_i\left(x+h\frac{\delta}{2}\right)
f_j\left(x-h\frac{\delta}{2}\right)
\prod_{k\neq i,j}f_k(\zeta_k).
\]
Lemma~\ref{lem:uniform-density} therefore gives
$q_\rho(\mu)\le p_\nu(\mu)$ for Lebesgue-almost every
$\mu\in\mathcal{S}$. The same inequality holds outside
$\mathcal{S}$ because $q_\rho=0$ there. Consequently, for every
Borel set $E\subseteq\mathbb R^K$,
\[
Q_\rho(E)
=
\int_E q_\rho(\mu)\,\dd\mu
\le
\int_E p_\nu(\mu)\,\dd\mu
=
P_\nu(E).
\]
Hence $Q_\rho\le P_\nu$.   
\end{proof}
Figure \ref{fig:subprior-domination} provides a visual intuition for this result, illustrating how the constructed flat subprior bounds the true Gaussian prior from below over the compact gap interval.

\begin{figure}[t]
    \centering
\begin{tikzpicture}[x=9cm,y=5cm]

\def\deltarho{0.45}
\def\yflat{0.30}
\pgfmathsetmacro{\yrho}{0.92*exp(-3.4*\deltarho*\deltarho)}

\draw[->, thick] (0,0) -- (1.18,0);
\draw[->, thick] (0,0) -- (0,1.02);

\fill[orange, opacity=0.9] (0,0) rectangle (\deltarho,\yflat);
\draw[orange, very thick] (0,\yflat) -- (\deltarho,\yflat);

\draw[blue, very thick, domain=0:1.08, samples=200]
    plot (\x,{0.92*exp(-3.4*\x*\x)});

\draw[black, dotted, very thick] (0,\yrho) -- (\deltarho,\yrho);

\draw[gray, dashed, thick] (\deltarho,0) -- (\deltarho,\yrho);

\node[below] at (0,0) {$0$};
\node[below] at (\deltarho,0) {$\delta_\rho$};

\node[left] at (-0.01,\yrho) {$p_\nu(\delta_\rho)$};
\node[left] at (-0.01,\yflat) {$(1-\rho)p_\nu(\delta_\rho)$};

\node[below, font=\normalsize] at (0.55,-0.12) {Local Pairwise Gap $(\delta)$};
\node[rotate=90, font=\normalsize] at (-0.33,0.50) {Density (Unnormalized)};

\node[font=\normalsize] at (0.23,0.15) {$q_\rho\le p_\nu$};

\draw[->, thick] (0.20,\yflat) -- (0.20,\yrho);
\node[right, font=\normalsize] at (0.21,{0.5*(\yflat+\yrho)}) {$\rho\downarrow 0$};

\draw[blue, very thick] (0.68,0.94) -- (0.76,0.94);
\node[right, font=\normalsize] at (0.78,0.94) {True Prior $p_\nu(\delta)$};

\draw[orange, very thick] (0.68,0.86) -- (0.76,0.86);
\node[right, font=\normalsize] at (0.78,0.86) {Flat Subprior $q_\rho(\delta)$};

\draw[black, dotted, very thick] (0.68,0.78) -- (0.76,0.78);
\node[right, font=\normalsize] at (0.78,0.78) {Limit $\lim_{\rho\to 0} q_\rho$};

\end{tikzpicture}
    \caption{The true Gaussian prior $P_\nu$ is bounded from below by the flat subprior $Q_\rho$ on the interval $[0, \delta_\rho]$. In the above, for brevity, we write the density  as $p_\nu(\delta):=p_\nu(\mu(\xi,\delta,h))$ and similarly for $q_\rho$.}
    \label{fig:subprior-domination}
\end{figure}
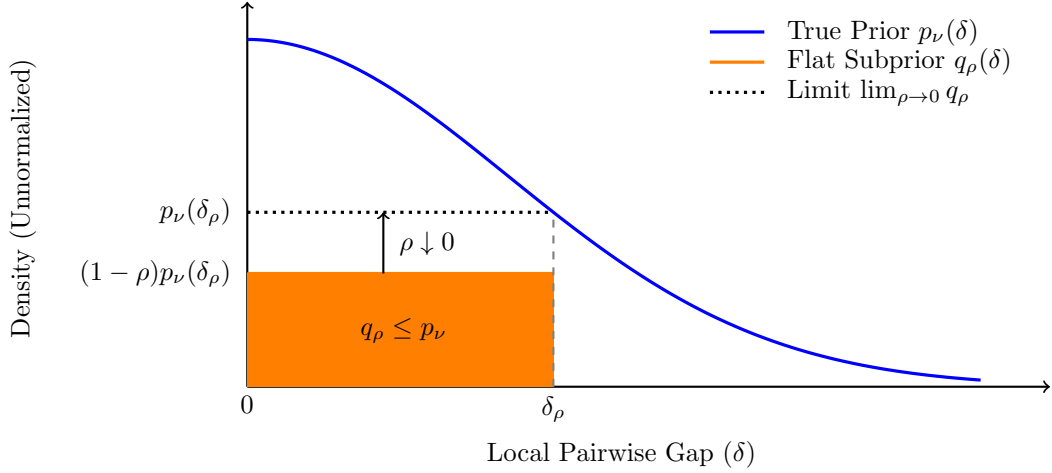

\begin{lemma}[Properties of $Q_\rho$]
\label{lem:q_rho_properties}
Let $\Xi$ (background oracle), $\Delta$ (gap), and $H$ (sign) denote  random quantities under   $Q_\rho$. The construction of $Q_\rho$ results in the following joint distribution:
\begin{equation}
Q_\rho(\dd\mu) =Q_\rho(\Xi\in \dd\xi,\Delta\in \dd\delta,H=h) = \Lambda_\rho(\dd\xi)\,\dd\delta
, \qquad h\in\{+1,-1\}.\label{eqn:Q_rho_dist}    
\end{equation}
Because $Q_\rho$ is not a probability measure in general, we refer to the $Q_\rho$ measure of a set as its {\em $Q_\rho$-mass} in the following. This measure has two important properties:
\begin{enumerate}
    \item \textbf{Fair Sign:} Conditional on any realization $(\Xi=\xi, \Delta=\delta)$, the two signs have equal $Q_\rho$-mass, i.e., $Q_\rho(H=+1 \mid \xi, \delta) = Q_\rho(H=-1 \mid \xi, \delta)$.
    \item \textbf{Flat Gap:} The marginal unnormalized density of the gap is uniform, i.e., $Q_\rho(\Delta \in \dd\delta) = c_\rho\,\dd\delta$ for $\delta \in [0, \delta_\rho]$.
\end{enumerate}
\end{lemma}

\begin{proof}
By the definition of $Q_\rho$, the measure on the product space $\mathcal{C}_\rho \times [0, \delta_\rho] \times \{+1, -1\}$ factors as $\Lambda_\rho(\dd\xi) \otimes \dd\delta \otimes \mathbbm{1}_{\{h\}}$. Because this is a   product measure that does not depend on $h$, the sign $H$ occurs with equal probability $1/2$ regardless of $\Xi$ or $\Delta$. Furthermore, integrating out $\Xi$ and $H$ leaves the measure of the gap $\Delta$ as $2\Lambda_\rho(\mathcal{C}_\rho)\, \dd\delta = c_\rho \,\dd\delta$, proving the flat gap property. 
 
\end{proof}

\begin{lemma}[Probability Domination Under $Q_\rho$]\label{lem:risk-domination}

For any policy $\pi$, the Bayes probability evaluated under the subprior $Q_\rho$   is a lower bound of  the probability under the original prior $P_\nu$. Specifically, if a strategy $\pi$ satisfies the abstention condition $\mathcal{A}_T(\pi) \le \alpha$ under $P_\nu$, it also satisfies:
$$\mathcal{A}_T^{Q_\rho}(\pi) := \int_{\mathbb{R}^K} \mathbb{P}_\mu^\pi(\hat{a}_T=? )\,Q_\rho(\dd\mu) \le \mathcal{A}_T(\pi) \le \alpha.$$
In addition, the true undetected error is lower bounded by the error evaluated under the subprior:
$$\mathcal{E}_T(\pi) \ge \mathcal{E}_T^{Q_\rho}(\pi) := \int_{\mathbb{R}^K} \mathbb{P}_\mu^\pi(\hat{a}_T\notin\{a^\star,?\})\,Q_\rho(\dd\mu).$$
\end{lemma}

\begin{proof}
This follows directly from Lemma~\ref{lem:subprior-domination}.  
\end{proof}
This lemma explains how to  use  $Q_\rho$. Any strategy satisfying $\mathcal A_T(\pi)\le\alpha$ under the original prior also has $Q_\rho$-abstention mass at most $\alpha$. Consequently, to lower bound the original Bayes undetected error probability, it suffices to lower bound the error contributed by the flat subprior $Q_\rho$.

\subsection{Oracle Reduction: Lower Bounding via a 2-Arm Hypothesis Testing Problem} 
\label{subsec:local-oracle-reduction}

Next, we evaluate the probabilities under the local subprior $Q_\rho$. To isolate the fundamental difficulty of the problem, we reduce the $K$-arm task to a $2$-arm hypothesis testing problem by giving the learner additional side (oracle) information. 

\paragraph{Revealing Oracle Information.}
Before sampling begins, we reveal the \emph{background oracle} $\Xi = (I, J, X, \zeta_{-IJ})$, which depends on the random means $\mu$, to the learner. For a specific realization $\Xi = \xi = (i, j, x, \zeta_{-ij})$, this means the learner knows the identities of the top two arms, their midpoint $x$, and the   means $\zeta_{-ij}$ of all remaining background arms.

Conditional on any realization of the background oracle $\Xi = \xi$, the only remaining random variables relevant to the decision under $Q_\rho$ are the gap magnitude and the sign, i.e.,
$$\Delta \in [0, \delta_\rho], \quad\mbox{and}\quad H \in \{+1, -1\}.$$
For ease of notation, we define the unknown signed gap $\Theta := H\Delta$. Let $\mathbb{P}_\theta^\xi$ denote the probability law of the experiment conditional on the background oracle $\Xi = \xi$ and the signed gap $\Theta = \theta$.

Since any policy in the original experiment can be implemented in the oracle-revealed setting by simply ignoring the side information $\Xi$, any lower bound derived for this oracle-revealed experiment provides a lower bound for the original problem. Furthermore, since the background means are known, querying any arm other than $i$ or $j$ provides no information about the unknown gap $\Theta$. We show that we can restrict our analysis to strategies that only sample the top two arms and output the decision of the sign; we call these {\em local pair strategies}.

\begin{lemma}[Sufficiency of Local Pair Strategies]
\label{lem:sufficiency-local-pair-strategies}
Fix a realization of the background oracle $\Xi=\xi=(i,j,x,\zeta_{-ij})$. For any adaptive strategy $\pi$, there exists a modified strategy $\pi'$ that samples only arms $i$ and $j$, restricts its terminal recommendation to $\hat{a}'_T \in \{+,-,?\}$, and satisfies for every signed-gap realization $\theta \in [-\delta_\rho, \delta_\rho] \setminus \{0\}$ with true sign $h = \text{sgn}(\theta) \in \{+,-\}$, i.e., 
\begin{equation}
    \mathbb{P}_\theta^\xi(\hat{a}'_T = ?) = \mathbb{P}_\theta^\xi(\hat{a}_T = ?) , \quad\mbox{and}\quad
\mathbb{P}_\theta^\xi(\hat{a}'_T = -h) \le \mathbb{P}_\theta^\xi(\hat{a}_T \notin \{a^\star(h), ?\}), \label{eqn:simulate}
\end{equation}
where $a^\star(+) = i$ and $a^\star(-) = j$. This says that the abstention probability of $\pi'$ is unchanged from that of $\pi$ and the undetected error probability of $\pi'$ is less than or equal to that of $\pi$.
\end{lemma}

\begin{proof}
Fix $\xi = (i, j, x, \zeta_{-ij})$. We construct a modified strategy $\pi'$ from $\pi$ that possesses the desired properties in \eqref{eqn:simulate}. 

When $\pi$ samples from an arm $k \in \{i, j\}$, $\pi'$ queries that same arm and feeds the observed reward back to $\pi$. Whenever $\pi$ chooses to sample from an arm $k \notin \{i,j\}$, $\pi'$ instead queries arm $i$, discards the true reward, and feeds $\pi$ an independent simulated reward drawn from $\mathcal{N}(\zeta_k, 1)$.

Because $\Xi = \xi$ is known, $\zeta_k$ is the true mean of arm $k$. Furthermore, for $k\ne i,j$, the reward distributions of arm $k$ are independent of the unknown signed gap $\Theta$. Therefore, the simulated reward from $\pi'$ matches that from the true distribution of arm $k$. By induction, the internal history observed by $\pi$ under $\pi'$ is identically distributed to the true history in the original experiment. Consequently, the terminal distribution of $\pi$ is preserved for every $\theta$.

When $\pi$ stops and outputs a recommendation $\hat{a}_T$, $\pi'$ maps this to a sign decision $\hat{a}'_T \in \{+,-,?\}$ as follows: $i \mapsto +$, $j \mapsto -$, $? \mapsto ?$, and any $k \notin \{i, j, ?\}$ is mapped arbitrarily to $+$. This map preserves the abstention probability. For the undetected error, recall that on the support of $Q_\rho$, arms $i$ and $j$ are strictly the top two arms. Thus, any original output $k \notin \{i, j, ?\}$ is  an error. If $h = -$, mapping $k$ to $+$ yields $\hat{a}'_T = +$, which is an incorrect sign, thereby preserving the error probability. If $h = +$, mapping $k$ to $+$ converts an original error into a correct guess,  reducing the error probability. Thus, across all cases, the probability of an incorrect sign under $\pi'$ is less than or equal to the probability of an undetected error under $\pi$, showing~\eqref{eqn:simulate}.  
\end{proof}

\newcommand{\sgn}{\mathrm{sgn}}

\subsection{Transformed Observations and Sufficient Statistics:}
\label{subsec:local-transformed-observations}

To establish clean analytical expressions for the abstention and undetected error probabilities, we now demonstrate that any decision rule acting on the full history of this restricted two-arm game can be reduced to a randomized decision rule acting on a   scalar sufficient statistic, without altering any terminal-output probability.

Fix the revealed background oracle realization $\Xi=\xi=(i,j,x,\zeta_{-ij})$, and consider strategies that sample only from arms $i$ and $j$. Define the sign indicator for the arms: $\sgn(i):=+1$ and $\sgn(j):=-1$. At time $t$, if the learner samples $A_t\in\{i,j\}$ and observes the reward $X_t$, we define the {\em transformed observation}: $Z_t:=\sgn(A_t)(X_t-x).$


For any realization of the signed gap $\Theta=H\Delta=\theta$, the means of the top two arms are:
$$\mu_i=x+\frac{\theta}{2}, \quad\mbox{and}\quad \mu_j=x-\frac{\theta}{2}.$$

Let $\mathbb{P}_\theta^\xi$ denote the probability law of the experiment conditional on $\Xi=\xi$ and $\Theta=\theta$. Furthermore, let:
$\mathcal{H}_t^\xi:=(\xi,A_1,Z_1,\dots,A_t,Z_t)$
denote the \emph{transformed oracle history} up to time $t$.

\begin{lemma}[Transformed Observations are i.i.d.]
\label{lem:local-transformed-iid}
Fix an oracle realization $\Xi=\xi$ and a signed-gap realization $\theta\in[-\delta_{\rho},\delta_{\rho}]$. Under $\mathbb{P}_\theta^\xi$, conditionally on $\mathcal{H}_{t-1}^\xi$ and the chosen arm $A_t\in\{i,j\}$, the transformed observation satisfies:
$Z_t \sim \mathcal{N}\!\left(\frac{\theta}{2},1\right).$
i.e., under the full joint Bayesian law, 
\begin{equation}
    Z_t \mid \mathcal{H}_{t-1}^{\xi},A_t=i,\Theta=\theta,\Xi=\xi \sim \mathcal{N}\!\left(\frac{\theta}{2},1\right). \label{eqn:Z_t}
\end{equation}
Because this conditional law does not depend on the chosen arm $A_t$, the sequence $Z_1,\dots,Z_T$ is identically and independently distributed as $\mathcal{N}(\theta/2,1)$ under $\mathbb{P}_\theta^\xi$.
\end{lemma}

\begin{proof}
If $A_t=i$, the true observation follows
$X_t \mid \mathcal{H}_{t-1}^{\xi},A_t=i,\Theta=\theta,\Xi=\xi \sim \mathcal{N}\!\left(x+\frac{\theta}{2},1\right).$
Applying the transformation $Z_t=X_t-x$ yields~\eqref{eqn:Z_t}.
While if $A_t=j$, the true observation follows $X_t \mid \mathcal{H}_{t-1}^{\xi},A_t=j,\Theta=\theta,\Xi=\xi \sim \mathcal{N}\!\left(x-\frac{\theta}{2},1\right)$. Similarly, applying the transformation yields~\eqref{eqn:Z_t}.
Thus, the conditional law of $Z_t$ is  $\mathcal{N}(\theta/2,1)$ regardless of which of the two arms is chosen.

Next we prove independence. Let $\gamma_\theta$ denote the measure corresponding to $\mathcal{N}(\theta/2,1)$. For any Borel set $B$, we have:
$$\mathbb{P}_\theta^{\xi}(Z_t\in B\mid \mathcal{H}_{t-1}^{\xi},A_t)=\gamma_\theta(B).$$
Iterating this identity using the tower property shows that  for any sequence of Borel sets $B_1,\dots,B_T$:
$$\mathbb{P}_\theta^{\xi}(Z_1\in B_1,\dots,Z_T\in B_T)=\prod_{t=1}^T\gamma_\theta(B_t).$$
Hence, given $\xi$ and $\theta$, $Z_1,\dots,Z_T$ are i.i.d.\ over the transformed history.  
\end{proof}

We consolidate the sequence of observations into a sum and its time-scaled counterpart. Define:
$$W:=\sum_{t=1}^T Z_t, \quad\mbox{and}\quad V:=\frac{W}{\sqrt{T}}.$$

\begin{lemma}[Likelihood Ratios Depend Only on $W$]
\label{lem:local-lr-W}
For every $\xi$ and every $\theta\in[-\delta_{\rho},\delta_{\rho}]$, the measure $\mathbb{P}_\theta^{\xi}$ is absolutely continuous with respect to $\mathbb{P}_0^{\xi}$ (i.e., $\mathbb{P}_\theta^{\xi}\ll \mathbb{P}_0^{\xi}$), and the likelihood ratio of the transformed oracle history $\mathcal{H}_T^{\xi}$ is:
$$\frac{\dd\mathbb{P}_\theta^{\xi}}{\dd\mathbb{P}_0^{\xi}}(\mathcal{H}_T^{\xi})
=\exp\!\left(\frac{\theta}{2}W-\frac{T\theta^2}{8}\right)
=\exp\!\left(\frac{\theta\sqrt{T}}{2}V-\frac{T\theta^2}{8}\right).$$
Consequently, $W$ (and equivalently $V$) is a sufficient statistic for $\theta$.
\end{lemma}

\begin{proof}
Let $h_T=(\xi,a_1,z_1,\dots,a_T,z_T)$ be the realized transformed oracle history, and let $\phi$ be the standard normal density as usual. Under $\mathbb{P}_\theta^{\xi}$, the density of $h_T$ factors as:
$$p_\theta^{\xi}(h_T)=\left[\prod_{t=1}^T\pi_t(a_t\mid \xi,a_1,z_1,\dots,a_{t-1},z_{t-1})\right]\prod_{t=1}^T\phi\!\left(z_t-\frac{\theta}{2}\right),$$
where $\pi_t$ is the strategy's sampling kernel at time $t$. Because the action probabilities are identical functions of the realized history under both $\mathbb{P}_\theta^{\xi}$ and $\mathbb{P}_0^{\xi}$, they cancel   in the likelihood ratio, i.e., 
$$\frac{p_\theta^{\xi}(h_T)}{p_0^{\xi}(h_T)}=\prod_{t=1}^T\frac{\phi(z_t-\theta/2)}{\phi(z_t)}=\prod_{t=1}^T\exp\!\left(\frac{\theta}{2}z_t-\frac{\theta^2}{8}\right).$$
Aggregating the sum inside the exponential yields:
$$\frac{p_\theta^{\xi}(h_T)}{p_0^{\xi}(h_T)}=\exp\!\left(\frac{\theta}{2}\sum_{t=1}^T z_t-\frac{T\theta^2}{8}\right)=\exp\!\left(\frac{\theta}{2}W-\frac{T\theta^2}{8}\right).  $$

\end{proof}

\begin{lemma}[Sufficiency of $V$ for Terminal Decisions]
\label{lem:local-reduce-to-V}
Fix the revealed background oracle realization $\Xi=\xi$. Let $\tilde{\psi}^\xi = (\tilde{\psi}_+^\xi, \tilde{\psi}_-^\xi, \tilde{\psi}_?^\xi)$ be any randomized terminal rule based on the transformed oracle history $\mathcal{H}_T^{\xi}$, and let $\hat{a} \in \{+,-,?\}$ be its random output. There exists a modified terminal rule $\psi^\xi$ depending on the history only through $V$, with random output $\hat{a}^V \in \{+,-,?\}$, such that for every $\theta\in[-\delta_{\rho},\delta_{\rho}]$ and every decision $d\in\{+,-,?\}$:
$$ \mathbb{P}_\theta^{\xi}(\hat{a}^V=d) = \mathbb{P}_\theta^{\xi}(\hat{a}=d). $$
\end{lemma}

\begin{proof}
By Lemma~\ref{lem:local-lr-W}, the likelihood ratio $\dd\mathbb{P}_\theta^{\xi}/\dd\mathbb{P}_0^{\xi}$ is   $\sigma(W)$-measurable. Since $\sigma(V)=\sigma(W)$, it is also $\sigma(V)$-measurable.

Let $\tilde{\psi}_d^\xi(h_T)$ denote the probability that the original terminal rule outputs decision $d\in\{+,-,?\}$ after observing transformed history $h_T$. Let $\mathbb{Q}^{\xi}(v,\dd h_T)$ be a regular conditional distribution of $\mathcal{H}_T^{\xi}$ given $V=v$ under $\mathbb{P}_0^{\xi}$. i.e., $\mathbb{Q}^{\xi}(V,B)$ is a version of $\mathbb{P}_0^{\xi}\left( \mathcal{H}_T^{\xi} \in B \mid V \right)$.

By Lemma~\ref{lem:local-lr-W}, $\mathbb{P}^\xi_{\theta}(V \in dv)
=\exp\!\left(\frac{\theta\sqrt{T}}{2}v-\frac{T\theta^2}{8}\right) \mathbb{P}^\xi_{0}(V \in dv)$. Therefore, for every measurable set $B$ and $C \in \mathcal B(\mathbb R)$, $$
\begin{aligned}
    \mathbb{P}_\theta^{\xi}\left(\mathcal{H}_T^{\xi} \in B , V \in C \right) &= \mathbb E^\xi_0 \left[ \exp\!\left(\frac{\theta\sqrt{T}}{2}V-\frac{T\theta^2}{8}\right) \mathbbm{1}_{\{\mathcal{H}_T^{\xi} \in B\}} \mathbbm{1}_{\{V \in C\}} \right] \\
    &= \int_C \exp\!\left(\frac{\theta\sqrt{T}}{2}v-\frac{T\theta^2}{8}\right) Q^\xi(v,B) \mathbb P^\xi_{0}(V \in \dd v) \\
    &= \int_C Q^\xi(v,B) \mathbb P^\xi_{\theta}(V \in \dd v).
\end{aligned}
$$
This implies   that the kernel $\mathbb{Q}^{\xi}(v,\dd h_T)$ is a conditional law under   $\mathbb{P}_\theta^{\xi}$. Hence, for every bounded measurable $f$, 
\begin{equation}\label{RCD_Q_xi}
  \mathbb E^\xi_\theta \big[ f(\mathcal H^\xi_T) \mid V \big]  = \int f(h_T) Q^\xi(V,\dd h_T) \quad \mathbb P^\xi_\theta\text{-a.s.,}
\end{equation}

Define the terminal rule $\psi^\xi = (\psi_+^\xi, \psi_-^\xi, \psi_?^\xi)$ via the kernel  $\mathbb{Q}^{\xi}(v,\dd h_T)$ as follows:
$$ \psi_d^\xi(v) := \int \tilde{\psi}_d^\xi(h_T)\,\mathbb{Q}^{\xi}(v,\dd h_T), \qquad d\in\{+,-,?\}. $$
By definition, $\psi_{(\cdot)}^\xi(v)$ is a probability distribution on $\{+,-,?\}$. For every signed-gap realization $\theta$ and every decision $d \in\{+,-,?\}$, taking $f=\tilde{\psi}_d^\xi$ in~\eqref{RCD_Q_xi}, the tower property guarantees:
$$ \mathbb{P}_\theta^{\xi}(\hat{a}^V=d) = \mathbb{E}_\theta^{\xi}[\psi_d^\xi(V)] = \mathbb{E}_\theta^{\xi}\big[\mathbb{E}_\theta^{\xi}[\tilde{\psi}_d^\xi(\mathcal{H}_T^{\xi})\mid V]\big] = \mathbb{E}_\theta^{\xi}[\tilde{\psi}_d^\xi(\mathcal{H}_T^{\xi})] = \mathbb{P}_\theta^{\xi}(\hat{a}=d). $$
Thus, the  $V$-based rule   preserves all  probabilities under every   realization of the signed-gap~$\theta$.  
\end{proof}

\subsection{The Resulting Bayesian Flat Sign Experiment}
\label{subsec:flat-sign-experiment}

We are now ready to  express the abstention and undetected error probability by integrating over the local subprior $Q_\rho$. 

Recall from Lemma~\ref{lem:q_rho_properties} that on the local coordinate space, the subprior measure factors as in~\eqref{eqn:Q_rho_dist}.
Under the  measure $Q_\rho$, conditional on $(\Xi,\Delta)$, the two signs have equal mass. The unnormalized marginal measure of the gap is $Q_\rho(\Delta\in \dd\delta) = c_\rho\,\dd\delta$, where $c_\rho = 2\Lambda_\rho(\mathcal{C}_\rho)$.

Define the  rescaled gaps:
$G := \frac{\Delta\sqrt{T}}{2}$ and $L_T := \frac{\delta_{\rho}\sqrt{T}}{2}$.
Under $Q_\rho$, the support of the rescaled gap $G$ is  $[0, L_T]$. For any realization $G=g$ on the interval $[0, L_T]$, the corresponding absolute gap is $\delta = \frac{2g}{\sqrt{T}}$, yielding the differential change of variables as follows:
$\dd\delta = \frac{2}{\sqrt{T}}\,\dd g.$

Applying Lemma~\ref{lem:local-transformed-iid}, the conditional distribution of $V$ given the rescaled gap and sign is Gaussian and independent of the background oracle $\Xi$, namely, 
\begin{equation}
V \mid (G=g, H=h, \Xi=\xi) \sim \mathcal{N}(hg, 1).
\end{equation}


This reduces the given experiment to one that has the following properties under $Q_\rho$:
\begin{itemize}
    \item The unnormalized marginal density of $G$ is constant on $[0, L_T]$ (and zero elsewhere).
    \item $H \in \{+1, -1\}$ is uniform.
    \item $V \mid (G=g, H=h) \sim \mathcal{N}(hg, 1)$.
\end{itemize}

As established in Lemma~\ref{lem:local-reduce-to-V}, any terminal decision can be represented by a $V$-based rule. Thus it suffices to consider rules that only depend on $V$. We represent this rule via the tuple of measurable functions $\psi^\xi = (\psi_+^\xi, \psi_-^\xi, \psi_?^\xi)$, mapping the scalar $v$ to the probability simplex on $\{+,-,?\}$. In other words,
$\psi_+^\xi(v) + \psi_-^\xi(v) + \psi_?^\xi(v) = 1.$
The superscript $\xi$ indicates that the decision rule is allowed to depend on the revealed background oracle~$\xi$. It follows that we can write 
\begin{equation} \label{integral_of_abs_prob}
    \mathbb P^\xi_{h\delta}\left( \hat{a}^V=? \right) 
    = \mathbb E^\xi_{h\delta}\left[ \psi^\xi_?(V) \right] 
    = \int_{\mathbb R} \psi^\xi_?(v) \phi\left( v - h\frac{\delta\sqrt{T}}{2} \right) \dd v
\end{equation}


For ease of notation, we define the unnormalized marginal density of $V$ corresponding to ($H=+1$) and ($H=-1$) sign realizations as follows:
$$A_T(v) := \int_0^{L_T} \phi(v-g)\,\dd g, \quad\mbox{and}\quad B_T(v) := \int_0^{L_T} \phi(v+g)\,\dd g.$$
Then we can express the total $Q_\rho$-abstention mass of any rule $\psi$ as:
$$
\begin{aligned}
\mathcal{A}_T^{Q_\rho}(\psi)
&= \int_{\mathbb R^K} \mathbb P_\mu\left(\hat{a}^V=?\right) Q_\rho(\dd\mu) \\
&\stackrel{(a)}{=} \int_{\mathcal{C}_\rho} \sum_{h \in \{+1,-1\}} \int^{\delta_\rho}_0 \mathbb P^\xi_{h\delta}\left(\hat{a}^V=?\right) \, \dd\delta \, \Lambda_\rho(\dd\xi) \\
&\stackrel{(b)}{=} \int_{\mathcal{C}_\rho} \sum_{h \in \{+1,-1\}} \int^{\delta_\rho}_0 \int_{\mathbb R} \psi^\xi_?(v) \phi\left(v - h\frac{\delta\sqrt{T}}{2}\right) \,\dd v\, \dd\delta\,  \Lambda_\rho(\dd\xi) \\
&\stackrel{(c)}{=} \int_{\mathcal{C}_\rho} \Lambda_\rho(\dd\xi) \sum_{h \in \{+1,-1\}} \int_0^{L_T} \left( \int_{\mathbb{R}} \psi_?^\xi(v)\phi(v-hg)\,\dd v \right) \frac{2}{\sqrt{T}}\,\dd g \\
&= \frac{2}{\sqrt{T}} \int_{\mathcal{C}_\rho} \int_{\mathbb{R}} \psi_?^\xi(v) \bigl[A_T(v) + B_T(v)\bigr] \,\dd v\,\Lambda_\rho(\dd\xi),
\end{aligned}
$$
where $(a)$ is due to the fair sign property and factorization of Lemma~\ref{lem:q_rho_properties}, $(b)$ is due to \eqref{integral_of_abs_prob} and $(c)$ is due to change of variables with $\dd\delta = \frac{2}{\sqrt{T}}\,\dd g.$ 

Similarly, an undetected error occurs when the rule outputs $-$ when $H=+1$, or outputs $+$ when $H=-1$. Hence the $Q_\rho$-undetected error mass can be expressed as:
$$\mathcal{E}_T^{Q_\rho}(\psi) = \frac{2}{\sqrt{T}} \int_{\mathcal{C}_\rho} \int_{\mathbb{R}} \left[ \psi_-^\xi(v)A_T(v) + \psi_+^\xi(v)B_T(v) \right] \,\dd v\,\Lambda_\rho(\dd\xi).$$

\subsection{Solving the Bayesian Flat Sign Experiment}
\label{subsec:solve-flat-sign-experiment}

We are now ready to solve the core optimization problem: establishing a lower bound on the undetected error probability $\mathcal{E}_T^{Q_\rho}(\psi)$ for any terminal decision rule $\psi$ that satisfies the abstention budget $\mathcal{A}_T^{Q_\rho}(\psi) \le \alpha$.

\begin{lemma}[Pointwise Optimal Sign Decision]
\label{lem:flat-pointwise-sign}
For every terminal decision rule $\psi$, the undetected error mass is bounded below as:
$$ \mathcal{E}_T^{Q_\rho}(\psi) \ge \frac{2}{\sqrt{T}} \int_{\mathcal{C}_\rho} \int_{\mathbb{R}} \bigl(1-\psi_?^\xi(v)\bigr) \min\{A_T(v),B_T(v)\} \,\dd v\,\Lambda_\rho(\dd\xi). $$
\end{lemma}

\begin{proof}
Fix $(\xi,v)$. Since $\psi_+^{\xi}(v)+\psi_-^{\xi}(v)=1-\psi_?^{\xi}(v)$, and both $\psi_+^{\xi}(v)$ and $\psi_-^{\xi}(v)$ are non-negative, $$ \psi_-^{\xi}(v)A_T(v)+\psi_+^{\xi}(v)B_T(v) \ge \bigl(\psi_-^{\xi}(v)+\psi_+^{\xi}(v)\bigr)\min\{A_T(v),B_T(v)\} = \bigl(1-\psi_?^{\xi}(v)\bigr)\min\{A_T(v),B_T(v)\}. $$
Integrating over $v$ and $\xi$ proves the claim.  
\end{proof}

\begin{lemma}[Uniform Central Bound]
\label{lem:flat-central-approx}
Fix $\eta\in(0,1)$. For all sufficiently large $T$, the following bounds hold   for all $|v|\le \frac{L_T}{2}$:
$$
    A_T(v)+B_T(v)  \ge 1-\eta, \quad\mbox{and}\quad
  \min\{A_T(v),B_T(v)\} \ge (1-\eta)\Phi(-|v|). 
$$ 
\end{lemma}

\begin{proof}
First, evaluating the integrals of the standard normal density $\phi$ yields:
$$  A_T(v) = \int_0^{L_T}\phi(v-g)\,\dd g = \Phi(v) - \Phi(v-L_T),\;\;\mbox{and}\;\;
B_T(v) = \int_0^{L_T}\phi(v+g)\,\dd g = \Phi(v+L_T) - \Phi(v).$$
Adding these yields:
$$A_T(v) + B_T(v) = \Phi(v+L_T) - \Phi(v-L_T).$$
For $|v| \le L_T/2$, this domain restriction implies $v+L_T \ge L_T/2$ and $v-L_T \le -L_T/2$. Since $\Phi$ is strictly increasing, we can lower bound the sum by evaluating it at these extremes:
$$A_T(v) + B_T(v) \ge \Phi(L_T/2) - \Phi(-L_T/2) = 1 - 2\Phi(-L_T/2).$$
As $T \to \infty$, $L_T \to \infty$, so $1 - 2\Phi(-L_T/2) \to 1$. Thus, for sufficiently large $T$, the sum is bounded below by $1-\eta$.

For the second inequality, observe that $\phi$ is symmetric and strictly decreasing away from zero. Thus, the relationship between $A_T(v)$ and $B_T(v)$ depends only on the sign of $v$. For $v \ge 0$, the integration interval for $A_T(v)$ is $[v-L_T, v]$, which is closer to the origin than the interval $[v, v+L_T]$ for $B_T(v)$. Hence, $A_T(v) \ge B_T(v)$, yielding:
$$
\min\{A_T(v), B_T(v)\} = B_T(v) 
= \Phi(v+L_T) - \Phi(v)
= \Phi(-v) - \Phi(-(v+L_T))
$$
since $\Phi(x) = 1 - \Phi(-x)$.
By identical symmetric reasoning for $v < 0$, we find $\min\{A_T(v), B_T(v)\} = A_T(v)$. Consolidating both cases yields a single identity for all $v$:
\begin{equation}\label{eq:min_AB_with_Phi}
    \min\{A_T(v), B_T(v)\} = \Phi(-|v|) - \Phi(-(|v|+L_T)).
\end{equation}
We wish to show that $\Phi(-|v|) - \Phi(-(|v|+L_T)) \ge (1-\eta)\Phi(-|v|)$, which is equivalent to bounding the ratio:
$$\frac{\Phi(-(|v|+L_T))}{\Phi(-|v|)} \le \eta.$$
Since $0 \le |v| \le L_T/2$, the numerator is bounded above by $\Phi(-L_T)$ and the denominator is bounded below by $\Phi(-L_T/2)$. Therefore:
$$\frac{\Phi(-(|v|+L_T))}{\Phi(-|v|)} \le \frac{\Phi(-L_T)}{\Phi(-L_T/2)} \to 0$$
uniformly as $T \to \infty$. Thus, the bound holds for sufficiently large $T$.  
\end{proof}

To complete the lower bound, we use the following specialized consequence of
the \emph{Bathtub Principle} (cf. Theorem~1.14 of \citet{Lieb2001Analysis}). The following lemma states the form of this principle needed for our subsequent analysis.


\begin{lemma}[Abstention Rearrangement / Bathtub Principle]
\label{lem:flat-rearrangement}
Let $B>0$, let $\Lambda$ be a finite measure on a space $\mathcal{O}$, and let $w(v)\geq0$ be an even function that is non-increasing in $|v|$ on $[-B,B]$. Let $r:\mathcal{O} \times \mathbb{R} \to [0,1]$ be measurable. If there exists some $a\in[0,B]$ such that:
$$ \int_{\mathcal{O}}\int_{-B}^{B}r(\xi,v)\,\dd v\,\Lambda(\dd\xi) \le 2a\,\Lambda(\mathcal{O}), $$
then the complementary integral is bounded below as:
$$ \int_{\mathcal{O}}\int_{-B}^{B} (1-r(\xi,v))w(v)\,\dd v\,\Lambda(\dd\xi) \ge 2\,\Lambda(\mathcal{O})\,\int_a^B w(u)\,\dd u. $$
\end{lemma}
\begin{proof}
Define the central cylinder set $C:=\mathcal{O}\times\{v:|v|<a\}$. The measure of $C$ is $2a\,\Lambda(\mathcal{O})$.
Because $w(v)$ is non-increasing in $|v|$, we have $w(v)\ge w(a)$ on $C$ and $w(v)\le w(a)$ on $C^c$. Thus, the integrand $\bigl(\mathbbm{1}_C(\xi,v)-r(\xi,v)\bigr)\bigl(w(v)-w(a)\bigr)$ is  non-negative everywhere. Consequently:
$$ 0 \le \int_{ \mathcal{O}\times (-B,B) } \bigl(\mathbbm{1}_C-r\bigr)\bigl(w-w(a)\bigr). $$
Expanding this inequality gives:
$$ \int_{ \mathcal{O}\times (-B,B) } r w \le \int_C w - w(a)\left( \int_C 1 - \int_{ \mathcal{O}\times (-B,B) } r \right) \le \int_C w, $$
where the final step follows because the assumption $\int_{ \mathcal{O}\times (-B,B) } r \le 2a\,\Lambda(\mathcal{O}) = \int_C 1$ forces the bracketed term to be non-negative. Therefore:
$$ \int_{ \mathcal{O}\times (-B,B) } (1-r)w \ge \int_{ \mathcal{O}\times (-B,B) } w - \int_C w = \int_{C^c} w = 2\ \Lambda(\mathcal{O})\ \int_a^B w(u)\,\dd u.  $$
\end{proof}

\begin{lemma}[Flat Bayesian Sign Lower Bound]
\label{lem:flat-sign-lower-detailed}
Fix $\eta\in(0,1)$ and suppose $\alpha < \frac{(1-\eta)c_\rho\delta_{\rho}}{4}$. Define the threshold parameter:
\begin{equation}
a_{T,\eta} := \frac{\alpha\sqrt{T}}{2(1-\eta)c_\rho}. \label{eqn:aTeta}
\end{equation}
Then, for every terminal rule $\psi$ satisfying $\mathcal{A}_T^{Q_\rho}(\psi) \le \alpha$, we have that as $T\to\infty$:
$$ \mathcal{E}_T^{Q_\rho}(\psi) \ge (1-o_T(1)) \frac{2(1-\eta)c_\rho}{\sqrt{T}} \int_{a_{T,\eta}}^\infty \Phi(-u)\,\dd u. $$
\end{lemma}

\begin{proof}
By Lemma~\ref{lem:flat-central-approx}, for all sufficiently large $T$, $A_T(v)+B_T(v)\ge 1-\eta$ on the interval $|v|\le L_T/2$. Given the abstention constraint $\mathcal{A}_T^{Q_\rho}(\psi) \le \alpha$, we can further restrict the integration domain to this interval $[-L_T/2,L_T/2]$ to obtain:
$$ \frac{2}{\sqrt{T}} \int_{\mathcal{C}_\rho} \int_{|v|\le L_T/2} \psi_?^\xi(v) \bigl[A_T(v)+B_T(v)\bigr] \,\dd v\,\Lambda_\rho(\dd\xi) \le \alpha. $$
Substituting the uniform central bound (first bound in Lemma \ref{lem:flat-central-approx}) yields:
$$ \int_{\mathcal{C}_\rho} \int_{|v|\le L_T/2} \psi_?^\xi(v) \,\dd v\,\Lambda_\rho(\dd\xi) \le \frac{\alpha\sqrt{T}}{2(1-\eta)}. $$
Because $\Lambda_\rho(\mathcal{C}_\rho) = c_\rho/2$, the right-hand side is   $2a_{T,\eta}\,\Lambda_\rho(\mathcal{C}_\rho)$. 

Next, we evaluate the undetected error mass. Discarding the non-negative contribution from $|v|>L_T/2$ and applying Lemma~\ref{lem:flat-pointwise-sign} and the second bound of Lemma~\ref{lem:flat-central-approx}, we obtain:
$$ \mathcal{E}_T^{Q_\rho}(\psi) \ge \frac{2(1-\eta)}{\sqrt{T}} \int_{\mathcal{C}_\rho} \int_{|v|\le L_T/2} \bigl(1-\psi_?^\xi(v)\bigr)\Phi(-|v|) \,\dd v\,\Lambda_\rho(\dd\xi). $$
We now apply the abstention rearrangement result from Lemma~\ref{lem:flat-rearrangement} with $B=L_T/2$, $w(v)=\Phi(-|v|)$, and $r(\xi,v)=\psi_?^{\xi}(v)$. Our assumption on $\alpha$  ensures that $a_{T,\eta} < L_T/4 < L_T/2$. This allows the rearrangement to yield:
$$ \int_{\mathcal{C}_\rho} \int_{|v|\le L_T/2} \bigl(1-\psi_?^\xi(v)\bigr)\Phi(-|v|) \,\dd v\,\Lambda_\rho(\dd\xi) \ge \Lambda_\rho(\mathcal{C}_\rho) \,2\int_{a_{T,\eta}}^{L_T/2}\Phi(-u)\,\dd u. $$
Substituting $\Lambda_\rho(\mathcal{C}_\rho)=c_\rho/2$:
$$ \mathcal{E}_T^{Q_\rho}(\psi) \ge \frac{2(1-\eta)c_\rho}{\sqrt{T}} \int_{a_{T,\eta}}^{L_T/2}\Phi(-u)\,\dd u. $$
Finally, because $a_{T,\eta} < L_T/4$, the omitted tail $\int_{L_T/2}^\infty \Phi(-u)\,\dd u$ decays exponentially faster than $\int_{a_{T,\eta}}^\infty \Phi(-u)\,\dd u$. Thus
$$ \int_{a_{T,\eta}}^{L_T/2}\Phi(-u)\,\dd u = (1-o_T(1)) \int_{a_{T,\eta}}^\infty \Phi(-u)\,\dd u, $$
which concludes the proof.  
\end{proof}

\subsection{Proof of the  Lower Bound Theorem}
\begin{proof}[Proof of Theorem \ref{thm:bayesian-abstention-lower}]
\label{subsec:lower-bound-conclusion}

Let $\pi$ be any policy satisfying the global abstention budget $\mathcal{A}_T(\pi) \le \alpha$. By Lemma~\ref{lem:risk-domination}, the abstention and undetected error mass of this strategy under the local subprior $Q_\rho$ are  bounded  above by their true global counterparts, i.e., 
$$ \mathcal{A}_T^{Q_\rho}(\pi) \le \mathcal{A}_T(\pi) \le \alpha, \quad \mbox{and} \quad
\mathcal{E}_T^{Q_\rho}(\pi)\le \mathcal{E}_T(\pi).$$

By Lemma~\ref{lem:sufficiency-local-pair-strategies} and Lemma~\ref{lem:local-reduce-to-V}, the original strategy $\pi$ induces a $V$-based terminal rule   $\psi^\xi$ in the experiment with background oracle revealed. This reduction preserves the abstention probability and ensures the induced sign-rule error is no larger than the original undetected error. Therefore
$$ \mathcal{A}_T^{Q_\rho}(\psi^\xi) \le \alpha. $$
Because the abstention constraint is satisfied, we may apply Lemma~\ref{lem:flat-sign-lower-detailed}. This yields:
$$ \mathcal{E}_T(\pi) \ge \mathcal{E}_T^{Q_\rho}(\psi^\xi) \ge (1-o_T(1)) \frac{2(1-\eta)c_\rho}{\sqrt{T}} \int_{a_{T,\eta}}^\infty \Phi(-u)\,\dd u, $$
where the lower limit of the integral $a_{T,\eta}$ is defined in \eqref{eqn:aTeta}.

Fix $\rho \in (0,1)$ and $\eta \in (0,1)$. The construction of the local subprior gives constants $c_\rho > 0$ and $\delta_{\rho} > 0$. We define the critical abstention threshold
$ \alpha_0(\rho,\eta) := \frac{(1-\eta)c_\rho\delta_{\rho}}{4}. $
For every fixed $\alpha \in (0, \alpha_0(\rho,\eta))$, Lemma~\ref{lem:flat-sign-lower-detailed} applies. Furthermore, as $T \to \infty$, $  a_{T,\eta} = \frac{\alpha\sqrt{T}}{2(1-\eta)c_\rho} \to \infty. $

Using $\int_a^\infty \Phi(-u)\,\dd u = \phi(a)-a\Phi(-a)$ and the standard Gaussian tail asymptotics:
$$ \log\left(\int_a^\infty \Phi(-u)\,\dd u\right) = -\frac{a^2}{2} + O(\log a) \qquad \text{as } a \to \infty, $$
we extract the exponential decay rate for any feasible strategy $\pi$:
$$ \liminf_{T\to\infty} \frac{1}{\alpha^2 T} \log \mathcal{E}_T(\pi) \ge -\frac{1}{8(1-\eta)^2 c_\rho^2}. $$
Taking the infimum over all feasible strategies yields the lower bound on the optimal undetected error probability $\mathfrak{E}_{T}(\alpha)$. Since this bound holds for all $\alpha \in (0, \alpha_0(\rho,\eta))$, taking the limit as $\alpha \downarrow 0$:
\begin{equation}
    \liminf_{\alpha\downarrow 0} \liminf_{T\to\infty} \frac{1}{\alpha^2 T} \log \mathfrak{E}_{T}(\alpha) \ge -\frac{1}{8(1-\eta)^2 c_\rho^2} 
\ge -\frac{1}{8(1-\eta)^2 (1-\rho)^4 \kappa_\nu^2}, \label{eqn:lowerbd}
\end{equation}
where the second inequality follows because our definition of the restricted set $\mathcal C_{ij}$ ensures that the constant satisfies $c_\rho \ge (1-\rho)^2 \kappa_\nu$.
Because \eqref{eqn:lowerbd} holds for arbitrary $\eta \in (0,1)$ and $\rho \in (0,1)$, we may take $\eta \downarrow 0$ and then $\rho \downarrow 0$ to obtain the desired lower bound.   
\end{proof}

\section{Proof of Theorem~\ref{thm:no-abstention-lower}}
\label{app:no-abstention-lower}

We prove the no-abstention (forced-decision) lower bound by reusing the local subprior and
oracle reduction developed in the proof of
Theorem~\ref{thm:bayesian-abstention-lower}. The only difference is that the
terminal rule is not allowed to abstain. Thus, in the resulting one-dimensional
flat sign experiment, the rule must always choose one of the two signs.

Fix $\rho\in(0,1)$. Let $Q_\rho$, $\mathcal{C}_\rho$, $\Lambda_\rho$,
$\delta_\rho$, and $c_\rho$ be the flat local subprior objects
constructed in the proof of Theorem~\ref{thm:bayesian-abstention-lower}. We
recall the two properties of the construction of $Q_\rho$. First, by Lemma~\ref{lem:subprior-domination},
$Q_\rho\le P_\nu $.
Second, by Lemma \ref{lem:q_rho_properties}, under $Q_\rho$, the local coordinates
$(\Xi,\Delta,H)$ satisfy \eqref{eqn:Q_rho_dist}
and
    $c_\rho:=
    2\Lambda_\rho(\mathcal{C}_\rho)
    \ge
    (1-\rho)^2\kappa_\nu $.
On the support of $Q_\rho$, the two arms in $\Xi=(i,j,x,\zeta_{-ij})$ are (strictly) the top two arms.

Let $\pi$ be any forced-decision policy, so 
$\hat a_T\in[K]$ almost surely. Since abstention never occurs,
misidentification and undetected error coincide:
$\{\hat a_T\neq a^\star\}
=
\{\hat a_T\notin\{a^\star,?\}\}$.
So by $Q_\rho\le P_\nu $,
\begin{equation}
\label{eq:no-abstention-risk-domination}
\mathbb{P}_\nu^\pi(\hat{a}_T \neq a^\star) 
= \int \mathbb{P}_\mu^\pi(\hat{a}_T \neq a^\star) P_\nu(\dd\mu) 
\ge \int \mathbb{P}_\mu^\pi(\hat{a}_T \notin \{a^\star, ?\}) Q_\rho(\dd\mu) 
:= \mathcal{E}_T^{Q_\rho}(\pi)
\ge \mathcal{E}_T^{Q_\rho}(\psi^\xi).
\end{equation}

The lower bound of the error probability in the oracle-revealed experiment provides a lower bound for the original problem. Moreover, by Lemma~\ref{lem:sufficiency-local-pair-strategies}, after
revealing $\Xi=\xi$ it is enough to consider a strategy that samples only the
two local top arms and outputs a sign in $\{+,-\}$. By
Lemma~\ref{lem:local-reduce-to-V}, the terminal rule can further be taken to
depend on the history only through the sufficient statistic
$V=\frac1{\sqrt T}\sum_{t=1}^T Z_t$,
where $Z_t$ denotes the transformed observation defined in the local reduction.

Let $G:=\frac{\Delta\sqrt T}{2}$ and $L_T:=\frac{\delta_\rho\sqrt T}{2}$.
Then $G$ has support $[0,L_T]$ under $Q_\rho$, and the flat sign experiment satisfies
\[
    V\mid(G=g,H=h,\Xi=\xi)\sim \mathcal N(hg,1),
    \qquad h\in\{+1,-1\}.
\]
Any forced-decision terminal rule $\psi^\xi$ satisfies
\[
    \psi_+^\xi(v)+\psi_-^\xi(v)=1,
    \quad\mbox{and}\quad
    \psi_+^\xi(v),\psi_-^\xi(v)\ge0 .
\]
Define
\[
    A_T(v)
    :=
    \int_0^{L_T}\phi(v-g)\,\dd g,
    \quad\mbox{and}\quad
    B_T(v)
    :=
    \int_0^{L_T}\phi(v+g)\,\dd g .
\]
Using \eqref{eqn:Q_rho_dist} and the change of variables
$\dd\delta=2\,\dd g/\sqrt T$, 
$$
    \mathcal E_T^{Q_\rho}(\psi)
    =
    \frac{2}{\sqrt T}
    \int_{\mathcal{C}_\rho}\int_{\mathbb R}
    \left[
        \psi_-^\xi(v)A_T(v)
        +
        \psi_+^\xi(v)B_T(v)
    \right]\,\dd v\,\Lambda_\rho(\dd\xi).\notag
$$
Since $\psi_+^\xi(v)+\psi_-^\xi(v)=1$, we have pointwise
\[
    \psi_-^\xi(v)A_T(v)
    +
    \psi_+^\xi(v)B_T(v)
    \ge
    \min\{A_T(v),B_T(v)\}.
\]
Therefore,
\begin{equation}
\label{eq:no-abstention-min-lower}
    \mathcal E_T^{Q_\rho}(\psi)
    \ge
    \frac{2}{\sqrt T}
    \Lambda_\rho(\mathcal{C}_\rho)
    \int_{\mathbb R}\min\{A_T(v),B_T(v)\}\,\dd v
    =
    \frac{c_\rho}{\sqrt T}
    \int_{\mathbb R}\min\{A_T(v),B_T(v)\}\,\dd v .
\end{equation}

It remains to evaluate the one-dimensional integral. Recall  that~\eqref{eq:min_AB_with_Phi} gives
    $\min\{A_T(v),B_T(v)\}
    =
    \Phi(-|v|)
    -
    \Phi(-(|v|+L_T))$.
Consequently,
$$    \int_{\mathbb R}\min\{A_T(v),B_T(v)\}\,\dd v
    =
    2\int_0^\infty
    \left[
        \Phi(-v)-\Phi(-(v+L_T))
    \right]\,\dd v  \notag
    =
    2\int_0^{L_T} \Phi(-u)\,\dd u .
$$
Since $L_T\to\infty$ as $T\to\infty$,
\[
    \int_{\mathbb R}\min\{A_T(v),B_T(v)\}\,\dd v
    \to
    2\int_0^\infty \Phi(-u)\,\dd u .
\]
Finally,
\[
    \int_0^\infty \Phi(-u)\,\dd u
    =
    \int_0^\infty\int_u^\infty \phi(z)\,\dd z\,\dd u
    =
    \int_0^\infty z\phi(z)\,\dd z
    =
    \frac{1}{\sqrt{2\pi}}.
\]
Thus
\begin{equation}
\label{eq:no-abstention-integral-limit}
    \int_{\mathbb R}\min\{A_T(v),B_T(v)\}\,\dd v
    \to
    \sqrt{\frac{2}{\pi}} .
\end{equation}

Combining \eqref{eq:no-abstention-risk-domination},
\eqref{eq:no-abstention-min-lower}, and
\eqref{eq:no-abstention-integral-limit}, we obtain a policy-independent bound: for any $\pi\in\Pi_T(0)$, 
\[
    \liminf_{T\to\infty}
    \sqrt T\,
    \mathbb P_\nu^\pi(\hat a_T\neq a^\star)
    \ge
    c_\rho\sqrt{\frac{2}{\pi}} .
\]
Minimizing over all policies and using the definition of $c_\rho$ yields
\[
    \liminf_{T\to\infty}
    \sqrt T\,
    \mathfrak E_{T}(0)
    \ge
    c_\rho\sqrt{\frac{2}{\pi}} 
    \ge
    (1-\rho)^2\kappa_\nu\sqrt{\frac{2}{\pi}}.
\]
Because $\rho\in(0,1)$ is arbitrary, taking $\rho\downarrow0$ yields the desired lower bound.  

\section{Proof of Theorem~\ref{thm:two-arm-frequentist-abstention}}
\label{app:frequentist-abstention}

\begin{proof} \label{sec:proof_thm4}
Fix any strategy $\pi\in\Pi_T^{\rm freq}(\alpha)$. The proof uses
the midpoint instance $\mu^0:=(x,x)$
as a change-of-measure reference. Since the abstention constraint is uniform over
all two-arm Gaussian instances,
\begin{equation}
\label{eq:freq-abst-midpoint}
    \mathbb P_{\mu^0}^\pi(\hat a_T=?)\le \alpha .
\end{equation}
For $\theta\in\{-1,0,+1\}$, define
\[
    \mu^\theta
    :=
    \left(
        x+\theta\frac{\Delta}{2},
        x-\theta\frac{\Delta}{2}
    \right).
\]
Thus $\mu^0=(x,x)$. Let
\[
    S_t:=\mathbbm{1}\{A_t=1\}-\mathbbm{1}\{A_t=2\}\in\{+1,-1\},
\]
and define the centered signed observation $Z_t:=S_t(X_t-x)$.
Finally, set $W:=\sum_{t=1}^T Z_t$, $V:=\frac{W}{\sqrt T}$ and $m:=\frac{\Delta\sqrt T}{2}$.

We first reduce the adaptive experiment to a one-dimensional Gaussian shift
experiment. Under $\mu^\theta$, conditionally on the past and on the chosen arm
$A_t$, the transformed observation satisfies $Z_t\sim \mathcal N\!\left(\theta\frac{\Delta}{2},1\right)$.
The conditional law of $Z_t$ is therefore independent of both the past and the
chosen arm. Hence
    $Z_1,\ldots,Z_T
    \overset{\mathrm{i.i.d.}}{\sim}
    \mathcal N\!\left(\theta\frac{\Delta}{2},1\right)$,
and consequently
\[
    V\sim
    \begin{cases}
        \mathcal N(m,1), & \theta=+1,\\
        \mathcal N(0,1), & \theta=0,\\
        \mathcal N(-m,1), & \theta=-1.
    \end{cases}
\]

The same argument as in Lemma~\ref{lem:local-reduce-to-V} applies here, with the background oracle consisting only of the two arms and their midpoint $x$. Therefore, for the three instances
$\mu^0,\mu^+,\mu^-$, the terminal decision rule may be represented by measurable
functions $\psi = (\psi_+, \psi_-,\psi_?)$,
where $\psi_+(v)$ is the probability of recommending arm $1$ after observing
$V=v$, $\psi_-(v)$ is the probability of recommending arm $2$, and
$\psi_?(v)$ is the probability of abstaining. These functions satisfy, for every $v$,
\[
    \psi_+(v)+\psi_-(v)+\psi_?(v)=1,
    \qquad
    \psi_+(v),\psi_-(v),\psi_?(v)\in[0,1].
\]

Let $\phi$ denote the standard normal density. Under the midpoint instance
$\mu^0$, we have $V\sim\mathcal N(0,1)$, so
\eqref{eq:freq-abst-midpoint} means
\begin{equation}
\label{eq:freq-psi-abstention}
    \int_{\mathbb R}\psi_?(v)\phi(v)\,\dd v\le \alpha .
\end{equation}
Under $\mu^+$, arm $1$ is uniquely optimal and $V\sim\mathcal N(m,1)$.
Therefore
\begin{equation}
\label{eq:freq-error-plus-psi}
    \mathcal E_{T,\mu^+}(\pi)
    =
    \mathbb P_{\mu^+}^\pi(\hat a_T=2)
    =
    \int_{\mathbb R}\psi_-(v)\phi(v-m)\,\dd v .
\end{equation}
Similarly, under $\mu^-$, arm $2$ is uniquely optimal and
$V\sim\mathcal N(-m,1)$, so
\begin{equation}
\label{eq:freq-error-minus-psi}
    \mathcal E_{T,\mu^-}(\pi)
    =
    \mathbb P_{\mu^-}^\pi(\hat a_T=1)
    =
    \int_{\mathbb R}\psi_+(v)\phi(v+m)\,\dd v .
\end{equation}
We next lower bound the average of the two mirror-instance errors. Define
\[
    p_+(v):=\phi(v-m),
    \qquad
    p_-(v):=\phi(v+m).
\]
By \eqref{eq:freq-error-plus-psi} and \eqref{eq:freq-error-minus-psi},
$$
\begin{aligned}
    \overline{\mathcal E}
    :=
    \frac12\mathcal E_{T,\mu^+}(\pi)
    +
    \frac12\mathcal E_{T,\mu^-}(\pi)
    &=
    \frac12
    \int_{\mathbb R}
    \left[
        \psi_-(v)p_+(v)
        +
        \psi_+(v)p_-(v)
    \right]\,\dd v \\
    &\ge
    \frac12
    \int_{\mathbb R}
    \bigl(1-\psi_?(v)\bigr)
    \min\{p_+(v),p_-(v)\}\,\dd v .
\end{aligned}
$$
Note that
$
    \min\{p_+(v),p_-(v)\}
    =
    \phi(|v|+m)
$. Hence
\begin{equation}
\label{eq:freq-average-error-phi}
    \overline{\mathcal E}
    \ge
    \frac12
    \int_{\mathbb R}
    \bigl(1-\psi_?(v)\bigr)\phi(|v|+m)\,\dd v .
\end{equation}

It remains to identify the least possible value of the right-hand side subject
to the midpoint abstention constraint \eqref{eq:freq-psi-abstention}. Let $r(v):=\psi_?(v)$.
Then $0\le r(v)\le1$ and
\[
    \int_{\mathbb R}r(v)\phi(v)\,\dd v\le\alpha .
\]
Write
\[
    \phi(|v|+m)
    =
    q_m(v)\phi(v),
    \qquad
    q_m(v):=
    \exp\!\left(-m|v|-\frac{m^2}{2}\right).
\]
The function $q_m$ is even and nonincreasing in $|v|$. Let $C_\alpha:=[-z_\alpha,z_\alpha]$.
By the definition of $z_\alpha$,
\[
    \int_{C_\alpha}\phi(v)\,\dd v=\alpha .
\]
Since $q_m(v)\ge q_m(z_\alpha)$ on $C_\alpha$ and
$q_m(v)\le q_m(z_\alpha)$ on $C_\alpha^c$, we have pointwise
\[
    \bigl(r(v)-\mathbbm{1}_{C_\alpha}(v)\bigr)
    \bigl(q_m(v)-q_m(z_\alpha)\bigr)
    \le 0 .
\]
Integrating with respect to $\phi(v)\,\dd v$ gives
\[
    \int_{\mathbb R} r(v)q_m(v)\phi(v)\,\dd v
    \le
    \int_{C_\alpha}q_m(v)\phi(v)\,\dd v
    +
    q_m(z_\alpha)
    \left[
        \int_{\mathbb R}r(v)\phi(v)\,\dd v
        -
        \int_{C_\alpha}\phi(v)\,\dd v
    \right].
\]
The bracketed term is nonpositive by \eqref{eq:freq-psi-abstention}, and hence
\[
    \int_{\mathbb R} r(v)q_m(v)\phi(v)\,\dd v
    \le
    \int_{C_\alpha}q_m(v)\phi(v)\,\dd v .
\]
Equivalently,
\begin{equation}
\label{eq:freq-rearrangement}
    \int_{\mathbb R}
    \bigl(1-r(v)\bigr)\phi(|v|+m)\,\dd v
    \ge
    \int_{\mathbb R\setminus C_\alpha}\phi(|v|+m)\,\dd v .
\end{equation}
By symmetry,
\[
    \int_{\mathbb R\setminus C_\alpha}\phi(|v|+m)\,\dd v
    =
    2\int_{z_\alpha}^{\infty}\phi(v+m)\,\dd v
    =
    2\Phi(-m-z_\alpha).
\]
Combining this identity with \eqref{eq:freq-average-error-phi} and
\eqref{eq:freq-rearrangement} yields $\overline{\mathcal E}
    \ge
    \Phi(-m-z_\alpha)$.
Finally,
\[
    \max\left\{
        \mathcal E_{T,\mu^+}(\pi),
        \mathcal E_{T,\mu^-}(\pi)
    \right\}
    \ge
    \overline{\mathcal E}
    \ge
    \Phi(-m-z_\alpha)
    =\Phi\!\left(
        -\frac{\Delta\sqrt T}{2}-z_\alpha
    \right). 
\]

\end{proof}

\section{Proof of Theorem~\ref{thm:general-transfer}} \label{app:pf_thm_general}

\subsection{Assumptions}

We start by recalling the assumptions for Theorem~\ref{thm:general-transfer} and providing some more explanations.
Let \(\mathcal I\subseteq \mathbb R\) be an open interval, called the
{\em quality parameter space}. For each arm \(i\in[K]\), let
\(\mu_i\in\mathcal I\) denote its {\em quality parameter}. Define the best
arm by $a^\star := \arg\max_{i\in[K]}\mu_i$. 
When arm \(i\) is sampled, its reward distribution $P_{\mu_i}$ belongs to a one-parameter
family \(\{P_\mu:\mu\in\mathcal I\}\).

\begingroup
\renewcommand{\thesection}{4}
\AssRegularExpfam*
\endgroup
Under Assumption~\ref{ass:regular-expfam}, the Fisher information for the
quality parameter \(\mu\) is
\[
    I(\mu)
    =
    A''(\theta(\mu))\{\theta'(\mu)\}^2 .
\]
Thus \(I(\mu)>0\), and \(I\) is bounded away from zero and infinity on every
compact subset of \(\mathcal I\).  Fix \(\mu_\circ\in\mathcal I\), and
define the \emph{Fisher--Rao information coordinate}
\[
    s(\mu)
    :=
    \int_{\mu_\circ}^\mu\sqrt{I(v)}\,\dd v,
\]
where \(\mu_\circ\in\mathcal I\) is arbitrary. This is the arc-length coordinate induced by the
Fisher--Rao metric \citep{Rao1992,amari2000methods}.
We remark that all our results are independent of the choice of \(\mu_\circ\).
Indeed, if another base point \(\widetilde\mu_\circ\in\mathcal I\) is used, and 
$\widetilde s(\mu) := \int_{\widetilde\mu_\circ}^{\mu}\sqrt{I(v)}\,\dd v$,
then $\widetilde s(\mu) = s(\mu)-s(\widetilde\mu_\circ)$. Thus all transformed arm qualities are shifted by the same constant. Consequently, the ordering of arms, all pairwise information-coordinate gaps, and the event of
a near top-two tie are unchanged. The interval \(\mathcal H\), the densities
\(\bar f_i\), and the distribution functions \(\bar F_i\) are merely translated,
so the integral defining \(\kappa\) is invariant under the corresponding change
of variables. The quadratic KL condition below is also unchanged because it
depends only on differences \(s(\mu)-s(z)\).

Also let $\eta_i:=s(\mu_i)$ and $\beta(\eta):=\theta(s^{-1}(\eta))$. Note $\theta(\mu)=\beta(s(\mu))$. Since \(s\) is strictly increasing, the best arm is unchanged by the
transformation \(\mu_i\mapsto \eta_i\).

\begingroup
\renewcommand{\thesection}{4}
\AssRegularInfoPrior*
\endgroup

Define the information-scaled top-two tie density
\begin{equation}
\label{eq:kappa-transfer-upper}
    \kappa
    :=
    \sum_{i\neq j}
    \int_{\mathcal H}
        \bar f_i(u)\bar f_j(u)
        \prod_{k\neq i,j}\bar F_k(u)
    \,\dd u .
\end{equation}
Equivalently, if \(f_i\) and \(F_i\) respectively denote the prior density and distribution function
of \(\mu_i\) in the original quality coordinate system, then $$
\kappa = \sum_{i \neq j} \int_{\mathcal I} \frac{f_i(\mu)f_j(\mu)\prod_{k\neq i,j}{F_k(\mu)}}{\sqrt{I(\mu)}}  \dd\mu
$$
Thus \(\kappa\) is the density at zero of the top-two gap after transforming to
the Fisher--Rao information coordinate. Note that under Assumption~\ref{ass:regular-info-prior},
\(0<\kappa<\infty\). 

\begin{remark} We note that, each arm has 3 (equivalent) descriptions:
$$
\mu \leftrightarrow \eta :=s(\mu) \leftrightarrow \theta := \theta(\mu) := \beta(\eta),
$$
where \begin{enumerate}
    \item $\mu$ is the actual parameter of an arm, the value drawn from the prior. It is the \emph{quality (parameter) coordinate} in which the decision problem is defined: $a^*=\argmax_i \mu_i$, and $\mu$ measures the \emph{quality} of the arms.

    \item $\theta(\mu)$ is the \emph{natural (parameter) coordinate}. It is the coordinate in which the reward distribution has exponential-family form.

    \item $\eta$ is the \emph{Fisher--Rao information (parameter) coordinate}. It rescales the quality coordinate so that local Fisher information is equal to one.
\end{enumerate}

Note that in the Gaussian case, with $\mu_\circ=0$, we have $\mu=s(\mu)=\theta(\mu)$.
\end{remark}

For the achievability result we also assume a global quadratic lower bound on
the Kullback--Leibler divergence. For \(u,v\in\mathcal I\), let
\[
    D(u,v)
    :=
    D(P_u\,\|\,P_v)
    =
    \{\theta(u)-\theta(v)\}A'(\theta(u))
    -
    A(\theta(u))
    +
    A(\theta(v)).
\]

\begingroup
\renewcommand{\thesection}{4}
\AssQuadraticKL*
\endgroup
The assumption holds for several common one-parameter families when $s(\cdot)$ is chosen as the information coordinate. In particular, Gaussian location, Bernoulli,  and Poisson  satisfy \eqref{eqn:quad_growth} with $c_{\rm KL}=1/4$.  Assumption \ref{ass:quadratic-kl} can be viewed as a global version of the local quadratic approximation $D(\mu,z)\approx\frac{1}{2}d_{\mathrm{FR}}(\mu,z)^2$ where $d_{\mathrm{FR}}$ is the Fisher--Rao distance; cf.\  \citet{amari2000methods}.

\subsection{Lower Bound for Theorem~\ref{thm:general-transfer}}
\label{app:transfer-lower-proof}
We first prove the lower bound
\begin{theorem}[Lower Bound for the General Case]
\label{thm:transfer-lower}
Under Assumptions~\ref{ass:regular-expfam} and \ref{ass:regular-info-prior}, 
\[
    \liminf_{\alpha\downarrow0}
    \liminf_{T\to\infty}
    \frac{1}{\alpha^2T}
    \log \mathfrak E_T(\alpha)
    \ge
    -\frac{1}{8\kappa^2}.
\]
\end{theorem}
Throughout this proof, \(o_T(\alpha^2T)\) denotes a quantity $r_T(\alpha)\alpha^2T$ such that \(r_T(\alpha)\to0\) as \(T\to\infty\) for each fixed \(\alpha>0\).
All $o_T(\alpha^2T)$ terms below are uniform over $\xi\in\mathcal C_\rho$ and over all local pair strategies, for fixed $\rho$, $\varepsilon$, and $D$.

\subsubsection{Step 1. A Dominated Flat Local Subprior}
\label{subsubsec:flat-local-subprior-transfer-lower}

Fix $\rho\in(0,1)$.

Recall we let $\beta(\eta):=\theta\bigl(s^{-1}(\eta)\bigr)$ for $\eta\in\mathcal H$.
Since both $\theta(\cdot)$ and $s(\cdot)$ are strictly increasing,
$\beta:\mathcal H\to\Theta$ is strictly increasing. Moreover, as \(\eta=s(\mu)\),
\[
    \beta'(\eta)
    =
    \frac{\theta'(\mu)}{s'(\mu)}
    =
    \frac{\theta'(\mu)}
         {\sqrt{A''(\theta(\mu))}\theta'(\mu)}
    =
    \frac{1}{\sqrt{A''(\theta(\mu))}}.
\]

\begin{definition}[Local Change of Variable]
\label{def:local-coordinate-transfer-lower}
Fix a pair $i<j$, and write
\[
    \xi=(i,j,u,\zeta_{-ij}),
    \qquad
    u\in\mathcal H,
    \qquad
    \zeta_{-ij}\in\mathcal H^{K-2}.
\]
For $\delta\ge0$ and $h\in\{+1,-1\}$, define
\[
\begin{aligned}
    \theta_i(u,\delta,h)
    &:=
    \beta(u)
    +
    h\frac{\delta}{2\sqrt{A''(\beta(u))}},\\
    \theta_j(u,\delta,h)
    &:=
    \beta(u)
    -
    h\frac{\delta}{2\sqrt{A''(\beta(u))}}.
\end{aligned}
\]
Whenever these two natural parameters belong to $\Theta$, define
$\eta(\xi,\delta,h)\in\mathcal H^K$ by
\begin{equation}
\label{eq:local-chart-transfer-lower}
\begin{aligned}
    \eta_i(\xi,\delta,h)
    &:=
    \beta^{-1}\bigl(\theta_i(u,\delta,h)\bigr),\\
    \eta_j(\xi,\delta,h)
    &:=
    \beta^{-1}\bigl(\theta_j(u,\delta,h)\bigr),\\
    \eta_k(\xi,\delta,h)
    &:=
    \zeta_k,
    \qquad k\neq i,j.
\end{aligned}
\end{equation}
The corresponding vector in the original parameter space is defined
coordinatewise as
\[
    \mu_k(\xi,\delta,h)
    :=
    s^{-1}\bigl(\eta_k(\xi,\delta,h)\bigr),
    \qquad k\in[K].
\]
\end{definition}

Note that for each fixed pair $i<j$ and sign $h$, the map
    $(u,\delta,\zeta_{-ij})
    \longmapsto
    \eta(\xi,\delta,h)$
is one-to-one. Indeed, writing
$\theta_i=\beta(\eta_i)$ and $\theta_j=\beta(\eta_j)$, one recovers
\[
    u
    =
    \beta^{-1}\left(\frac{\theta_i+\theta_j}{2}\right),
    \qquad
    \delta
    =
    h\sqrt{A''(\beta(u))}\,(\theta_i-\theta_j),
    \qquad
    \zeta_k=\eta_k
    \quad(k\neq i,j).
\]
In particular, when $\delta>0$, the sign is uniquely determined by
$h=\operatorname{sign}(\eta_i-\eta_j)$.

For fixed $i<j$ and $h$, let
\[
    J_{ijh}(u,\delta)
    :=
    \left|
        \det
        \frac{\partial(\eta_i,\eta_j)}
             {\partial(u,\delta)}
    \right|.
\]
The remaining coordinates satisfy $\eta_k=\zeta_k$, so
$J_{ijh}(u,\delta)$ is also the absolute Jacobian determinant of the full
transformation
$(u,\delta,\zeta_{-ij})\mapsto\eta(\xi,\delta,h)$.

At $\delta=0$, we have
$\eta_i(\xi,0,h)=\eta_j(\xi,0,h)=u$. Using
$\beta'(u)=1/\sqrt{A''(\beta(u))}$ gives
\[
    \left.
    \frac{\partial(\eta_i,\eta_j)}
         {\partial(u,\delta)}
    \right|_{\delta=0}
    =
    \begin{pmatrix}
        1 & h/2\\
        1 & -h/2
    \end{pmatrix}.
\]
Consequently, $J_{ijh}(u,0)=1$. Unlike the Gaussian location model,
however, $J_{ijh}(u,\delta)$ is not generally equal to one when
$\delta>0$.

For each pair $i<j$, define the top-pair tie domain
\[
    \mathcal T_{ij}
    :=
    \left\{
        (i,j,u,\zeta_{-ij})
        \in\mathcal H\times\mathcal H^{K-2}:
        \zeta_k<u
        \text{ for every }k\neq i,j
    \right\},
\]
and let $\mathcal T:=\biguplus_{i<j}\mathcal T_{ij}$ denote the disjoint
union over pairs. Define a finite measure $\lambda$ on $\mathcal T$
componentwise as
\[
    \lambda(\dd\xi)
    :=
    \bar f_i(u)\bar f_j(u)
    \prod_{k\neq i,j}\bar f_k(\zeta_k)
    \,\dd u\,\dd\zeta_{-ij},
    \qquad
    \xi\in\mathcal T_{ij}.
\]
Then
\[
\begin{aligned}
    2\lambda(\mathcal T)
    &=
    2\sum_{i<j}
    \int_{\mathcal H}
        \bar f_i(u)\bar f_j(u)
        \prod_{k\neq i,j}\bar F_k(u)
    \,\dd u\\
    &=
    \sum_{i\neq j}
    \int_{\mathcal H}
        \bar f_i(u)\bar f_j(u)
        \prod_{k\neq i,j}\bar F_k(u)
    \,\dd u
    =
    \kappa.
\end{aligned}
\]

Since $\lambda$ is a finite Borel measure, it is inner regular. Hence there
exists a compact set $\mathcal C_\rho\subset\mathcal T$ such that
\begin{equation}
\label{eq:Xi-rho-mass-transfer-lower}
    2\lambda(\mathcal C_\rho)
    \ge
    (1-\rho)\kappa.
\end{equation}
For each pair $i<j$, let
    $\mathcal C_{\rho,ij}
    :=
    \mathcal C_\rho\cap\mathcal T_{ij}$,
so that
$\mathcal C_\rho=\biguplus_{i<j}\mathcal C_{\rho,ij}$.

Because $\mathcal C_\rho$ is a compact subset of the open set
$\mathcal T$, there exists $\ell_\rho>0$ such that, for every
$\xi=(i,j,u,\zeta_{-ij})\in\mathcal C_{\rho,ij}$,
\begin{equation}
\label{eq:compact-context-margin}
    \zeta_k
    \le
    u-2\ell_\rho
    \quad(k\neq i,j),
    \qquad
    [u-2\ell_\rho,u+2\ell_\rho]
    \subset\mathcal H.
\end{equation}
This will ensure that sufficiently small local perturbations remain in
$\mathcal H$ and preserve $\{i,j\}$ as the unique top-two pair.

Define a finite measure $\Lambda_\rho$ on $\mathcal C_\rho$ as
\begin{equation}
\label{eq:Lambda-rho-transfer-lower}
    \Lambda_\rho(\dd\xi)
    :=
    (1-\rho)
    \bar f_i(u)\bar f_j(u)
    \prod_{k\neq i,j}\bar f_k(\zeta_k)
    \,\dd u\,\dd\zeta_{-ij},
    \qquad
    \xi\in\mathcal C_{\rho,ij}.
\end{equation}
Let
\begin{equation}
\label{eq:c-rho-transfer-lower}
    c_\rho
    :=
    2\Lambda_\rho(\mathcal C_\rho).
\end{equation}
By \eqref{eq:Xi-rho-mass-transfer-lower},
\begin{equation}
\label{eq:c-rho-lower-transfer-lower}
    c_\rho
    =
    2(1-\rho)\lambda(\mathcal C_\rho)
    \ge
    (1-\rho)^2\kappa.
\end{equation}

\begin{lemma}[Uniform Density Lower Bound]
\label{lem:uniform-density-transfer-lower}
There exists $\delta_\rho>0$ such that, for every $i<j$,
every $\xi=(i,j,u,\zeta_{-ij})\in\mathcal C_{\rho,ij}$,
every $h\in\{+1,-1\}$, and every $\delta\in[0,\delta_\rho]$, the coordinate
map in \eqref{eq:local-chart-transfer-lower} is well-defined,
$J_{ijh}(u,\delta)>0$, and
\begin{equation}
\label{eq:local-density-domination-transfer-lower}
    (1-\rho)\bar f_i(u)\bar f_j(u)
    \le
    \bar f_i\bigl(\eta_i(\xi,\delta,h)\bigr)
    \bar f_j\bigl(\eta_j(\xi,\delta,h)\bigr)
    J_{ijh}(u,\delta).
\end{equation}
Moreover, arms $i$ and $j$ are the unique top-two arms since
    $\min\{
        \eta_i(\xi,\delta,h),
        \eta_j(\xi,\delta,h)
    \}
    >
    \eta_k(\xi,\delta,h)$ for $k\neq i,j$.
\end{lemma}

\begin{proof}
For $\xi=(i,j,u,\zeta_{-ij})\in\mathcal C_{\rho,ij}$, define
\[
    R_{ijh}(\xi,\delta)
    :=
    \frac{
        \bar f_i\bigl(\eta_i(\xi,\delta,h)\bigr)
        \bar f_j\bigl(\eta_j(\xi,\delta,h)\bigr)
        J_{ijh}(u,\delta)
    }{
        \bar f_i(u)\bar f_j(u)
    }.
\]
Since $\mathcal C_\rho$ is
compact and the prior densities are continuous and strictly positive,
the denominator is uniformly bounded away from zero on
$\mathcal C_\rho$.

At $\delta=0$, we have
$\eta_i(\xi,0,h)=\eta_j(\xi,0,h)=u$ and
$J_{ijh}(u,0)=1$. Hence $R_{ijh}(\xi,0)=1$.

The coordinate maps, $J_{ijh}$, and $R_{ijh}$ are continuous in a
neighborhood of
$\mathcal C_\rho\times\{0\}\times\{+1,-1\}$. Since
$\mathcal C_\rho$ is compact and there are finitely many pairs and signs,
there exists $\delta_\rho>0$ such that, uniformly over all admissible
$\xi$, $h$, and $\delta\in[0,\delta_\rho]$,
\[
    \eta_i(\xi,\delta,h),
    \eta_j(\xi,\delta,h)
    \in
    (u-\ell_\rho,u+\ell_\rho),
    \qquad
    J_{ijh}(u,\delta)>0,
\]
and $R_{ijh}(\xi,\delta)\ge1-\rho$. The latter inequality is
\eqref{eq:local-density-domination-transfer-lower}.

Finally, by \eqref{eq:compact-context-margin},
\[
\begin{aligned}
    \min\{
        \eta_i(\xi,\delta,h),
        \eta_j(\xi,\delta,h)
    \}
    &>
    u-\ell_\rho\\
    &>
    u-2\ell_\rho\\
    &\ge
    \zeta_k
    =
    \eta_k(\xi,\delta,h),
    \qquad k\neq i,j.
\end{aligned}
\]
Thus $i$ and $j$ are the unique top-two arms. 
\end{proof}

\begin{definition}[Flat Local Subprior]
\label{def:flat-local-subprior-transfer-lower}
For every $0<\bar\delta\le\delta_\rho$, define the finite measure
$Q_{\rho,\bar\delta}$ on $(\mathcal I)^K$ as
\begin{equation}
\label{eq:Q-rho-definition-transfer-lower}
Q_{\rho,\bar\delta}(E)
:=
\sum_{i<j}\sum_{h\in\{+1,-1\}}
\int_{\mathcal C_{\rho,ij}}
\int_0^{\bar\delta}
\mathbbm 1_E\bigl(\mu(\xi,\delta,h)\bigr)
\,\dd\delta\,\Lambda_\rho(\dd\xi),
\qquad
E\in\mathcal B((\mathcal I)^K).
\end{equation}
\end{definition}

\begin{lemma}[Domination]
\label{lem:subprior-domination-transfer-lower}
For every $0<\bar\delta\le\delta_\rho$, the measure
$Q_{\rho,\bar\delta}$ is a submeasure of the original prior $P$; that is,
    $Q_{\rho,\bar\delta}(E)
    \le
    P(E)$
    for every 
    $E\in\mathcal B((\mathcal I)^K)$.
\end{lemma}

\begin{proof}

Let
    $s(\mu)
    := \bigl(s(\mu_1),\ldots,s(\mu_K)\bigr)$. Define the measures in the information coordinate as
\[
    \bar P(B)
    :=
    P\bigl(s^{-1}(B)\bigr),
    \qquad
    \bar Q_{\rho,\bar\delta}(B)
    :=
    Q_{\rho,\bar\delta}\bigl(s^{-1}(B)\bigr).
\]
Note $\bar P$ has Lebesgue density
    $\bar p(\eta)
    =
    \prod_{k=1}^K\bar f_k(\eta_k)$.

For each pair $i<j$, define
\[
    \mathcal S_{i,j}^{\bar\delta}
    :=
    \left\{
        \eta(\xi,\delta,h):
        \xi\in\mathcal C_{\rho,ij},\
        0\le\delta\le\bar\delta,\
        h\in\{+1,-1\}
    \right\},
\]
and set
$\mathcal S_\rho^{\bar\delta}:=
\bigcup_{i<j}\mathcal S_{i,j}^{\bar\delta}$.

By Lemma~\ref{lem:uniform-density-transfer-lower}, arms $i$ and $j$ are
the unique strict top-two arms on $\mathcal S_{i,j}^{\bar\delta}$.
Consequently, the sets $\mathcal S_{i,j}^{\bar\delta}$ are pairwise
disjoint. Moreover, when $\delta>0$, the sign is uniquely determined by
$h=\operatorname{sign}(\eta_i-\eta_j)$. The two sign representations
overlap only when $\delta=0$, which is a
$Q_{\rho,\bar\delta}$-null set.

For any Borel set $B\subseteq\mathcal H^K$, substituting the definition
of $\Lambda_\rho$ in \eqref{eq:Lambda-rho-transfer-lower} into
\eqref{eq:Q-rho-definition-transfer-lower} gives
\[
\begin{aligned}
    \bar Q_{\rho,\bar\delta}(B)
    =
    \sum_{i<j}\sum_{h\in\{+1,-1\}}
    \int_{\mathcal C_{\rho,ij}}
    \int_0^{\bar\delta}
    \mathbbm 1_B\bigl(\eta(\xi,\delta,h)\bigr)
    (1-\rho)\bar f_i(u)\bar f_j(u)
    \prod_{k\neq i,j}\bar f_k(\zeta_k)
    \,\dd\delta\,\dd u\,\dd\zeta_{-ij}.
\end{aligned}
\]

For each fixed pair and sign, apply change-of-variables with 
    $(u,\delta,\zeta_{-ij})
    \longmapsto
    \eta(\xi,\delta,h)$.
The absolute Jacobian determinant is $J_{ijh}(u,\delta)$. It follows that
$\bar Q_{\rho,\bar\delta}$ has Lebesgue density
$\bar q_{\rho,\bar\delta}$ satisfying
\[
    \bar q_{\rho,\bar\delta}
    \bigl(\eta(\xi,\delta,h)\bigr)
    =
    \frac{
        (1-\rho)\bar f_i(u)\bar f_j(u)
        \displaystyle\prod_{k\neq i,j}\bar f_k(\zeta_k)
    }{
        J_{ijh}(u,\delta)
    }
\]
for Lebesgue-almost every
$\eta(\xi,\delta,h)\in\mathcal S_{i,j}^{\bar\delta}$.

Furthermore, the definition
\eqref{eq:Q-rho-definition-transfer-lower} implies
    $\bar Q_{\rho,\bar\delta}
    \bigl(
        \mathcal H^K\setminus
        \mathcal S_\rho^{\bar\delta}
    \bigr)
    =
    0$.
Hence
$\bar q_{\rho,\bar\delta}(\eta)=0$ for Lebesgue-almost every
$\eta\notin\mathcal S_\rho^{\bar\delta}$.

For
$\eta=\eta(\xi,\delta,h)\in\mathcal S_{i,j}^{\bar\delta}$,
Lemma~\ref{lem:uniform-density-transfer-lower} gives
\[
    (1-\rho)\bar f_i(u)\bar f_j(u)
    \le
    \bar f_i(\eta_i)\bar f_j(\eta_j)
    J_{ijh}(u,\delta).
\]
Since $J_{ijh}(u,\delta)>0$ and $\eta_k=\zeta_k$ for $k\neq i,j$,
\[
\begin{aligned}
    \bar q_{\rho,\bar\delta}(\eta)
    \le
    \bar f_i(\eta_i)\bar f_j(\eta_j)
    \prod_{k\neq i,j}\bar f_k(\eta_k)
    =
    \bar p(\eta).
\end{aligned}
\]
The same inequality holds outside $\mathcal S_\rho^{\bar\delta}$ because
$\bar q_{\rho,\bar\delta}=0$ there. Thus
$\bar q_{\rho,\bar\delta}\le\bar p$ Lebesgue-almost everywhere on
$\mathcal H^K$.

Consequently, for every Borel set $B\subseteq\mathcal H^K$,
\[
    \bar Q_{\rho,\bar\delta}(B)
    =
    \int_B\bar q_{\rho,\bar\delta}(\eta)\,\dd\eta
    \le
    \int_B\bar p(\eta)\,\dd\eta
    =
    \bar P(B).
\]
Finally, for every Borel set
$E\subseteq(\mathcal I)^K$,
\[
\begin{aligned}
    Q_{\rho,\bar\delta}(E)
    =
    \bar Q_{\rho,\bar\delta}\bigl(s(E)\bigr)
    \le
    \bar P\bigl(s(E)\bigr)
    =
    P(E).
\end{aligned}
\]
Hence $Q_{\rho,\bar\delta}\le P$. 
\end{proof}

\begin{lemma}[Properties of $Q_{\rho,\bar\delta}$]
\label{lem:q-rho-properties-transfer-lower}
Except on the $Q_{\rho,\bar\delta}$-null tie set $\{\Delta=0\}$, let
$\Xi$, $\Delta$, and $H$ denote the unique background oracle, local gap coordinate, and sign associated with the representation
$\mu=\mu(\Xi,\Delta,H)$. Then
\begin{equation}
\label{eq:Q-coordinate-factorization-transfer-lower}
    Q_{\rho,\bar\delta}
    \left(
        \Xi\in\dd\xi,\,
        \Delta\in\dd\delta,\,
        H=h
    \right)
    =
    \Lambda_\rho(\dd\xi)\,\dd\delta,
    \qquad
    h\in\{+1,-1\}.
\end{equation}
Consequently,
\[
    Q_{\rho,\bar\delta}
    \bigl(
        H=+1
        \mid
        \Xi=\xi,\Delta=\delta
    \bigr)
    =
    Q_{\rho,\bar\delta}
    \bigl(
        H=-1
        \mid
        \Xi=\xi,\Delta=\delta
    \bigr)
    =
    \frac12,
\]
and
\[
    Q_{\rho,\bar\delta}(\Delta\in\dd\delta)
    =
    c_\rho\,\dd\delta,
    \qquad
    0\le\delta\le\bar\delta.
\]
In particular,
$Q_{\rho,\bar\delta}((\mathcal I)^K)=c_\rho\bar\delta$.
\end{lemma}

\begin{proof}
The uniqueness of the pair follows from the strict top-two property in
Lemma~\ref{lem:uniform-density-transfer-lower}. When $\delta>0$, the sign
is determined by $h=\operatorname{sign}(\eta_i-\eta_j)$, and the center
and remaining coordinates are recovered from the inverse transformation
given above. The two sign representations coincide only when $\delta=0$,
which is a $Q_{\rho,\bar\delta}$-null set. Define the coordinates
arbitrarily on this null set.

For Borel sets $C\subseteq\mathcal C_\rho$ and
$D\subseteq[0,\bar\delta]$, the definition
\eqref{eq:Q-rho-definition-transfer-lower} gives
\[
\begin{aligned}
    Q_{\rho,\bar\delta}
    \bigl(
        \Xi\in C,\,
        \Delta\in D,\,
        H=h
    \bigr)
    =
    \int_C\int_D
    \dd\delta\,\Lambda_\rho(\dd\xi)
    =
    \Lambda_\rho(C)\operatorname{Leb}(D).
\end{aligned}
\]
This proves
\eqref{eq:Q-coordinate-factorization-transfer-lower}. Since the two signs
have the same mass conditional on $(\Xi,\Delta)$, integrating over $\xi$ and summing over the
two signs gives
\[
    Q_{\rho,\bar\delta}(\Delta\in\dd\delta)
    =
    2\Lambda_\rho(\mathcal C_\rho)\,\dd\delta
    =
    c_\rho\,\dd\delta.
\]
\end{proof}

\begin{lemma}[Probability Domination Under $Q_{\rho,\bar\delta}$]
\label{lem:risk-domination-transfer-lower}
For any policy $\pi$, define the abstention mass
\[
    \mathcal A_T^{Q_{\rho,\bar\delta}}(\pi)
    :=
    \int
        \mathbb P_\mu^\pi(\widehat a_T=?)
    \,Q_{\rho,\bar\delta}(\dd\mu)
\]
and the undetected error mass
\[
    \mathcal E_T^{Q_{\rho,\bar\delta}}(\pi)
    :=
    \int
        \mathbb P_\mu^\pi
        \bigl(
            \widehat a_T
            \notin
            \{a^\star(\mu),?\}
        \bigr)
    \,Q_{\rho,\bar\delta}(\dd\mu).
\]
Then
\begin{equation}
\label{eq:risk-domination-transfer-lower}
    \mathcal A_T^{Q_{\rho,\bar\delta}}(\pi)
    \le
    \mathcal A_T(\pi),
    \qquad
    \mathcal E_T^{Q_{\rho,\bar\delta}}(\pi)
    \le
    \mathcal E_T(\pi).
\end{equation}
In particular, any policy satisfying
$\mathcal A_T(\pi)\le\alpha$ also satisfies
$\mathcal A_T^{Q_{\rho,\bar\delta}}(\pi)\le\alpha$.
\end{lemma}

\begin{proof}
Follows directly from $Q_{\rho,\bar\delta}\le P$. 
\end{proof}
\subsubsection{Step 2. Oracle Reduction}

Before sampling begins, we reveal the \emph{background oracle} $\Xi=(I,J,U,Z_{-IJ})$, which depends on the random means $\mu$, to the learner. Conditional on
$\Xi=\xi=(i,j,u,\zeta_{-ij})$, the learner knows the identities of the
local top-two arms, the local center $u$, and the information-coordinate
values $\zeta_k$ of all background arms. Equivalently, the learner knows
that each background arm $k\neq i,j$ has quality parameter
$s^{-1}(\zeta_k)$ and reward distribution $P_{s^{-1}(\zeta_k)}$.

The only unknown local coordinates under $Q_{\rho,\bar\delta}$ are
\[
    \Delta\in[0,\bar\delta],
    \qquad
    H\in\{+1,-1\}.
\]
Here $H=+1$ means that arm $i$ is better, whereas $H=-1$ means that arm
$j$ is better.

For ease of notation, denote the probability law of the experiment conditional on $\Xi$, $H$ and $\Delta$
\[
    \mathbb P_{h,\delta}^{\xi,\pi}
    :=
    \mathbb P_{\mu(\xi,\delta,h)}^\pi,
    \qquad
    h\in\{+1,-1\},
    \quad
    \delta\in[0,\bar\delta].
\]

Since any policy in the original experiment can be implemented in the oracle-revealed setting by simply ignoring the side information $\Xi$, any lower bound derived for this oracle-revealed experiment provides a lower bound for the original problem. Furthermore, since the background means are known, querying any arm other than $i$ or $j$ provides no information about $\Delta$ and $H$. We show that we can restrict our analysis to strategies that only sample the top two arms and output the decision of the sign; we call these {\em local pair strategies}.

\begin{lemma}[Sufficiency of Local Pair Strategies]
\label{lem:oracle-reduction-transfer-lower}
Fix a realization of the background oracle $\Xi=\xi=(i,j,u,\zeta_{-ij}) \in \mathcal C_\rho$. For any adaptive strategy $\pi$, there exists a modified strategy $\pi'$ that samples only arms $i$ and $j$, restricts its terminal recommendation to $\hat{a}'_T \in \{+,-,?\}$, and satisfies, for every $\delta\in(0,\bar\delta]$ and $h\in\{+1,-1\}$,
\begin{equation}
    \mathbb{P}_{h,\delta}^{\xi,\pi'}(\hat{a}'_T = ?) = \mathbb{P}_{h,\delta}^{\xi,\pi}(\hat{a}_T = ?) , \quad\mbox{and}\quad
\mathbb{P}_{h,\delta}^{\xi,\pi'}(\hat{a}'_T = -h) \le \mathbb{P}_{h,\delta}^{\xi,\pi}(\hat{a}_T \notin \{a^\star(h), ?\}), \label{eqn:simulate2}
\end{equation}
where $a^\star(+) = i$ and $a^\star(-) = j$. This says that the abstention probability of $\pi'$ is unchanged from that of $\pi$ and the undetected error probability of $\pi'$ is less than or equal to that of $\pi$.
\end{lemma}

\begin{proof}
Fix $\xi=(i,j,u,\zeta_{-ij})\in\mathcal C_\rho$. We construct the
modified strategy $\pi'$ that satisfies \eqref{eqn:simulate2} by internally simulating $\pi$.

Whenever the internal copy of $\pi$ samples arm $k\in\{i,j\}$, the strategy
$\pi'$ samples the same arm and feeds the observed reward to the internal
copy. Whenever the internal copy samples an arm $k\notin\{i,j\}$, the
strategy $\pi'$ instead samples arm $i$, discards the resulting reward, and
feeds the internal copy an independent simulated reward drawn from $P_{s^{-1}(\zeta_k)}$.
Conditional on $\Xi=\xi$, the quality parameter of every background arm
$k\notin\{i,j\}$ is known to be $s^{-1}(\zeta_k)$, and its reward
distribution is independent of $(H,\Delta)$. Therefore, for every fixed
$(h,\delta)$, the simulated reward has the same distribution as the reward
that $\pi$ would have observed by sampling arm $k$ in the original
experiment.

It follows inductively over the sampling rounds that the internal history
observed by $\pi$ under $\pi'$ has the same distribution as the history
observed by $\pi$ under
$\mathbb P_{h,\delta}^{\xi,\pi}$. Consequently, the terminal output of the
internal copy has the same distribution as $\hat a_T$ under the original
strategy.

At time $T$, the strategy $\pi'$ maps the terminal output of the internal
copy according to
\[
    i\mapsto +,\qquad
    j\mapsto -,\qquad
    ?\mapsto ?,
\]
and maps every recommendation $k\notin\{i,j\}$ arbitrarily to one of the
two signs, say $+$. This mapping preserves the abstention event, and hence
\[
    \mathbb P_{h,\delta}^{\xi,\pi'}(\hat a'_T=?)
    =
    \mathbb P_{h,\delta}^{\xi,\pi}(\hat a_T=?).
\]

Moreover, on the support of $Q_{\rho,\bar\delta}$, arms $i$ and $j$ are
the unique top-two arms, with $a^\star(+)=i$ and $a^\star(-)=j$. Therefore,
any original recommendation $k\notin\{i,j\}$ is already an undetected
error. Mapping such a recommendation to a sign either preserves that error
or converts it into a correct sign decision. Since the recommendations
$i$ and $j$ are mapped to their corresponding signs, the mapping cannot
turn a correct original recommendation into an incorrect sign decision.
Thus
    $\left\{\hat a'_T=-h\right\}
    \subseteq
    \left\{
        \hat a_T\notin\{a^\star(h),?\}
    \right\}$
under the coupled construction, and consequently
\[
    \mathbb P_{h,\delta}^{\xi,\pi'}(\hat a'_T=-h)
    \le
    \mathbb P_{h,\delta}^{\xi,\pi}
    \bigl(\hat a_T\notin\{a^\star(h),?\}\bigr).
\]
This proves \eqref{eqn:simulate2}. 
\end{proof}


\subsubsection{Step 3. Log-Likelihood Ratio and MLE}

Fix $D>0$, to be chosen later, and set
\[
    \bar\delta_\alpha:=D\alpha,
    \qquad
    L_T:=\frac{\bar\delta_\alpha\sqrt T}{2}
    =
    \frac{D\alpha\sqrt T}{2}.
\]
For all sufficiently small $\alpha$, we have
$\bar\delta_\alpha\le\delta_\rho$.

Fix a realization of the background oracle
$\Xi=\xi=(i,j,u,\zeta_{-ij})\in\mathcal C_\rho$ and a local-pair strategy
$\pi$. Recall that
\[
    \mathbb P_{h,\delta}^{\xi,\pi}
    :=
    \mathbb P_{\mu(\xi,\delta,h)}^\pi,
    \qquad
    h\in\{+1,-1\},
    \quad
    \delta\in[0,\bar\delta_\alpha].
\]
At $\delta=0$, the two signs induce the same parameter vector. We denote
their common law by
    $\mathbb P_0^{\xi,\pi}
    :=
    \mathbb P_{+1,0}^{\xi,\pi}
    =
    \mathbb P_{-1,0}^{\xi,\pi}$.
When the strategy is clear, we suppress the superscript $\pi$.

Consider the experiment where only two arms $i$ or $j$ can be sampled, let
    $H_t^\xi
    :=
    (\xi,A_1,X_1,\ldots,A_t,X_t)$
denote the \emph{oracle history} up to time $t$.
Let $(\mathcal H_T^\xi,\mathscr H_T^\xi)$ denote the measurable space of oracle histories for this two-arm experiment. We write
\[
    \omega
    =
    (\xi,a_1,x_1,\ldots,a_T,x_T)
    \in\mathcal H_T^\xi
\]
for a realization, and let
    $\omega_{t-1}
    :=
    (\xi,a_1,x_1,\ldots,a_{t-1},x_{t-1})$
denote its restriction to the first $t-1$ rounds.

Write
\[
    \theta_0:=\beta(u),
    \qquad
    a_0:=A''(\theta_0).
\]
For $h\in\{+1,-1\}$ and
$\delta\in[0,\bar\delta_\alpha]$, define the \emph{signed rescaled local gap}
    $$r:=h\frac{\delta\sqrt T}{2}
    \in[-L_T,L_T].$$
The natural parameters of the two local arms are then
\[
    \theta_i(r)
    =
    \theta_0+\frac{r}{\sqrt{Ta_0}},
    \qquad
    \theta_j(r)
    =
    \theta_0-\frac{r}{\sqrt{Ta_0}}.
\]

For a realized history $\omega\in\mathcal H_T^\xi$, define
\[
    S_t(\omega)
    :=
    \begin{cases}
        +1, & a_t=i,\\
        -1, & a_t=j.
    \end{cases}
\]
Thus, under the signed rescaled gap $r$, the natural parameter of the arm
sampled at time $t$ is
    $\theta_0+\frac{S_t(\omega)r}{\sqrt{Ta_0}}$.

For $h\in\{+1,-1\}$ and
$\delta\in[0,\bar\delta_\alpha]$, define the log-likelihood ratio of the oracle history relative to the tie law:
\begin{equation} \label{def_loglikelihood}
    \ell_{T,\xi}
    \left(
        r; \omega
    \right)=
    \ell_{T,\xi}
    \left(
        h\frac{\delta\sqrt T}{2};
        \omega
    \right)
    :=
    \log
    \frac{
        \dd\mathbb P_{h,\delta}^{\xi}
    }{
        \dd\mathbb P_0^\xi
    }(\omega).
\end{equation}
This definition is unambiguous at $\delta=0$ because
$\mathbb P_{+1,0}^{\xi}=\mathbb P_{-1,0}^{\xi}$.

\begin{lemma}[Local Likelihood Factorization]
\label{lem:local-likelihood-sufficient-transfer-lower}
For every $h\in\{+1,-1\}$ and
$\delta\in[0,\bar\delta_\alpha]$, let
$r=h\delta\sqrt T/2$. Then
$\mathbb P_{h,\delta}^{\xi}\ll\mathbb P_0^\xi$, and for
$\mathbb P_0^\xi$-almost every
$\omega\in\mathcal H_T^\xi$,
\begin{equation}
\label{eq:exact-local-likelihood-transfer-lower}
\begin{aligned}
    \ell_{T,\xi}(r;\omega)
    =
    \sum_{t=1}^T
    \left[
        \frac{S_t(\omega)r}{\sqrt{Ta_0}}t(x_t)
        -
        A\left(
            \theta_0+
            \frac{S_t(\omega)r}{\sqrt{Ta_0}}
        \right)
        +
        A(\theta_0)
    \right].
\end{aligned}
\end{equation}
\end{lemma}

\begin{proof}
Let $\pi_t(a_t\mid\xi,\omega_{t-1})$ denote the sampling kernel used by
the strategy at time $t$, evaluated at the realized history
$\omega_{t-1}$. This kernel is fixed by the strategy and is the same
measurable function of the realized past under every value of
$(h,\delta)$.

For a natural parameter $\theta\in\Theta$, write
\[
    p_\theta(x)
    :=
    h(x)\exp\{\theta t(x)-A(\theta)\}.
\] 
Under $\mathbb P_{h,\delta}^{\xi}$, where
$r=h\delta\sqrt T/2$, the density of the realized history $\omega$ is
\[
\begin{aligned}
    p_{h,\delta}^\xi(\omega)
    =
    \prod_{t=1}^T
        \pi_t(a_t\mid\xi,\omega_{t-1})
    \prod_{t=1}^T
        p_{\theta_0+S_t(\omega)r/\sqrt{Ta_0}}(x_t).
\end{aligned}
\]
Under the \emph{tie law} (where $\delta=0$),
\[
    p_0^\xi(\omega)
    =
    \prod_{t=1}^T
        \pi_t(a_t\mid\xi,\omega_{t-1})
    \prod_{t=1}^T
        p_{\theta_0}(x_t).
\]
The sampling kernels $\pi_t$ and carrier densities $h$ cancel in the likelihood ratio.
Therefore,
\[
\begin{aligned}
    \log
    \frac{
        p_{h,\delta}^\xi(\omega)
    }{
        p_0^\xi(\omega)
    }
    =
    \sum_{t=1}^T
    \left[
        \frac{S_t(\omega)r}{\sqrt{Ta_0}}t(x_t)
        -
        A\left(
            \theta_0+
            \frac{S_t(\omega)r}{\sqrt{Ta_0}}
        \right)
        +
        A(\theta_0)
    \right],
\end{aligned}
\]
which proves
\eqref{eq:exact-local-likelihood-transfer-lower}.

\end{proof}

We use the likelihood function to define a scalar statistic that orders full histories. For $\omega\in\mathcal H_T^\xi$, let
\begin{equation}
\label{eq:local-mle-transfer-lower}
    V_{T,\xi}(\omega)
    :=
    \argmax_{|r|\le L_T}
        \ell_{T,\xi}(r;\omega)
\end{equation}
be the bounded local maximum-likelihood estimator of the signed rescaled gap $r=h\frac{\delta\sqrt T}{2}$.
Differentiating \eqref{eq:exact-local-likelihood-transfer-lower} gives
\[
\begin{aligned}
    \ell_{T,\xi}'(r;\omega)
    =
    \frac{1}{\sqrt{Ta_0}}
    \sum_{t=1}^T
        S_t(\omega)
        \left[
            t(x_t)
            -
            A'\left(
                \theta_0+
                \frac{S_t(\omega)r}{\sqrt{Ta_0}}
            \right)
        \right],
\end{aligned}
\]
and
\begin{equation}
\label{eq:exact-curvature-transfer-lower}
    -\ell_{T,\xi}''(r;\omega)
    =
    \frac{1}{Ta_0}
    \sum_{t=1}^T
        A''\left(
            \theta_0+
            \frac{S_t(\omega)r}{\sqrt{Ta_0}}
        \right)
    >
    0.
\end{equation}
Hence $\ell_{T,\xi}(\cdot;\omega)$ is strictly concave on
$[-L_T,L_T]$, and its maximizer is unique. Since
$\ell_{T,\xi}(r;\omega)$ is measurable in $\omega$ and continuous in $r$,
the map $\omega\mapsto V_{T,\xi}(\omega)$ is measurable.

\subsubsection{Step 4. The Resulting Bayesian Flat Sign Experiment}

Recall that we defined
\[
    G:=\frac{\Delta\sqrt T}{2},
    \qquad
    L_T:=\frac{\bar\delta_\alpha\sqrt T}{2}.
\]
By the coordinate factorization of $Q_{\rho,\bar\delta_\alpha}$ and the
change of variables
\[
    g=\frac{\delta\sqrt T}{2},
    \qquad
    \delta=\frac{2g}{\sqrt T},
    \qquad
    \dd\delta=\frac{2}{\sqrt T}\,\dd g,
\]
we have
\[
    Q_{\rho,\bar\delta_\alpha}
    \left(
        \Xi\in\dd\xi,\,
        G\in\dd g,\,
        H=h
    \right)
    =
    \frac{2}{\sqrt T}
    \Lambda_\rho(\dd\xi)\,\dd g,
    \qquad
    h\in\{+1,-1\},
\]
for $g\in[0,L_T]$.

Fix $\xi\in\mathcal C_\rho$. Recall $(\mathcal H_T^\xi,\mathscr H_T^\xi)$ is the oracle history space, and 
$\omega\in\mathcal H_T^\xi$ denotes a realization. For every
$E\in\mathscr H_T^\xi$, define the \emph{unnormalized marginal measures of the history}, corresponding to two values of the sign $h$,
\begin{equation}
\label{eq:sign-mixture-measures-transfer-lower}
\begin{aligned}
    \mathbb A_{T,\xi}(E)
    :=
    \int_0^{L_T}
        \mathbb P_{+1,\,2g/\sqrt T}^{\xi}(E)
    \,\dd g,\quad
    \mathbb B_{T,\xi}(E)
    :=
    \int_0^{L_T}
        \mathbb P_{-1,\,2g/\sqrt T}^{\xi}(E)
    \,\dd g.
\end{aligned}
\end{equation}
Let
\[
    \mathbb M_{T,\xi}
    :=
    \mathbb A_{T,\xi}
    +
    \mathbb B_{T,\xi}.
\]
In particular,
\[
    \mathbb A_{T,\xi}(\mathcal H_T^\xi)
    =
    \mathbb B_{T,\xi}(\mathcal H_T^\xi)
    =
    L_T,
    \qquad
    \mathbb M_{T,\xi}(\mathcal H_T^\xi)
    =
    2L_T.
\]
and define their \emph{overlap measure}
    $\mathbb N_{T,\xi}
    :=
    \mathbb A_{T,\xi}\wedge\mathbb B_{T,\xi}$
by
\[
    \frac{\dd\mathbb N_{T,\xi}}
         {\dd\mathbb M_{T,\xi}}
    :=
    \min\left\{
        \frac{\dd\mathbb A_{T,\xi}}
             {\dd\mathbb M_{T,\xi}},
        \frac{\dd\mathbb B_{T,\xi}}
             {\dd\mathbb M_{T,\xi}}
    \right\}.
\]

Recall that, from \eqref{def_loglikelihood}, for $g\in[0,L_T]$,
\[
    \frac{
        \dd\mathbb P_{+1,\,2g/\sqrt T}^{\xi}
    }{
        \dd\mathbb P_0^\xi
    }(\omega)
    =
    \exp\{\ell_{T,\xi}(g;\omega)\},
\]
whereas
\[
    \frac{
        \dd\mathbb P_{-1,\,2g/\sqrt T}^{\xi}
    }{
        \dd\mathbb P_0^\xi
    }(\omega)
    =
    \exp\{\ell_{T,\xi}(-g;\omega)\}.
\]
Tonelli's theorem therefore gives
\begin{equation}
\label{eq:mixture-likelihood-densities-transfer-lower}
\begin{aligned}
    \frac{\dd\mathbb A_{T,\xi}}
         {\dd\mathbb P_0^\xi}(\omega)
    &=
    \int_0^{L_T}
        \exp\{\ell_{T,\xi}(g;\omega)\}
    \,\dd g
    =\int_{0}^{L_T}
        \exp\{\ell_{T,\xi}(r;\omega)\}
    \,\dd r\\
    \frac{\dd\mathbb B_{T,\xi}}
         {\dd\mathbb P_0^\xi}(\omega)
    &=
    \int_0^{L_T}
        \exp\{\ell_{T,\xi}(-g;\omega)\}
    \,\dd g
    =
    \int_{-L_T}^{0}
        \exp\{\ell_{T,\xi}(r;\omega)\}
    \,\dd r.
\end{aligned}
\end{equation}
Consequently,
\[
    \frac{\dd\mathbb M_{T,\xi}}
         {\dd\mathbb P_0^\xi}(\omega)
    =
    \int_{-L_T}^{L_T}
        \exp\{\ell_{T,\xi}(r;\omega)\}
    \,\dd r,
\]
and
\[
    \frac{\dd\mathbb N_{T,\xi}}
         {\dd\mathbb P_0^\xi}(\omega)
    =
    \min\left\{
        \int_0^{L_T}
            e^{\ell_{T,\xi}(r;\omega)}
        \,\dd r,\,
        \int_{-L_T}^{0}
            e^{\ell_{T,\xi}(r;\omega)}
        \,\dd r
    \right\}.
\]
In particular, wherever the denominator is positive,
\begin{equation}
\label{eq:overlap-ratio-likelihood-transfer-lower}
\begin{aligned}
    \frac{\dd\mathbb N_{T,\xi}}
         {\dd\mathbb M_{T,\xi}}(\omega)
    =
    \frac{
        \min\left\{
            \displaystyle
            \int_0^{L_T}
                e^{\ell_{T,\xi}(r;\omega)}
            \,\dd r,\,
            \displaystyle
            \int_{-L_T}^{0}
                e^{\ell_{T,\xi}(r;\omega)}
            \,\dd r
        \right\}
    }{
        \displaystyle
        \int_{-L_T}^{L_T}
            e^{\ell_{T,\xi}(r;\omega)}
        \,\dd r
    }.
\end{aligned}
\end{equation}

Let
    $\psi^\xi
    =
    \bigl(
        \psi_+^\xi,\psi_-^\xi,\psi_?^\xi
    \bigr)$
denote the randomized terminal rule of an oracle local-pair strategy. Thus,
for every $\omega\in\mathcal H_T^\xi$,
\[
    \psi_+^\xi(\omega)
    +
    \psi_-^\xi(\omega)
    +
    \psi_?^\xi(\omega)
    =
    1,
\]
where $\psi_+^\xi(\omega)$, $\psi_-^\xi(\omega)$, and
$\psi_?^\xi(\omega)$ are the probabilities of recommending $+$,
recommending $-$, and abstaining after observing the history $\omega$.

\begin{lemma}[Risk Representation in the Flat Sign Experiment]
\label{lem:flat-sign-risk-representation-transfer-lower}
For every terminal rule $\psi$, its abstention mass
under $Q_{\rho,\bar\delta_\alpha}$ is
\begin{equation}
\label{eq:flat-sign-abstention-transfer-lower}
    \mathcal A_T^{Q_{\rho,\bar\delta_\alpha}}(\psi)
    =
    \frac{2}{\sqrt T}
    \int_{\mathcal C_\rho}
    \int_{\mathcal H_T^\xi}
        \psi_?^\xi(\omega)
    \,\mathbb M_{T,\xi}(\dd\omega)\,
    \Lambda_\rho(\dd\xi).
\end{equation}
Its undetected error mass is
\begin{equation}
\label{eq:flat-sign-error-transfer-lower}
\begin{aligned}
    \mathcal E_T^{Q_{\rho,\bar\delta_\alpha}}(\psi)
    =
    \frac{2}{\sqrt T}
    \int_{\mathcal C_\rho}
    \bigg[
        \int_{\mathcal H_T^\xi}
            \psi_-^\xi(\omega)
        \,\mathbb A_{T,\xi}(\dd\omega)
        +
        \int_{\mathcal H_T^\xi}
            \psi_+^\xi(\omega)
        \,\mathbb B_{T,\xi}(\dd\omega)
    \bigg]
    \Lambda_\rho(\dd\xi).
\end{aligned}
\end{equation}
Moreover,
\begin{equation}
\label{eq:flat-sign-overlap-lower-transfer-lower}
    \mathcal E_T^{Q_{\rho,\bar\delta_\alpha}}(\psi)
    \ge
    \frac{2}{\sqrt T}
    \int_{\mathcal C_\rho}
    \int_{\mathcal H_T^\xi}
        \bigl(1-\psi_?^\xi(\omega)\bigr)
    \,\mathbb N_{T,\xi}(\dd\omega)\,
    \Lambda_\rho(\dd\xi).
\end{equation}
\end{lemma}

\begin{proof}
For fixed $(\xi,\delta,h)$, the conditional abstention probability induced
by $\psi^\xi$ is
\begin{equation} \label{integral_of_abs}
    \mathbb P_{h,\delta}^{\xi}(\hat a_T=?)
    =
    \int_{\mathcal H_T^\xi}
        \psi_?^\xi(\omega)
    \,\mathbb P_{h,\delta}^{\xi}(\dd\omega).
\end{equation}\
Hence we can write
\[
\begin{aligned}
    \mathcal A_T^{Q_{\rho,\bar\delta_\alpha}}(\psi)
    &= \int_{(\mathcal I)^K} \mathbb P_\mu(\hat a_T =?) Q_{\rho, \bar\delta_\alpha}(\dd \mu) \\
    &\stackrel{(a)}{=}
    \sum_{h\in\{+1,-1\}}
    \int_{\mathcal C_\rho}
    \int_0^{\bar\delta_\alpha}
        \mathbb P_{h,\delta}^{\xi}
        (\hat a_T=?)
    \,\dd\delta\,
    \Lambda_\rho(\dd\xi)\\
    &\stackrel{(b)}{=}
    \sum_{h\in\{+1,-1\}}
    \int_{\mathcal C_\rho}
    \int_0^{\bar\delta_\alpha}
    \int_{\mathcal H_T^\xi}
        \psi_?^\xi(\omega)
    \,\mathbb P_{h,\delta}^{\xi}(\dd\omega)
    \,\dd\delta\,
    \Lambda_\rho(\dd\xi).\\
    &\stackrel{(c)}{=}
    \frac{2}{\sqrt T}
    \int_{\mathcal C_\rho}
    \int_0^{L_T}
    \int_{\mathcal H_T^\xi}
        \psi_?^\xi(\omega)
    \,\mathbb P_{+1,\,2g/\sqrt T}^{\xi}(\dd\omega)
    \,\dd g\,
    \Lambda_\rho(\dd\xi)\\
    &\quad+
    \frac{2}{\sqrt T}
    \int_{\mathcal C_\rho}
    \int_0^{L_T}
    \int_{\mathcal H_T^\xi}
        \psi_?^\xi(\omega)
    \,\mathbb P_{-1,\,2g/\sqrt T}^{\xi}(\dd\omega)
    \,\dd g\,
    \Lambda_\rho(\dd\xi).
\end{aligned}
\]
where $(a)$ is due to the coordinate factorization of $Q_{\rho, \bar\delta_\alpha}$~\eqref{eq:Q-coordinate-factorization-transfer-lower}, $(b)$ is due to~\eqref{integral_of_abs} and $(c)$ is due to the change of variable with $g=\delta\sqrt T/2$ separately for the two signs. Hence by the definitions of $\mathbb M_{T,\xi}$, this is \eqref{eq:flat-sign-abstention-transfer-lower}.

An undetected sign error occurs when the rule recommends $-$ under
$H=+1$, or recommends $+$ under $H=-1$. Repeating similar arguments as for~\eqref{eq:flat-sign-abstention-transfer-lower} with the two error events gives \eqref{eq:flat-sign-error-transfer-lower}. Note the tie case $\delta=0$ does not
affect the calculation because it is null with respect to the
$\delta$-integration.

To prove the lower bound, write
    $a_{T,\xi}
    :=
    \frac{\dd\mathbb A_{T,\xi}}
         {\dd\mathbb M_{T,\xi}}$,
    and
    $b_{T,\xi}
    :=
    \frac{\dd\mathbb B_{T,\xi}}
         {\dd\mathbb M_{T,\xi}}$.
Since
$\mathbb M_{T,\xi}=\mathbb A_{T,\xi}+\mathbb B_{T,\xi}$,
\[
    a_{T,\xi}(\omega)+b_{T,\xi}(\omega)
    =
    1
    \qquad
    \mathbb M_{T,\xi}\text{-a.e.}
\]
For $\mathbb M_{T,\xi}$-almost every $\omega$,
\[
\begin{aligned}
    \psi_-^\xi(\omega)a_{T,\xi}(\omega)
    +
    \psi_+^\xi(\omega)b_{T,\xi}(\omega)
    &\ge
    \bigl(
        \psi_-^\xi(\omega)
        +
        \psi_+^\xi(\omega)
    \bigr)
    \min\{
        a_{T,\xi}(\omega),
        b_{T,\xi}(\omega)
    \}\\
    &=
    \bigl(
        1-\psi_?^\xi(\omega)
    \bigr)
    \frac{\dd\mathbb N_{T,\xi}}
         {\dd\mathbb M_{T,\xi}}(\omega).
\end{aligned}
\]
Integrating with respect to $\mathbb M_{T,\xi}$, and then with respect to
$\Lambda_\rho$, proves
\eqref{eq:flat-sign-overlap-lower-transfer-lower}. 
\end{proof}

In particular, if
\[
    \mathcal A_T^{Q_{\rho,\bar\delta_\alpha}}(\psi)
    \le
    \alpha,
\]
then \eqref{eq:flat-sign-abstention-transfer-lower} implies
\begin{equation}
\label{eq:flat-sign-abstention-budget-transfer-lower}
    \int_{\mathcal C_\rho}
    \int_{\mathcal H_T^\xi}
        \psi_?^\xi(\omega)
    \,\mathbb M_{T,\xi}(\dd\omega)\,
    \Lambda_\rho(\dd\xi)
    \le
    \frac{\alpha\sqrt T}{2}.
\end{equation}
Thus the remaining problem is to study, as functions of the full history
$\omega$, the total unnormalized marginal measure of the history $\mathbb M_{T,\xi}$ and the overlap ratio
    $\frac{\dd\mathbb N_{T,\xi}}
         {\dd\mathbb M_{T,\xi}}(\omega)$.
Next, we compare this overlap ratio with
$\Phi(-|V_{T,\xi}(\omega)|)$ and use the local MLE
$V_{T,\xi}(\omega)$ in the bathtub argument.

\subsubsection{Step 5. Solving the Bayesian Flat Sign Experiment}

Intuitively, the measures $\mathbb M_{T,\xi}$ and $\mathbb N_{T,\xi}$ determine,
respectively, the available abstention mass and the unavoidable undetected error mass. We now compare these measures with their Gaussian analogues through the
local MLE $V_{T,\xi}$. We use the abbreviations
\[
    S_t:=S_t(H_T^\xi),
    \qquad
    \ell_{T,\xi}(r)
    :=
    \ell_{T,\xi}(r;H_T^\xi),
    \qquad
    V_{T,\xi}
    :=
    V_{T,\xi}(H_T^\xi).
\]

For $r\in[-L_T,L_T]$, define
\[
    \mathrm{sgn}(r)
    :=
    \begin{cases}
        +1, & r\ge0,\\
        -1, & r<0,
    \end{cases}
    \qquad
    \delta_T(r)
    :=
    \frac{2|r|}{\sqrt T}.
\]
At $r=0$, the choice of $\mathrm{sgn}(r)$ is immaterial because both signs induce the
same tie law.

\begin{lemma}[Uniform Curvature and Local-MLE Tightness]
\label{lem:curvature-tightness-transfer-lower}
Fix $\rho\in(0,1)$ and $D<\infty$. There exist constants
$\alpha_0>0$ and $C,C_1<\infty$ such that, for every
$0<\alpha\le\alpha_0$, the following statements hold uniformly over
$\xi\in\mathcal C_\rho$ and all local pair strategies.

First, for every history realization $\omega$,
\begin{equation}
\label{eq:curvature-bound-transfer-lower}
    1-C\alpha
    \le
    -\ell_{T,\xi}''(r;\omega)
    \le
    1+C\alpha,
    \qquad
    |r|\le L_T.
\end{equation}
In particular, $\ell_{T,\xi}(\cdot;\omega)$ is strictly concave on
$[-L_T,L_T]$, and the local MLE $V_{T,\xi}(\omega)$ is unique.

Second, for every $R>0$ and every $T$ such that $R\le L_T/2$,
\begin{equation}
\label{eq:local-mle-tightness-transfer-lower}
    \sup_{\xi\in\mathcal C_\rho}
    \sup_{|r|\le L_T/2}
    \mathbb P_{\mathrm{sgn}(r),\,\delta_T(r)}^\xi
    \left(
        |V_{T,\xi}-r|>R
    \right)
    \le
    \frac{C_1}{R^2}.
\end{equation}
\end{lemma}

\begin{proof}
Fix $\xi=(i,j,u,\zeta_{-ij})\in\mathcal C_\rho$, and recall that
    $\theta_0=\beta(u)$
    and
    $a_0=A''(\theta_0)$.
By \eqref{eq:exact-curvature-transfer-lower},
\[
    -\ell_{T,\xi}''(r;\omega)
    =
    \frac{1}{Ta_0}
    \sum_{t=1}^T
        A''\left(
            \theta_0+
            \frac{S_t(\omega)r}{\sqrt{Ta_0}}
        \right).
\]

Because $\mathcal C_\rho$ is compact, the centers $\theta_0=\beta(u)$
range over a compact subset of $\Theta$, and $\inf_{\xi\in\mathcal C_\rho}a_0>0$.
Moreover, for $|r|\le L_T$,
    $\left|
        \frac{r}{\sqrt{Ta_0}}
    \right|
    \le
    C_D\alpha$
uniformly over $\xi\in\mathcal C_\rho$. Thus, for all sufficiently small
$\alpha$, every perturbed natural parameter
    $\theta_0+\frac{S_t(\omega)r}{\sqrt{Ta_0}}$
belongs to a common compact subset of $\Theta$. Since $A''$ is locally
Lipschitz on this compact set, there exists $C<\infty$ such that
\[
    \left|
        \frac{
            A''\left(
                \theta_0+
                \frac{S_t(\omega)r}{\sqrt{Ta_0}}
            \right)
        }{
            A''(\theta_0)
        }
        -1
    \right|
    \le
    C\alpha.
\]
Averaging over $t$ proves
\eqref{eq:curvature-bound-transfer-lower}. By decreasing $\alpha_0$ if
necessary, we may assume $C\alpha_0<1$, so the likelihood is strictly
concave.

We next prove \eqref{eq:local-mle-tightness-transfer-lower}. Fix
$r\in[-L_T/2,L_T/2]$ and consider the experiment under
$\mathbb P_{\mathrm{sgn}(r),\,\delta_T(r)}^\xi$. Define the score at the true signed
parameter by
\[
    U_{T,\xi}(r)
    :=
    \ell_{T,\xi}'(r)
    =
    \sum_{t=1}^T D_t(r),
\]
where
\[
    D_t(r)
    :=
    \frac{S_t}{\sqrt{Ta_0}}
    \left[
        t(X_t)
        -
        A'\left(
            \theta_0+
            \frac{S_t r}{\sqrt{Ta_0}}
        \right)
    \right].
\]
Let
    $\mathcal G_t^\xi
    :=
    \sigma\bigl(
        \xi,A_1,X_1,\ldots,A_{t-1},X_{t-1},A_t
    \bigr)$
denote the information available after choosing $A_t$ but before observing
$X_t$. Conditional on $\mathcal G_t^\xi$, the reward has natural parameter
\[
    \theta_0+\frac{S_t r}{\sqrt{Ta_0}}.
\]
Therefore,
\[
    \mathbb E_{\mathrm{sgn}(r),\,\delta_T(r)}^\xi
    \left[
        D_t(r)\mid\mathcal G_t^\xi
    \right]
    =
    0
\]
and
\[
    \mathbb E_{\mathrm{sgn}(r),\,\delta_T(r)}^\xi
    \left[
        D_t(r)^2\mid\mathcal G_t^\xi
    \right]
    =
    \frac{
        A''\left(
            \theta_0+
            \frac{S_t r}{\sqrt{Ta_0}}
        \right)
    }{
        Ta_0
    }.
\]
Moreover, for $s<t$, the variable $D_s(r)$ is
$\mathcal G_t^\xi$-measurable. Hence
\[
\begin{aligned}
    \mathbb E_{\mathrm{sgn}(r),\,\delta_T(r)}^\xi
    \left[
        D_s(r)D_t(r)
    \right]
    &=
    \mathbb E_{\mathrm{sgn}(r),\,\delta_T(r)}^\xi
    \left[
        D_s(r)
        \mathbb E_{\mathrm{sgn}(r),\,\delta_T(r)}^\xi
        \left[
            D_t(r)\mid\mathcal G_t^\xi
        \right]
    \right]\\
    &=0.
\end{aligned}
\]
Thus the martingale increments are orthogonal in $L^2$, and therefore
\[
\begin{aligned}
    \mathbb E_{\mathrm{sgn}(r),\,\delta_T(r)}^\xi
    \left[
        U_{T,\xi}(r)^2
    \right]
    &=
    \sum_{t=1}^T
    \mathbb E_{\mathrm{sgn}(r),\,\delta_T(r)}^\xi
    \left[
        D_t(r)^2
    \right]\\
    &=
    \mathbb E_{\mathrm{sgn}(r),\,\delta_T(r)}^\xi
    \left[
        \sum_{t=1}^T
        \mathbb E_{\mathrm{sgn}(r),\,\delta_T(r)}^\xi
        \left[
            D_t(r)^2\mid\mathcal G_t^\xi
        \right]
    \right]\\
    &=
    \mathbb E_{\mathrm{sgn}(r),\,\delta_T(r)}^\xi
    \left[
        \frac{1}{Ta_0}
        \sum_{t=1}^T
        A''\left(
            \theta_0+
            \frac{S_t r}{\sqrt{Ta_0}}
        \right)
    \right]\\
    &\stackrel{(a)}{=} 
    \mathbb E_{\mathrm{sgn}(r),\,\delta_T(r)}^\xi
    \left[
        -\ell_{T,\xi}''(r)
    \right]\\
    &\stackrel{(b)}{\le} 
    1+C\alpha,
\end{aligned}
\]
where $(a)$ follows from
\eqref{eq:exact-curvature-transfer-lower}, and $(b)$ follows from \eqref{eq:curvature-bound-transfer-lower}.

Let $R\le L_T/2$. If $V_{T,\xi}\ge r+R$, then concavity implies
$\ell_{T,\xi}'(r+R)\ge0$. Hence,
\[
\begin{aligned}
    0
    &\le
    \ell_{T,\xi}'(r+R)\\
    &=
    \ell_{T,\xi}'(r)
    +
    \int_r^{r+R}
        \ell_{T,\xi}''(v)
    \,\dd v\\
    &\le
    U_{T,\xi}(r)
    -
    (1-C\alpha)R,
\end{aligned}
\]
so that
\[
    U_{T,\xi}(r)
    \ge
    (1-C\alpha)R.
\]
Similarly,
\[
    V_{T,\xi}\le r-R
    \quad\Longrightarrow\quad
    U_{T,\xi}(r)
    \le
    -(1-C\alpha)R.
\]
Consequently,
\[
    |V_{T,\xi}-r|>R
    \quad\Longrightarrow\quad
    |U_{T,\xi}(r)|
    \ge
    (1-C\alpha)R.
\]
Chebyshev's inequality gives
\[
\begin{aligned}
    \mathbb P_{\mathrm{sgn}(r),\,\delta_T(r)}^\xi
    \left(
        |V_{T,\xi}-r|>R
    \right)
    &\le
    \frac{
        1+C\alpha
    }{
        (1-C\alpha)^2R^2
    }.
\end{aligned}
\]
Since $\alpha\le\alpha_0$, the right-hand side is at most $C_1/R^2$
for a constant $C_1<\infty$ independent of $\xi$, $T$, $r$, and the
sampling strategy. 
\end{proof}

\begin{lemma}[Uniform Local Likelihood Curve Comparison]
\label{lem:local-likelihood-curve-transfer-lower}
Fix $\rho\in(0,1)$, $D<\infty$, and
$\varepsilon\in(0,1/10)$. Let
    $a_T
    :=
    \frac{\alpha\sqrt T}
         {2(1-3\varepsilon)c_\rho}$,
and assume
    $D>
    \frac{4}{(1-3\varepsilon)c_\rho}$,
so that
    $L_T
    =
    \frac{D\alpha\sqrt T}{2}
    >
    4a_T$.

For all sufficiently small $\alpha$, and then all sufficiently large $T$,
the following statements hold uniformly over
$\xi\in\mathcal C_\rho$ and all nonanticipating local-pair sampling
strategies.

First,
\begin{equation}
\label{eq:local-mass-transfer-lower}
    \mathbb M_{T,\xi}
    \left(
        |V_{T,\xi}|\le a_T
    \right)
    \ge
    2(1-\varepsilon)a_T.
\end{equation}
Second,
\begin{equation}
\label{eq:local-weight-transfer-lower}
\begin{aligned}
    \int_{\{a_T<|V_{T,\xi}|\le L_T/2\}}
        \Phi\bigl(-|V_{T,\xi}(\omega)|\bigr)
    \,\mathbb M_{T,\xi}(\dd\omega)
    \ge
    \exp\{-o_T(\alpha^2T)\}
    2(1-\varepsilon)
    \int_{a_T}^{L_T/2}
        \Phi(-u)
    \,\dd u.
\end{aligned}
\end{equation}
Third, on the central history set
$\{\omega : |V_{T,\xi}(\omega)|\le L_T/2\}$,
\begin{equation}
\label{eq:local-overlap-transfer-lower}
\begin{aligned}
    \frac{\dd\mathbb N_{T,\xi}}
         {\dd\mathbb M_{T,\xi}}(\omega)
    \ge
    (1-\varepsilon)
    \exp\{
        -\varepsilon\alpha^2T
        -o_T(\alpha^2T)
    \}
    \Phi\bigl(-|V_{T,\xi}(\omega)|\bigr),
    \qquad
    \mathbb M_{T,\xi}\text{-a.e.}
\end{aligned}
\end{equation}
The $o_T(\alpha^2T)$ terms are uniform over
$\xi\in\mathcal C_\rho$ and all nonanticipating local-pair sampling
strategies.
\end{lemma}

\begin{proof}
Recall we have established that, for every measurable (history event) $E \in \mathscr H^\xi_T$,
\begin{equation}
\label{eq:M-signed-mixture-transfer-lower}
\begin{aligned}
    \mathbb M_{T,\xi}(E)
    &=
    \int_0^{L_T}
        \mathbb P_{+1,\,2g/\sqrt T}^{\xi}(E)
    \,\dd g
    +
    \int_0^{L_T}
        \mathbb P_{-1,\,2g/\sqrt T}^{\xi}(E)
    \,\dd g\\
    &=
    \int_{-L_T}^{L_T}
        \mathbb P_{\mathrm{sgn}(r),\,\delta_T(r)}^\xi(E)
    \,\dd r.
\end{aligned}
\end{equation}

Choose $R>0$ sufficiently large that
    $\frac{C_1}{R^2}
    \le
    \frac{\varepsilon}{4}$,
where $C_1$ is the constant in
Lemma~\ref{lem:curvature-tightness-transfer-lower}. For all sufficiently
large $T$, we have $R<a_T$ and $R<L_T/2$. Also recall $L_T>4a_T$, hence $[-a_T+R,a_T-R] \subset [-L_T,L_T]$. And on this interval $|r|\le a_T-R$, hence $\{ |V_{T,\xi}-r|\le R \} \subset \{ |V_{T,\xi}|\le a_T \}$. Therefore,
\[
\begin{aligned}
    \mathbb M_{T,\xi}
    \left(
        |V_{T,\xi}|\le a_T
    \right)
    &=
    \int_{-L_T}^{L_T}
        \mathbb P_{\mathrm{sgn}(r),\,\delta_T(r)}^\xi
        \left(
            |V_{T,\xi}|\le a_T
        \right)
    \,\dd r\\
    &\ge
    \int_{-a_T+R}^{a_T-R}
        \mathbb P_{\mathrm{sgn}(r),\,\delta_T(r)}^\xi
        \left(
            |V_{T,\xi}-r|\le R
        \right)
    \,\dd r\\
    &\ge
    (2a_T-2R)
    \left(
        1-\frac{\varepsilon}{4}
    \right).
\end{aligned}
\]
Since $a_T\to\infty$ for every fixed $\alpha>0$, the final expression is
at least $2(1-\varepsilon)a_T$ for all sufficiently large $T$. This proves
\eqref{eq:local-mass-transfer-lower}.

For the second (weighted) bound, suppose first that $|V_{T,\xi}-r|\le R$ and $r\in
    \left[
        a_T+2R,\,
        \frac{L_T}{2}-R
    \right]$.
Then
    $a_T<V_{T,\xi}\le\frac{L_T}{2}$,
and hence
    $\Phi(-|V_{T,\xi}|)
    \ge
    \Phi(-(r+R))$.
Similarly, if $|V_{T,\xi}-r|\le R$ and $r\in
    \left[
        -\frac{L_T}{2}+R,\,
        -a_T-2R
    \right]$,
then
    $a_T<|V_{T,\xi}|\le\frac{L_T}{2}$,
and
    $\Phi(-|V_{T,\xi}|)
    \ge
    \Phi(-(|r|+R))$.
Using \eqref{eq:M-signed-mixture-transfer-lower}, we have
\[
\begin{aligned}
    &\int_{\{a_T<|V_{T,\xi}|\le L_T/2\}}
        \Phi\bigl(-|V_{T,\xi}(\omega)|\bigr)
    \,\mathbb M_{T,\xi}(\dd\omega)\\
    &=
    \int_{\mathcal H^\xi_T}
    \mathbbm 1
        \left\{
            a_T<|V_{T,\xi}(\omega)|
            \le
            \frac{L_T}{2}
        \right\}
        \Phi\bigl(-|V_{T,\xi}(\omega)|\bigr)
    \int_{-L_T}^{L_T}
    \dd r \,\mathbb P_{\mathrm{sgn}(r),\,\delta_T(r)}^\xi(\dd\omega)
    \\
    &=
    \int_{-L_T}^{L_T}
    \int_{\mathcal H^\xi_T}
        \mathbbm 1
        \left\{
            a_T<|V_{T,\xi}(\omega)|
            \le
            \frac{L_T}{2}
        \right\}
        \Phi\bigl(-|V_{T,\xi}(\omega)|\bigr)
    \,\mathbb P_{\mathrm{sgn}(r),\,\delta_T(r)}^\xi(\dd\omega)
    \,\dd r\\
    &\ge
    \int_{a_T+2R}^{L_T/2-R}
    \int_{\mathcal H^\xi_T}
        \mathbbm 1
        \left\{
            |V_{T,\xi}(\omega)-r|\le R
        \right\}
        \Phi(-(r+R))
    \,\mathbb P_{\mathrm{sgn}(r),\,\delta_T(r)}^\xi(\dd\omega)
    \,\dd r\\
    &\quad+
    \int_{-L_T/2+R}^{-a_T-2R}
    \int_{\mathcal H^\xi_T}
        \mathbbm 1
        \left\{
            |V_{T,\xi}(\omega)-r|\le R
        \right\}
        \Phi(-(|r|+R))
    \,\mathbb P_{\mathrm{sgn}(r),\,\delta_T(r)}^\xi(\dd\omega)
    \,\dd r\\
    &=
    \int_{a_T+2R}^{L_T/2-R}
        \Phi(-(r+R))
        \mathbb P_{\mathrm{sgn}(r),\,\delta_T(r)}^\xi
        \left(
            |V_{T,\xi}-r|\le R
        \right)
    \,\dd r\\
    &\quad+
    \int_{-L_T/2+R}^{-a_T-2R}
        \Phi(-(|r|+R))
        \mathbb P_{\mathrm{sgn}(r),\,\delta_T(r)}^\xi
        \left(
            |V_{T,\xi}-r|\le R
        \right)
    \,\dd r.
\end{aligned}
\]

By \eqref{eq:local-mle-tightness-transfer-lower},
\[
    \mathbb P_{\mathrm{sgn}(r),\,\delta_T(r)}^\xi
    \left(
        |V_{T,\xi}-r|\le R
    \right)
    \ge
    1-\frac{\varepsilon}{4}
\]
throughout both integration intervals. Then, with
the change of variables \(u=-r\), we get
\[
\begin{aligned}
    \int_{\{a_T<|V_{T,\xi}|\le L_T/2\}}
        \Phi\bigl(-|V_{T,\xi}(\omega)|\bigr)
    \,\mathbb M_{T,\xi}(\dd\omega)
    &\ge
    2\left(1-\frac{\varepsilon}{4}\right)
    \int_{a_T+2R}^{L_T/2-R}
        \Phi(-(r+R))
    \,\dd r\\
    &=
    2\left(1-\frac{\varepsilon}{4}\right)
    \int_{a_T+3R}^{L_T/2}
        \Phi(-u)
    \,\dd u.
\end{aligned}
\]

Since \(R\) is fixed, \(a_T\asymp\alpha\sqrt T\), and
\(L_T/2>2a_T\), Gaussian tail asymptotics ($\Phi(-x)=\frac{\phi(x)}{x}(1+O(x^{-2)})$ and $\int^\infty_x \Phi(-u) \dd u=\phi(x)-x\Phi(-x)$) imply
\[
    \log
    \frac{
        \displaystyle
        \int_{a_T+3R}^{L_T/2}
            \Phi(-u)\,\dd u
    }{
        \displaystyle
        \int_{a_T}^{L_T/2}
            \Phi(-u)\,\dd u
    }
    =
    -3Ra_T+O(1+\log a_T)
    =
    -o_T(\alpha^2T)
\]
for every fixed \(\alpha>0\). Therefore,
\[
\begin{aligned}
    &\int_{\{a_T<|V_{T,\xi}|\le L_T/2\}}
        \Phi\bigl(-|V_{T,\xi}(\omega)|\bigr)
    \,\mathbb M_{T,\xi}(\dd\omega)
    \ge
    \exp\{-o_T(\alpha^2T)\}
    2(1-\varepsilon)
    \int_{a_T}^{L_T/2}
        \Phi(-u)
    \,\dd u,
\end{aligned}
\]
which proves \eqref{eq:local-weight-transfer-lower}.

It remains to prove the third overlap bound. Fix a history realization $\omega$, and for ease of notation define $v:=V_{T,\xi}(\omega)$. Assume $|v|\le L_T/2$.
Since $v$ is an interior maximizer,
\[
    \ell_{T,\xi}'(v;\omega)=0.
\]
By \eqref{eq:curvature-bound-transfer-lower} in
Lemma~\ref{lem:curvature-tightness-transfer-lower}, for every
$r\in[-L_T,L_T]$,
\begin{equation}
\label{eq:quadratic-likelihood-comparison-transfer-lower}
\begin{aligned}
    \ell_{T,\xi}(v;\omega)
    -
    \frac{1+C\alpha}{2}(r-v)^2
    &\le
    \ell_{T,\xi}(r;\omega)\\
    &\le
    \ell_{T,\xi}(v;\omega)
    -
    \frac{1-C\alpha}{2}(r-v)^2.
\end{aligned}
\end{equation}
Recall from \eqref{eq:overlap-ratio-likelihood-transfer-lower} that
\[
\begin{aligned}
    \frac{\dd\mathbb N_{T,\xi}}
         {\dd\mathbb M_{T,\xi}}(\omega)
    =
    \frac{
        \min\left\{
            \displaystyle
            \int_0^{L_T}
                e^{\ell_{T,\xi}(r;\omega)}
            \,\dd r,\,
            \displaystyle
            \int_{-L_T}^{0}
                e^{\ell_{T,\xi}(r;\omega)}
            \,\dd r
        \right\}
    }{
        \displaystyle
        \int_{-L_T}^{L_T}
            e^{\ell_{T,\xi}(r;\omega)}
        \,\dd r
    }.
\end{aligned}
\]

Suppose first that $v\ge0$. By the lower bound in
\eqref{eq:quadratic-likelihood-comparison-transfer-lower},
\[
\begin{aligned}
    &\min\left\{
        \int_0^{L_T}
            e^{\ell_{T,\xi}(r;\omega)}
        \,\dd r,\,
        \int_{-L_T}^{0}
            e^{\ell_{T,\xi}(r;\omega)}
        \,\dd r
    \right\} \\
    &\quad\ge e^{\ell_{T,\xi}(v;\omega)}\min\left\{
        \int_0^{L_T}
            e^{-(1+C\alpha)(r-v)^2/2}
        \,\dd r,\,
        \int_{-L_T}^{0}
            e^{-(1+C\alpha)(r-v)^2/2}
        \,\dd r
    \right\} \\
    &\quad\ge
    e^{\ell_{T,\xi}(v;\omega)}
    \int_{-L_T}^{0}
        \exp\left\{
            -\frac{1+C\alpha}{2}(r-v)^2
        \right\}
    \,\dd r.
\end{aligned}
\]
By the upper bound in
\eqref{eq:quadratic-likelihood-comparison-transfer-lower},
\[
\begin{aligned}
    \int_{-L_T}^{L_T}
        e^{\ell_{T,\xi}(r;\omega)}
    \,\dd r
    &\le
    e^{\ell_{T,\xi}(v;\omega)}
    \int_{-\infty}^{\infty}
        \exp\left\{
            -\frac{1-C\alpha}{2}(r-v)^2
        \right\}
    \,\dd r\\
    &=
    e^{\ell_{T,\xi}(v;\omega)}
    \sqrt{\frac{2\pi}{1-C\alpha}}.
\end{aligned}
\]
Thus
\[
\begin{aligned}
    \frac{\dd\mathbb N_{T,\xi}}
         {\dd\mathbb M_{T,\xi}}(\omega)
    &\ge
    \sqrt{
        \frac{1-C\alpha}{1+C\alpha}
    }
    \left[
        \Phi\left(
            -\sqrt{1+C\alpha}\,v
        \right)
        -
        \Phi\left(
            -\sqrt{1+C\alpha}\,(L_T+v)
        \right)
    \right].
\end{aligned}
\]
The case $v\le0$ is symmetric and yields the same bound with $|v|$ in
place of $v$.

Since $\Phi(-x)$ is decreasing, uniformly over $|v|\le L_T/2$, 
\[
    \frac{
        \Phi\left(
            -\sqrt{1+C\alpha}(L_T+|v|)
        \right)
    }{
        \Phi\left(
            -\sqrt{1+C\alpha}|v|
        \right)
    }
    \leq \frac{\Phi\left(-\sqrt{1+C\alpha} L_T\right)}{ \Phi\left(-\sqrt{1+C\alpha}L_T/2\right)}
    =
    o_T(1).
\]
Moreover, note $\frac{\dd}{\dd t}  \log \Phi(-t) = -\frac{\phi(t)}{\Phi(-t)}$. By Mills ratio bound $\Phi(-t)\geq \frac{t}{1+t^2}\phi(t)$, we have $$
\begin{aligned}
\log \frac{\Phi\left(-\sqrt{1+C\alpha}\,v\right)}{\Phi\left(-v\right)} 
&= -\int^{\sqrt{1+C\alpha}\,v}_v \frac{\phi(t)}{\Phi(-t)} \dd t \\
&\geq -\int^{\sqrt{1+C\alpha}\,v}_v (t+\frac{1}{t}) \dd t \\
&= -\frac{1+C\alpha-1}{2}v^2 - \log\left(\sqrt{1+C\alpha}\right)
\end{aligned}
$$
And note $(1+C\alpha)^{-1/2}\geq 1-\frac{C}{2}\alpha$. Hence there exists $C_2=C/2<\infty$ such that, for all sufficiently small
$\alpha$ and all $v>0$,
\[
    \Phi\left(
        -\sqrt{1+C\alpha}\,v
    \right)
    \ge
    (1-C_2\alpha)
    \exp\{-C_2\alpha v^2\}
    \Phi(-v).
\]

Finally, uniformly over $|v|\le L_T/2$,
\[
\begin{aligned}
    \frac{\dd\mathbb N_{T,\xi}}
         {\dd\mathbb M_{T,\xi}}(\omega)
    &\ge
    \sqrt{
        \frac{1-C\alpha}{1+C\alpha}
    }
    \left[
        \Phi\left(
            -\sqrt{1+C\alpha}\,|v|
        \right)
        -
        \Phi\left(
            -\sqrt{1+C\alpha}\,(L_T+|v|)
        \right)
    \right] \\
    &= \sqrt{
        \frac{1-C\alpha}{1+C\alpha}
    } \Phi\left(
            -\sqrt{1+C\alpha}\,|v|
        \right) \left[ 1 - \frac{\Phi\left(
            -\sqrt{1+C\alpha}\,(L_T+|v|)
        \right)}{\Phi\left(
            -\sqrt{1+C\alpha}\,|v|
        \right)}\right] \\
    &\geq \left( 1- o_T(1) \right) 
    \sqrt{
        \frac{1-C\alpha}{1+C\alpha}
    }
    (1-C_2\alpha)
    \exp\{-C_2\alpha v^2\}
    \Phi(-|v|).
\end{aligned}
\]
Moreover, for all sufficiently small $\alpha$, there exists
$C_3<\infty$ such that
    $\sqrt{\frac{1-C\alpha}{1+C\alpha}}
    \ge
    1-C_3\alpha$.
Consequently,
\[
\begin{aligned}
    \frac{\dd\mathbb N_{T,\xi}}
         {\dd\mathbb M_{T,\xi}}(\omega)
    &\ge
    (1-o_T(1))
    (1-C_2\alpha)(1-C_3\alpha)
    \exp\left\{
        -C_2\alpha v^2
    \right\}
    \Phi\left(
        -|v|
    \right).
\end{aligned}
\]
Since
    $|v|
    \le
    \frac{L_T}{2}
    =
    \frac{D\alpha\sqrt T}{4}$,
there exists $C_4<\infty$ such that
    $C_2\alpha v^2
    \le
    C_4D^2\alpha^3T$.
Choose $\alpha$ sufficiently small such that
    $C_4D^2\alpha
    \le
    \varepsilon$.
Then
\[
    \exp\{
        -C_2\alpha v^2
    \}
    \ge
    \exp\left\{
        -\varepsilon\alpha^2T
    \right\}.
\]
Shrinking $\alpha$ further and then taking $T$ sufficiently large to obtain $(1-o_T(1))
    (1-C_2\alpha)(1-C_3\alpha)>(1-\varepsilon)$. This proves
\eqref{eq:local-overlap-transfer-lower}. 
\end{proof}


\begin{lemma}[Bathtub Principle under an Arbitrary Measure]
\label{lem:bathtub-general-measure-transfer-lower}
Let $\mathfrak m$ be a finite measure, let $V$ be a measurable real-valued
statistic, and let $w:\mathbb R\to[0,\infty)$ be even and nonincreasing in
$|v|$. Fix $0<a<B$, and let $\varphi$ be measurable with
$0\le\varphi\le1$. If
\[
    \int\varphi\,\dd\mathfrak m
    \le
    b
    \qquad\text{and}\qquad
    \mathfrak m(|V|\le a)
    \ge
    b,
\]
then
\begin{equation}
\label{eq:bathtub-general-measure-transfer-lower}
    \int_{\{|V|\le B\}}
        (1-\varphi)w(V)
    \,\dd\mathfrak m
    \ge
    \int_{\{a<|V|\le B\}}
        w(V)
    \,\dd\mathfrak m.
\end{equation}
\end{lemma}

\begin{proof}
Let
\[
    C:=\{|V|\le a\},
    \qquad
    D:=\{|V|\le B\}.
\]
Since $w(V)\le w(a)$ on $D\setminus C$,
\[
\begin{aligned}
    \int_D\varphi w(V)\,\dd\mathfrak m
    &\le
    \int_C\varphi w(V)\,\dd\mathfrak m
    +
    w(a)
    \int_{D\setminus C}\varphi\,\dd\mathfrak m.
\end{aligned}
\]
The budget and central-mass assumptions imply
\[
\begin{aligned}
    \int_{D\setminus C}\varphi\,\dd\mathfrak m
    &\le
    \int_{C^c}\varphi\,\dd\mathfrak m\\
    &\le
    \mathfrak m(C)
    -
    \int_C\varphi\,\dd\mathfrak m\\
    &=
    \int_C(1-\varphi)\,\dd\mathfrak m.
\end{aligned}
\]
Since $w(V)\ge w(a)$ on $C$,
\[
    w(a)
    \int_C(1-\varphi)\,\dd\mathfrak m
    \le
    \int_C
        (1-\varphi)w(V)
    \,\dd\mathfrak m.
\]
Therefore,
\[
\begin{aligned}
    \int_D\varphi w(V)\,\dd\mathfrak m
    &\le
    \int_C\varphi w(V)\,\dd\mathfrak m
    +
    \int_C(1-\varphi)w(V)\,\dd\mathfrak m\\
    &=
    \int_Cw(V)\,\dd\mathfrak m.
\end{aligned}
\]
Consequently,
\[
\begin{aligned}
    \int_D(1-\varphi)w(V)\,\dd\mathfrak m
    &=
    \int_Dw(V)\,\dd\mathfrak m
    -
    \int_D\varphi w(V)\,\dd\mathfrak m\\
    &\ge
    \int_Dw(V)\,\dd\mathfrak m
    -
    \int_Cw(V)\,\dd\mathfrak m\\
    &=
    \int_{D\setminus C}
        w(V)
    \,\dd\mathfrak m,
\end{aligned}
\]
which proves
\eqref{eq:bathtub-general-measure-transfer-lower}. 
\end{proof}

\begin{lemma}[Flat Local Sign Lower Bound]
\label{lem:flat-local-sign-transfer-lower}
Fix $\rho\in(0,1)$ and $\varepsilon\in(0,1/10)$, and choose
    $D>
    \frac{4}{(1-3\varepsilon)c_\rho}$.
Set
\[
    \bar\delta_\alpha:=D\alpha,
    \qquad
    L_T:=\frac{D\alpha\sqrt T}{2},
    \qquad
    a_T
    :=
    \frac{\alpha\sqrt T}
         {2(1-3\varepsilon)c_\rho}.
\]
For all sufficiently small $\alpha$, and then all sufficiently large $T$,
every oracle local-pair terminal rule $\psi$ satisfying
\[
    \mathcal A_T^{Q_{\rho,\bar\delta_\alpha}}(\psi)
    \le
    \alpha
\]
must satisfy
\begin{equation}
\label{eq:flat-sign-bound-transfer-lower}
\begin{aligned}
    \mathcal E_T^{Q_{\rho,\bar\delta_\alpha}}(\psi)
    &\ge
    \exp\{
        -\varepsilon\alpha^2T
        -o_T(\alpha^2T)
    \}
    \frac{
        2(1-\varepsilon)^2c_\rho
    }{
        \sqrt T
    }
    \int_{a_T}^{\infty}
        \Phi(-u)
    \,\dd u.
\end{aligned}
\end{equation}
\end{lemma}

\begin{proof}
Define a finite measure $\mathfrak m_T$ by
\[
\begin{aligned}
    \int F(\xi,\omega)\,
    \mathfrak m_T(\dd\xi,\dd\omega)
    &:=
    \int_{\mathcal C_\rho}
    \int_{\mathcal H_T^\xi}
        F(\xi,\omega)
    \,\mathbb M_{T,\xi}(\dd\omega)\,
    \Lambda_\rho(\dd\xi)
\end{aligned}
\]
for every nonnegative measurable function $F$. Let
\[
    V(\xi,\omega)
    :=
    V_{T,\xi}(\omega),
    \qquad
    \varphi(\xi,\omega)
    :=
    \psi_?^\xi(\omega).
\]

By the abstention representation
\eqref{eq:flat-sign-abstention-transfer-lower}, the constraint
$\mathcal A_T^{Q_{\rho,\bar\delta_\alpha}}(\psi)\le\alpha$ implies
\[
    \int
        \varphi
    \,\dd\mathfrak m_T
    \le
    \frac{\alpha\sqrt T}{2}.
\]
On the other hand, by
\eqref{eq:local-mass-transfer-lower} and
$\Lambda_\rho(\mathcal C_\rho)=c_\rho/2$,
\[
\begin{aligned}
    \mathfrak m_T(|V|\le a_T)
    &=
    \int_{\mathcal C_\rho}
        \mathbb M_{T,\xi}
        \left(
            |V_{T,\xi}|\le a_T
        \right)
    \,\Lambda_\rho(\dd\xi)\\
    &\ge
    2(1-\varepsilon)a_T
    \Lambda_\rho(\mathcal C_\rho)\\
    &=
    (1-\varepsilon)c_\rho a_T\\
    &=
    \frac{1-\varepsilon}{1-3\varepsilon}
    \frac{\alpha\sqrt T}{2}\\
    &>
    \frac{\alpha\sqrt T}{2}.
\end{aligned}
\]
Intuitively, this means the central set $\{|V|\le a_T\}$ has enough
$\mathfrak m_T$-mass to contain the entire abstention budget.

By the overlap lower bound
\eqref{eq:flat-sign-overlap-lower-transfer-lower},
\[
\begin{aligned}
    \mathcal E_T^{Q_{\rho,\bar\delta_\alpha}}(\psi)
    &\ge
    \frac{2}{\sqrt T}
    \int_{\mathcal C_\rho}
    \int_{\mathcal H_T^\xi}
        \bigl(1-\psi_?^\xi(\omega)\bigr)
    \,\mathbb N_{T,\xi}(\dd\omega)\,
    \Lambda_\rho(\dd\xi).
\end{aligned}
\]
Restricting the integral to
$\{|V_{T,\xi}|\le L_T/2\}$ and applying
\eqref{eq:local-overlap-transfer-lower} yields
\[
\begin{aligned}
    \mathcal E_T^{Q_{\rho,\bar\delta_\alpha}}(\psi)
    &\ge
    \frac{2}{\sqrt T}
    \int_{\mathcal C_\rho}
    \int_{\{|V_{T,\xi}|\le L_T/2\}}
        \bigl(1-\psi_?^\xi(\omega)\bigr)
    \,\mathbb N_{T,\xi}(\dd\omega)\,
    \Lambda_\rho(\dd\xi) \\
    &= \frac{2}{\sqrt T}
    \int_{\mathcal C_\rho}
    \int_{\{|V_{T,\xi}|\le L_T/2\}}
        \bigl(1-\psi_?^\xi(\omega)\bigr)
    \,\frac{\dd\mathbb N_{T,\xi}}{\dd\mathbb M_{T,\xi}}\mathbb M_{T,\xi}(\dd\omega)\,
    \Lambda_\rho(\dd\xi) \\
    &\ge
    \frac{2(1-\varepsilon)}{\sqrt T}
    \exp\{
        -\varepsilon\alpha^2T
        -o_T(\alpha^2T)
    \}
    \int_{\{|V|\le L_T/2\}}
        (1-\varphi)
        \Phi(-|V|)
    \,\dd\mathfrak m_T.
\end{aligned}
\]

Apply
Lemma~\ref{lem:bathtub-general-measure-transfer-lower} with
\[
    w(v):=\Phi(-|v|),
    \qquad
    a=a_T,
    \qquad
    B=L_T/2,
    \qquad
    b=\frac{\alpha\sqrt T}{2}.
\]
The budget and central-mass conditions were verified above. Therefore,
\[
\begin{aligned}
    \int_{\{|V|\le L_T/2\}}
        (1-\varphi)
        \Phi(-|V|)
    \,\dd\mathfrak m_T
    &\ge
    \int_{\{a_T<|V|\le L_T/2\}}
        \Phi(-|V|)
    \,\dd\mathfrak m_T\\
    &=
    \int_{\mathcal C_\rho}
    \int_{\{a_T<|V_{T,\xi}|\le L_T/2\}}
        \Phi\bigl(-|V_{T,\xi}(\omega)|\bigr)
    \,\mathbb M_{T,\xi}(\dd\omega)\,
    \Lambda_\rho(\dd\xi) \\
    &\stackrel{(a)}{\geq}
    \exp\{-o_T(\alpha^2T)\}
    2(1-\varepsilon)
    \Lambda_\rho(\mathcal C_\rho)
    \int_{a_T}^{L_T/2}
        \Phi(-u)
    \,\dd u\\
    &=
    \exp\{-o_T(\alpha^2T)\}
    (1-\varepsilon)c_\rho
    \int_{a_T}^{L_T/2}
        \Phi(-u)
    \,\dd u.
\end{aligned}
\]
where $(a)$ is by \eqref{eq:local-weight-transfer-lower} and
$\Lambda_\rho(\mathcal C_\rho)=c_\rho/2$.
Consequently,
\[
\begin{aligned}
    \mathcal E_T^{Q_{\rho,\bar\delta_\alpha}}(\psi)
    &\ge
    \exp\{
        -\varepsilon\alpha^2T
        -o_T(\alpha^2T)
    \}
    \frac{
        2(1-\varepsilon)^2c_\rho
    }{
        \sqrt T
    }
    \int_{a_T}^{L_T/2}
        \Phi(-u)
    \,\dd u.
\end{aligned}
\]

Finally, the choice of $D$ implies
    $\frac{L_T}{2}>2a_T$.
Since $a_T\to\infty$ for every fixed $\alpha>0$, Gaussian tail asymptotics
give
\[
    \int_{a_T}^{L_T/2}
        \Phi(-u)
    \,\dd u
    =
    (1-o_T(1))
    \int_{a_T}^{\infty}
        \Phi(-u)
    \,\dd u.
\]
Absorbing the factor $1-o_T(1)$ into
$\exp\{-o_T(\alpha^2T)\}$ proves
\eqref{eq:flat-sign-bound-transfer-lower}. 
\end{proof}

\subsubsection{Step 6. Conclusion}

Let \(\pi\) be any original adaptive policy satisfying
\[
    \mathcal A_T(\pi)\le\alpha.
\]
Fix \(\rho\in(0,1)\) and \(\varepsilon\in(0,1/10)\).  Choose $D>\frac{4}{(1-3\varepsilon)c_\rho}$
and set $\bar\delta_\alpha:=D\alpha$.

For all sufficiently small \(\alpha\), $\bar\delta_\alpha\le\delta_\rho$,
so \(Q_{\rho,\bar\delta_\alpha}\) is well-defined.

By \eqref{eq:risk-domination-transfer-lower},
\[
    \mathcal A_T^{Q_{\rho,\bar\delta_\alpha}}(\pi)
    \le
    \mathcal A_T(\pi)
    \le
    \alpha,
\]
and
\[
    \mathcal E_T(\pi)
    \ge
    \mathcal E_T^{Q_{\rho,\bar\delta_\alpha}}(\pi).
\]
After revealing the background oracle and applying
Lemma~\ref{lem:oracle-reduction-transfer-lower}, the original policy induces an
oracle local-pair sign rule whose abstention probability is at most \(\alpha\), and
whose sign-error probability is no larger than the original undetected-error probability.
Therefore Lemma~\ref{lem:flat-local-sign-transfer-lower} implies
\[
    \mathcal E_T(\pi)
    \ge
    \exp\{-\varepsilon\alpha^2T-o_T(\alpha^2T)\}
    \frac{2(1-\varepsilon)^2c_\rho}{\sqrt T}
    \int_{a_T}^{\infty}\Phi(-u)\,\dd u,
\]
where $a_T
    =
    \frac{\alpha\sqrt T}{2(1-3\varepsilon)c_\rho}$.

For fixed \(\alpha>0\), \(a_T\to\infty\).  The identity
\[
    \int_a^\infty \Phi(-u)\,\dd u
    =
    \phi(a)-a\Phi(-a)
\]
and the standard Gaussian tail expansion imply
\[
    \log
    \left(
        \int_a^\infty \Phi(-u)\,\dd u
    \right)
    =
    -\frac{a^2}{2}+O(\log a),
    \qquad a\to\infty.
\]
Thus, for every feasible policy \(\pi\),
\[
    \liminf_{T\to\infty}
    \frac{1}{\alpha^2T}
    \log \mathcal E_T(\pi)
    \ge
    -\varepsilon
    -
    \frac{1}{8(1-3\varepsilon)^2c_\rho^2}.
\]
The bound is uniform over all policies satisfying
\(\mathcal A_T(\pi)\le\alpha\).  Taking the infimum over feasible policies gives
\[
    \liminf_{T\to\infty}
    \frac{1}{\alpha^2T}
    \log \mathfrak E_T(\alpha)
    \ge
    -\varepsilon
    -
    \frac{1}{8(1-3\varepsilon)^2c_\rho^2}.
\]
Now take \(\alpha\downarrow0\).  Since \(\rho\) and \(\varepsilon\) are fixed,
\[
    \liminf_{\alpha\downarrow0}
    \liminf_{T\to\infty}
    \frac{1}{\alpha^2T}
    \log \mathfrak E_T(\alpha)
    \ge
    -\varepsilon
    -
    \frac{1}{8(1-3\varepsilon)^2c_\rho^2}.
\]
Using \eqref{eq:c-rho-lower-transfer-lower},
\[
    c_\rho\ge(1-\rho)^2\kappa,
\]
so
\[
    \liminf_{\alpha\downarrow0}
    \liminf_{T\to\infty}
    \frac{1}{\alpha^2T}
    \log \mathfrak E_T(\alpha)
    \ge
    -\varepsilon
    -
    \frac{
        1
    }{
        8(1-3\varepsilon)^2(1-\rho)^4\kappa^2
    }.
\]
Finally let \(\varepsilon\downarrow0\) and then \(\rho\downarrow0\).  This gives
\[
    \liminf_{\alpha\downarrow0}
    \liminf_{T\to\infty}
    \frac{1}{\alpha^2T}
    \log \mathfrak E_T(\alpha)
    \ge
    -\frac{1}{8\kappa^2}.
\]
This proves the lower-bound part of Theorem~\ref{thm:general-transfer}.

\subsection{Algorithm and Upper Bound}
\label{subsec:simple-transfer-upper}

Next we prove the matching upper bound.

Recall Assumption~\ref{ass:quadratic-kl} assumes that there exists \(c_{\rm KL}>0\) such that, for all \(\mu,z\in\mathcal I\),
$$
    D(\mu,z)
    \ge
    c_{\rm KL}\{s(\mu)-s(z)\}^2 .
$$
Take $\widetilde{c}:=\min\{c_{\rm KL},1/4\}$.
Then
\begin{equation}
\label{eq:ckl-under-simple-upper}
    D(\mu,z)
    \ge
    \widetilde{c}\,\{s(\mu)-s(z)\}^2,
    \qquad \mu,z\in\mathcal I .
\end{equation}

\subsubsection{The Algorithm}
We now define the information-coordinate version of PGWS,
\textsc{IC-PGWS}.  As in the Gaussian algorithm in the main paper, the policy
has two components.  During the sampling phase, it forms empirical centers in
the information coordinate and allocates samples according to information-gap
weights.  At the terminal time, it recommends the empirical information-coordinate leader unless the terminal certificate is below a calibrated
abstention threshold.

For each arm \(i\in[K]\), define
\[
    N_i(t)
    :=
    \sum_{r=1}^{t}\mathbbm{1}\{A_r=i\},
    \qquad
    S_i(t)
    :=
    \sum_{r\le t:A_r=i}t(X_r).
\]
The sampling rule uses an empirical information-coordinate center.  If
\(N_i(t)>0\), define
\[
    \widehat{t}_i(t)
    :=
    \frac{S_i(t)}{N_i(t)} .
\]
If \(N_i(t)>0\) and
\(\widehat{t}_i(t)\in A'(\Theta)\), set
\begin{equation}
\label{eq:icpgws-empirical-center}
    \widehat\theta_i(t)
    :=
    (A')^{-1}\!\left(\widehat{t}_i(t)\right),
    \qquad
    \widehat\eta_i(t)
    :=
    \beta^{-1}\!\left(\widehat\theta_i(t)\right),
    \qquad
    \widehat\mu_i(t)
    :=
    s^{-1}\!\left(\widehat\eta_i(t)\right).
\end{equation}
If \(N_i(t)=0\), or if
\(\widehat{t}_i(t)\notin A'(\Theta)\), use the deterministic fallback
\begin{equation}
\label{eq:icpgws-empirical-center-fallback}
    \widehat\eta_i(t):=\eta_\circ:=s(\mu_\circ),
    \qquad
    \widehat\mu_i(t):=\mu_\circ .
\end{equation}
The fallback is only a finite-time safeguard.  For example, in Bernoulli rewards,
the empirical natural statistic may equal \(0\) or \(1\) after only failures
or only successes.  Since the true parameter lies in the interior and forced
exploration gives \(N_i(t)\to\infty\), the strong law implies
\(\widehat{t}_i(t)\to A'(\theta(\mu_i))\in A'(\Theta)\), so the fallback
is used only finitely many times almost surely for each arm.

At time \(t\), define the empirical information-coordinate leader
\begin{equation}
\label{eq:icpgws-leader}
    \widehat b_t
    :=
    \operatorname*{arg\,max}_{i\in[K]}\widehat\eta_i(t),
\end{equation}
with deterministic tie-breaking.  For \(j\neq\widehat b_t\), define the empirical
information-coordinate gap
\[
    \widehat\Delta_{j,s}(t)
    :=
    \widehat\eta_{\widehat b_t}(t)-\widehat\eta_j(t),
\]
and define the empirical top-two information gap
\[
    \widehat\Gamma_{s,t}
    :=
    \min_{j\neq\widehat b_t}\widehat\Delta_{j,s}(t).
\]
When \(\widehat\Gamma_{s,t}>0\), set
\begin{equation}
\label{eq:icpgws-weights}
    w_{\widehat b_t}(t):=1,
    \qquad
    w_j(t)
    :=
    \min\left\{
        1,\,
        \frac{\widehat\Gamma_{s,t}^{\,2}}
        {4\widetilde c\,\widehat\Delta_{j,s}(t)^2}
    \right\},
    \qquad j\neq\widehat b_t .
\end{equation}
The sampling probabilities are
\begin{equation}
\label{eq:icpgws-sampling-prob}
    p_i(t)
    :=
    \begin{cases}
    \displaystyle
    \frac{w_i(t)}{\sum_{\ell=1}^K w_\ell(t)},
    & \widehat\Gamma_{s,t}>0,\\[10pt]
    \displaystyle
    \frac1K,
    & \widehat\Gamma_{s,t}=0 .
    \end{cases}
\end{equation}
The constant \(\widetilde c\le 1/4\) is the one in
\eqref{eq:ckl-under-simple-upper}.  The cap at \(1/4\) ensures that the closest empirical challenger receives the same sampling weight ($1$) as the leader.

At the terminal time, set
\begin{equation}
\label{eq:icpgws-terminal-leader}
    \widehat b_T
    :=
    \operatorname*{arg\,max}_{i\in[K]}\widehat\eta_i(T),
\end{equation}
again with deterministic tie-breaking. 
Given the terminal history, for ease of notation let
\[
    \varPi_T(\cdot):=\mathbb P(\cdot\mid\mathcal F_T)
\]
denote the posterior distribution induced by the prior and the observed data.
The \emph{terminal certificate} is
\begin{equation}
\label{eq:icpgws-terminal-statistic}
    R_T
    :=
    -\log
    \max_{j\neq\widehat b_T}
    \varPi_T(\mu_j\ge \mu_{\widehat b_T})
    =
    -\log
    \max_{j\neq\widehat b_T}
    \varPi_T(\eta_j\ge \eta_{\widehat b_T}),
\end{equation}
where the equality follows because \(s\) is strictly increasing. Thus \(R_T\) is large when the final leader is
well separated from all challengers under the posterior, and small when at least
one challenger remains plausible.

To calibrate abstention, let \(r_{T,\alpha}\) be the lower \(\alpha\)-quantile of
\(R_T\) under the joint Bayesian law induced by the sampling rule:
\begin{equation}
\label{eq:icpgws-terminal-threshold}
    r_{T,\alpha}
    :=
    \inf\left\{
        r\ge0:
        \mathbb P(R_T\le r)\ge\alpha
    \right\}.
\end{equation}
All probabilities in \eqref{eq:icpgws-terminal-threshold} are taken before the
final terminal randomization.  If the distribution of \(R_T\) has an atom at the
threshold, define
\begin{equation}
\label{eq:icpgws-boundary-randomization}
    \theta_{T,\alpha}
    :=
    \begin{cases}
    \displaystyle
    \frac{
        \alpha-\mathbb P(R_T<r_{T,\alpha})
    }{
        \mathbb P(R_T=r_{T,\alpha})
    },
    & \mathbb P(R_T=r_{T,\alpha})>0,\\[12pt]
    0,
    & \mathbb P(R_T=r_{T,\alpha})=0 .
    \end{cases}
\end{equation}
The terminal decision is
\begin{equation}
\label{eq:icpgws-terminal-rule-simple-upper}
    \widehat a_T
    =
    \begin{cases}
        ?, & R_T<r_{T,\alpha},\\[3pt]
        ?, & R_T=r_{T,\alpha}
             \text{ with probability } \theta_{T,\alpha},\\[3pt]
        \widehat b_T, & \text{otherwise}.
    \end{cases}
\end{equation}

\begin{algorithm}[t]
\caption{Information-Coordinate Posterior Gap Weighted Sampling with Abstention
\textsc{IC-PGWS}$(\alpha)$}
\label{alg:icpgws-simple-upper}
\begin{algorithmic}[1]
\Require Budget \(T\), abstention level \(\alpha\), functions
\(s\), \(\beta\), \(A\), natural statistic \(t\), prior densities
\(\{\bar f_i\}_{i=1}^K\), and a constant
\(\widetilde c\in(0,1/4]\) satisfying \eqref{eq:ckl-under-simple-upper}.
\State Initialize \(N_i(0)=0\) and \(S_i(0)=0\) for all \(i\in[K]\).
\For{\(t=0,1,\ldots,T-1\)}
    \If{there exists \(i\in[K]\) such that \(N_i(t)<\sqrt{t+1}\)}
        \State Sample one such arm, using deterministic tie-breaking.
    \Else
        \State Compute the empirical centers
        \(\{\widehat\eta_i(t),\widehat\mu_i(t)\}_{i=1}^K\) using
        \eqref{eq:icpgws-empirical-center}--\eqref{eq:icpgws-empirical-center-fallback}.
        \State Compute the leader \(\widehat b_t\), the empirical gaps
        \(\{\widehat\Delta_{j,s}(t)\}_{j\neq\widehat b_t}\), and
        \(\widehat\Gamma_{s,t}\) using \eqref{eq:icpgws-leader}.
        \State Compute the sampling probabilities
        \(\{p_i(t)\}_{i=1}^K\) using
        \eqref{eq:icpgws-weights}--\eqref{eq:icpgws-sampling-prob}.
        \State Sample arm \(i\) with probability \(p_i(t)\).
    \EndIf
    \State Observe \(X_{t+1}\), and update \(N_i(t+1)\) and \(S_i(t+1)\).
\EndFor
\State Compute \(\widehat b_T\) using \eqref{eq:icpgws-terminal-leader}.
\State Compute the terminal certificate \(R_T\) in
\eqref{eq:icpgws-terminal-statistic}.
\State Output \(\widehat a_T\) according to the abstention rule
\eqref{eq:icpgws-terminal-rule-simple-upper}.
\end{algorithmic}
\end{algorithm}


\subsubsection{Upper Bound Theorem}

\begin{theorem}[Upper Bound for \textsc{IC-PGWS}]
\label{thm:simple-transfer-upper}
Under Assumptions~\ref{ass:regular-expfam},
\ref{ass:regular-info-prior}, and the global quadratic KL assumption~\ref{ass:quadratic-kl}, $\textsc{IC-PGWS}(\alpha)$ with terminal rule
\eqref{eq:icpgws-terminal-rule-simple-upper} satisfies, for every \(T\) and every
\(\alpha\in(0,1)\),
\[
    \mathcal A_T(\textsc{IC-PGWS}(\alpha))=\alpha .
\]
Moreover,
\[
    \limsup_{\alpha\downarrow0}
    \limsup_{T\to\infty}
    \frac{1}{\alpha^2T}
    \log \mathcal E_T(\textsc{IC-PGWS}(\alpha))
    \le
    -\frac{1}{8\kappa^2}.
\]
\end{theorem}

We prove the theorem in several steps.  Throughout this proof, all probabilities
are under the joint Bayesian law induced by
Algorithm~\ref{alg:icpgws-simple-upper}, unless otherwise stated.  Recall that
\[
    \eta_i=s(\mu_i),
    \qquad
    \beta(\eta)=\theta(s^{-1}(\eta)),
    \qquad
    \mathcal H=s(\mathcal I).
\]

\subsubsection{Error Guarantee}

\begin{lemma}[Error Guarantee]
\label{lem:posterior-certification-simple-upper}
For every \(T\) and every \(\alpha\in(0,1)\),
\[
    \mathcal E_T(\textsc{IC-PGWS}(\alpha))
    \le
    (K-1)e^{-r_{T,\alpha}}.
\]
\end{lemma}

\begin{proof}
Condition on the terminal history \(\mathcal F_T\).  If the algorithm recommends
\(\widehat b_T\), then
\[
    \mathbb P(a^\star\neq \widehat b_T\mid\mathcal F_T)
    \le
    \sum_{j\neq \widehat b_T}
    \varPi_T(\mu_j\ge \mu_{\widehat b_T}).
\]
By the definition of \(R_T\),
\[
    \max_{j\neq\widehat b_T}
    \varPi_T(\mu_j\ge \mu_{\widehat b_T})
    =
    e^{-R_T}.
\]
Therefore
\[
    \mathbb P(a^\star\neq \widehat b_T\mid\mathcal F_T)
    \le
    (K-1)e^{-R_T}.
\]
On the event that the algorithm does not abstain, the terminal quantile rule
implies \(R_T\ge r_{T,\alpha}\).  Hence
\[
\begin{aligned}
    \mathcal E_T(\textsc{IC-PGWS}(\alpha))
    =
    \mathbb E\!\left[
        \mathbbm{1}\{\widehat a_T\neq ?\}
        \mathbb P(a^\star\neq \widehat b_T\mid\mathcal F_T)
    \right]  
    \le
    (K-1)e^{-r_{T,\alpha}}. 
\end{aligned}
\]

\end{proof}

\subsubsection{Consistency and Allocation Convergence.}

We next identify the almost-sure limit of \(R_T/T\). We first introduce the
limiting quantities associated with a realized parameter vector. Let
    $b^\star
    :=
    \operatorname*{arg\,max}_{i\in[K]}\eta_i
    =
    \operatorname*{arg\,max}_{i\in[K]}\mu_i$
be the true best arm. For each \(j\neq b^\star\), define the information coordinate
gap
    $\Delta_{j,s}
    :=
    \eta_{b^\star}-\eta_j$,
and define the true top two information gap
    $\Gamma_s
    :=
    \min_{j\neq b^\star}\Delta_{j,s}
    =
    \eta_{(1)}-\eta_{(2)}$,
where \(\eta_{(1)}>\eta_{(2)}>\cdots>\eta_{(K)}\) are the ordered
information-coordinate qualities.

Define the \emph{limiting unnormalized sampling weights} by
\begin{equation}
\label{eq:simple-limiting-weights}
    w_{b^\star}^\dagger(\eta):=1,
    \qquad
    w_j^\dagger(\eta)
    :=
    \min\left\{
        1,\,
        \frac{\Gamma_s^2}
        {4\widetilde c\,\Delta_{j,s}^2}
    \right\},
    \qquad j\neq b^\star.
\end{equation}
Let
\[
    W^\dagger(\eta)
    :=
    \sum_{\ell=1}^K w_\ell^\dagger(\eta),
    \qquad
    p_i^\dagger(\eta)
    :=
    \frac{w_i^\dagger(\eta)}{W^\dagger(\eta)}.
\]
Because \(\widetilde c\le 1/4\), every second-best arm has limiting weight one.
Consequently,
\[
    2\le W^\dagger(\eta)\le K,
    \qquad
    p_i^\dagger(\eta)>0
    \quad\text{for every }i\in[K].
\]

For each challenger \(j\neq b^\star\), define its pairwise error exponent
\[
    C_j^\dagger(\mu)
    :=
    \inf_{z\in[\mu_j,\mu_{b^\star}]}
    \left\{
        p_{b^\star}^\dagger(\eta)D(\mu_{b^\star},z)
        +
        p_j^\dagger(\eta)D(\mu_j,z)
    \right\}.
\]
The worst-case pairwise error exponent is
\begin{equation}
\label{eq:Cdagger-simple-upper}
    C^\dagger(\mu)
    :=
    \min_{j\neq b^\star}C_j^\dagger(\mu).
\end{equation}
Equivalently,
\begin{equation}
\label{eq:Cdagger-unnormalized-simple-upper}
    C^\dagger(\mu)
    =
    \frac{1}{W^\dagger(\eta)}
    \min_{j\neq b^\star}
    \inf_{z\in[\mu_j,\mu_{b^\star}]}
    \left\{
        D(\mu_{b^\star},z)
        +
        w_j^\dagger(\eta)D(\mu_j,z)
    \right\}.
\end{equation}

We now show that the empirical centers converge to the true qualities, the
empirical sampling proportions converge to \(p^\dagger(\eta)\), and the
normalized terminal certificate converges to \(C^\dagger(\mu)\).


\begin{lemma}[Information Coordinate Center Consistency]
\label{lem:center-consistency-simple-upper}
Under Algorithm~\ref{alg:icpgws-simple-upper}, for every arm \(i\),
\[
    \widehat \eta_i(t)\to \eta_i
    \qquad\text{and}\qquad
    \widehat \mu_i(t)\to \mu_i
    \qquad\text{a.s.}
\]
\end{lemma}

\begin{proof}
The forced-exploration rule implies
\[
    N_i(t)\to\infty
    \qquad\text{a.s. for every }i.
\]
Conditional on \(\mu_i\), the observations from arm \(i\) are i.i.d. from
\(P_{\mu_i}\).  Since the log-partition function is finite on an open natural-parameter
space, the natural statistic has finite mean under \(P_{\mu_i}\), and
\[
    \mathbb E_{\mu_i}[t(X)]
    =
    A'(\theta(\mu_i)).
\]
Therefore, by the strong law of large numbers,
\[
    \widehat{t}_i(t)
    =
    \frac{1}{N_i(t)}
    \sum_{r\le t:A_r=i}t(X_r)
    \to
    A'(\theta(\mu_i))
    \qquad\text{a.s.}
\]
Because \(\theta(\mu_i)\in\Theta\) and \(\Theta\) is open,
\(A'(\theta(\mu_i))\in A'(\Theta)\).  Hence $\widehat{t}_i(t)\in A'(\Theta)$ eventually almost surely.  On this eventual event,
\[
    \widehat\theta_i(t)
    =
    (A')^{-1}\!\left(\widehat{t}_i(t)\right)
    \to
    \theta(\mu_i),
\]
since \(A''>0\) implies that \(A'\) is strictly increasing and continuous.
Thus
\[
    \widehat \eta_i(t)
    =
    \beta^{-1}(\widehat\theta_i(t))
    \to
    \beta^{-1}(\theta(\mu_i))
    =
    s(\mu_i)
    =
    \eta_i.
\]
Continuity of \(s^{-1}\) gives
\[
    \widehat \mu_i(t)=s^{-1}(\widehat \eta_i(t))\to \mu_i. 
\]

\end{proof}

\begin{lemma}[Sampling Proportions]
\label{lem:sampling-proportions-simple-upper}
Under Algorithm~\ref{alg:icpgws-simple-upper},
\[
    \frac{N_i(T)}{T}
    \to
    p_i^\dagger(\eta)
    \qquad\text{a.s. for every }i\in[K].
\]
\end{lemma}

\begin{proof}
By Lemma~\ref{lem:center-consistency-simple-upper},
\[
    \widehat \eta_i(t)\to \eta_i
    \qquad\text{for every }i.
\]
Since the prior densities are continuous, ties have probability zero.  Thus the
best arm and the second-best arm are unique almost surely.  On this
probability-one event, $\widehat b_t=b^\star$
eventually, and
\[
    \widehat\Delta_{j,s}(t)\to\Delta_{j,s},
    \qquad
    \widehat\Gamma_{s,t}\to\Gamma_s.
\]
Therefore
\[
    w_i(t)\to w_i^\dagger(\eta),
    \qquad
    p_i(t)\to p_i^\dagger(\eta)
    \qquad\text{a.s. for every }i.
\]

Now fix an arm $i$. We take four steps:

\textbf{Step 1: Forced exploration is negligible.}
Let $F_i(T)$ be the number of forced-exploration pulls of arm $i$ up to time $T$.
If the $m$-th forced pull of arm $i$ occurs before time $T$, then immediately before that pull,
$N_i(t)\ge m-1$.
But a forced pull of arm $i$ can occur only if
$N_i(t)<\sqrt{t+1}\le \sqrt{T+1}$.
Hence
$m-1<\sqrt{T+1}$,
and therefore
$F_i(T)\le \lceil \sqrt{T+1}\rceil+1$.
Consequently,
\[
\frac{F_i(T)}{T}\to0.
\]
Let $F(T)$ be the total number of forced-exploration rounds up to time $T$. Since
$F(T)=\sum_{\ell=1}^K F_\ell(T)$,
we also have
\[
\frac{F(T)}{T}\le
\frac{K(\lceil\sqrt{T+1}\rceil+1)}{T}
\to0.
\]

\textbf{Step 2: Martingale concentration on non-forced rounds.}
Let $\mathcal F_t$ be the history before round $t+1$, and define
$B_t
:=
\mathbbm{1}\{\text{round }t+1\text{ is non-forced}\}$.
The event $\{B_t=1\}$ is $\mathcal F_t$-measurable. On non-forced rounds, the algorithm samples arm $i$ with conditional probability $p_i(t)$.

Define
$D_{i,t+1}
:=
B_t\Bigl(\mathbbm{1}\{A_{t+1}=i\}-p_i(t)\Bigr)$.
Then
$\mathbb E[D_{i,t+1}\mid \mathcal F_t]=0$,
so $(D_{i,t+1})_{t\ge0}$ is a martingale difference sequence. Also,
$|D_{i,t+1}|\le1$.

Therefore, by Azuma's inequality, for every $\varepsilon>0$,
\[
\mathbb P\left(
\left|
\frac1T\sum_{t=0}^{T-1}D_{i,t+1}
\right|
\ge \varepsilon
\right)
\le
2\exp\left(-\frac{\varepsilon^2T}{2}\right).
\]
The right-hand side is summable in $T$. Hence, by the Borel--Cantelli lemma,
\[
\frac1T\sum_{t=0}^{T-1}D_{i,t+1}\to0
\qquad\text{almost surely}.
\]

\textbf{Step 3: Decomposition of the empirical pull count.}
The total number of pulls of arm $i$ can be decomposed as
\[
N_i(T)
=
F_i(T)
+
\sum_{t=0}^{T-1}
B_t\mathbbm{1}\{A_{t+1}=i\}.
\]
Using the definition of $D_{i,t+1}$, we  have
$B_t\mathbbm{1}\{A_{t+1}=i\}
=
D_{i,t+1}+B_t p_i(t)$.
Thus
\[
\frac{N_i(T)}{T}
=
\frac{F_i(T)}{T}
+
\frac1T\sum_{t=0}^{T-1}D_{i,t+1}
+
\frac1T\sum_{t=0}^{T-1}B_t p_i(t).
\]
By Steps 1 and 2,
\[
\frac{N_i(T)}{T}
-
\frac1T\sum_{t=0}^{T-1}B_t p_i(t)
\to0
\qquad\text{almost surely}.
\]

\textbf{Step 4: Ces\`aro convergence of the target probabilities.}
Since
\[
p_i(t)\to p_i^\dagger(\eta)
\qquad\text{almost surely},
\]
the sequence $\{p_i(t)\}$ is Ces\`aro convergent, i.e., 
\[
\frac1T\sum_{t=0}^{T-1}p_i(t)
\to
p_i^\dagger(\eta)
\qquad\text{almost surely}.
\]

Now write
\[
\frac1T\sum_{t=0}^{T-1}B_t p_i(t)
=
\frac1T\sum_{t=0}^{T-1}p_i(t)
-
\frac1T\sum_{t=0}^{T-1}(1-B_t)p_i(t).
\]
Since $0\le p_i(t)\le1$,
\[
0\le
\frac1T\sum_{t=0}^{T-1}(1-B_t)p_i(t)
\le
\frac1T\sum_{t=0}^{T-1}(1-B_t)
=
\frac{F(T)}{T}
\to0.
\]
Hence
\[
\frac1T\sum_{t=0}^{T-1}B_t p_i(t)
\to
p_i^\dagger(\eta)
\qquad\text{almost surely}.
\]
Combining this with Step 3 gives the desired result.

\end{proof}

\subsubsection{Posterior Ordering Large Deviations}

Let \(X_{i,1},\ldots,X_{i,n}\) denote the first \(n\) rewards observed from
arm \(i\), and let
    $\mathcal G_{i,n}:=\sigma(X_{i,1},\ldots,X_{i,n})$.
Define the \emph{one-arm posterior} \(\varPi_{i,n}\) as the conditional distribution
of \(\eta_i=s(\mu_i)\) given \(\mathcal G_{i,n}\):
\[
    \varPi_{i,n}(B)
    :=
    \mathbb P(\eta_i\in B\mid\mathcal G_{i,n}),
    \qquad B\in\mathcal B(\mathcal H).
\]
Intuitively, the next lemma proves that $$
\varPi_{i,n} (\mu_i'\approx z) = \exp\{ -nD(\mu_i,z) + o(n) \}
$$

\begin{lemma}[One-arm Posterior Estimates for Ordering]
\label{lem:one-arm-laplace-simple-upper}
Fix an arm \(i\), and work on the event (that happens almost surely)
\[
    \frac1n\sum_{r=1}^nt(X_{i,r})
    \to
    A'(\theta(\mu_i)).
\]
Let \(\varPi_{i,n}\) be the one-arm posterior defined above.  Define
\[
    I_i(\eta):=D(\mu_i,s^{-1}(\eta)).
\]
Then the following hold.

First, for every compact interval \(K\Subset\mathcal H\),
\[
    \limsup_{n\to\infty}
    \frac1n\log\varPi_{i,n}(K)
    \le
    -\inf_{\eta\in K} I_i(\eta).
\]
Second, for every nonempty open interval \(G\subseteq\mathcal H\),
\[
    \liminf_{n\to\infty}
    \frac1n\log\varPi_{i,n}(G)
    \ge
    -\inf_{\eta\in G} I_i(\eta).
\]
Third, for every \(a,b\in\mathcal H\) with \(a<\eta_i<b\),
\[
    \limsup_{n\to\infty}
    \frac1n\log\varPi_{i,n}(\eta\le a)
    \le
    -I_i(a),
\]
and
\[
    \limsup_{n\to\infty}
    \frac1n\log\varPi_{i,n}(\eta\ge b)
    \le
    -I_i(b).
\]
\end{lemma}

\begin{proof}
Set
\[
    \eta_i=s(\mu_i),
    \qquad
    \theta_i:=\theta(\mu_i)=\beta(\eta_i),
    \qquad
    m_i:=A'(\theta_i),
\]
and write
\[
    \bar m_n:=\frac1n\sum_{r=1}^nt(X_{i,r}).
\]

For \(\eta\in\mathcal H\), write
\[
    p_\eta(w)
    :=
    \frac{\dd P_{s^{-1}(\eta)}}{\dd\lambda}(w)
    =
    \exp\{\beta(\eta)t(w)-A(\beta(\eta))\}h(w).
\]
Consider a realized full adaptive history $h_t=(a_1,x_1,\ldots,a_t,x_t)$,
For a candidate parameter vector
$\boldsymbol\eta=(\eta_1,\ldots,\eta_K)\in\mathcal H^K$,
the density of this realized history under the policy \(\pi\), with respect to
counting measure on actions and \(\lambda^{\otimes t}\) on rewards, is
\[
    L_t^\pi(\boldsymbol\eta;h_t)
    =
    \prod_{r=1}^t
    \pi_r(a_r\mid h_{r-1})
    \prod_{r=1}^t
    p_{\eta_{a_r}}(x_r).
\]
Here \(\pi_r(a_r\mid h_{r-1})\) is the sampling probability assigned by the
policy to the realized action \(a_r\) after the realized past \(h_{r-1}\).

Using the product prior in information coordinates, the posterior density of
\(\boldsymbol\eta\) is therefore proportional to
\[
    \prod_{\ell=1}^K
    \bar f_\ell(\eta_\ell)
    \prod_{r:a_r=\ell}
    p_{\eta_\ell}(x_r).
\]
Thus the posterior factorizes arm by arm.  In particular, the marginal posterior
of arm \(i\) depends only on the rewards observed from arm \(i\).  If these
rewards are denoted
$X_{i,1},\ldots,X_{i,n}$,
then, for every measurable set \(B\subseteq\mathcal H\),
\[
    \varPi_{i,n}(B)
    =
    \frac{
        \int_B
        \bar f_i(\eta)
        \prod_{r=1}^n p_\eta(X_{i,r})
        \,\dd\eta
    }{
        \int_{\mathcal H}
        \bar f_i(\eta)
        \prod_{r=1}^n p_\eta(X_{i,r})
        \,\dd\eta
    }.
\]
Substituting the exponential-family density gives
\[
    \prod_{r=1}^n p_\eta(X_{i,r})
    =
    \exp\left\{
        \beta(\eta)\sum_{r=1}^nt(X_{i,r})
        -
        nA(\beta(\eta))
    \right\}
    \prod_{r=1}^n h(X_{i,r}).
\]
\(\prod_{r=1}^n h(X_{i,r})\) does not depend on \(\eta\), so
it also cancels between numerator and denominator.
Bayes' rule gives
\[
    \varPi_{i,n}(B)
    =
    \frac{
        \int_B
        \bar f_i(\eta)
        \exp\{n[\beta(\eta)\bar m_n-A(\beta(\eta))]\}
        \,\dd\eta
    }{
        \int_{\mathcal H}
        \bar f_i(\eta)
        \exp\{n[\beta(\eta)\bar m_n-A(\beta(\eta))]\}
        \,\dd\eta
    }.
\]
Multiplying both numerator and denominator by
\[
    \exp\{-n[\beta(\eta_i)\bar m_n-A(\beta(\eta_i))]\}
\]
allows us to write
\[
    \varPi_{i,n}(B)
    =
    \frac{
        \int_B \bar f_i(\eta)e^{nL_n(\eta)}\,\dd\eta
    }{
        Z_n
    },
\]
where
\[
    L_n(\eta)
    :=
    (\beta(\eta)-\beta(\eta_i))\bar m_n
    -
    A(\beta(\eta))
    +
    A(\beta(\eta_i))
\]
and
\[
    Z_n
    :=
    \int_{\mathcal H}\bar f_i(\eta)e^{nL_n(\eta)}\,\dd\eta.
\]
Note \(L_n(\eta_i)=0\) and
\[
    L_n(\eta)\to (\beta(\eta)-\beta(\eta_i))A'(\beta(\eta_i)) - A(\beta(\eta))+A(\beta(\eta_i))
    = -D(\mu_i,s^{-1}(\eta))=-I_i(\eta)
\]
on compact subsets of \(\mathcal H\).

We first prove that
\begin{equation}
\label{eq:one-arm-denominator-rate-simple-upper}
    \frac1n\log Z_n\to0.
\end{equation}

For the lower bound on \(Z_n\), fix \(\tau>0\).  Since $\bar m_n\to m_i$, $L_n(\eta)\to -I_i(\eta)$ uniformly on compact neighborhoods of \(\eta_i\).  

Since \(I_i(\eta_i)=0\), choose a compact interval \(J_\tau\Subset\mathcal H\) containing \(\eta_i\) such that, for
all sufficiently large \(n\), $\inf_{\eta\in J_\tau}L_n(\eta)\ge -\tau$.

By continuity and strict positivity of \(\bar f_i\),
\[
    \int_{J_\tau}\bar f_i(\eta)\,\dd\eta>0.
\]
Therefore
\[
    Z_n
    \ge
    \int_{J_\tau}\bar f_i(\eta)e^{nL_n(\eta)}\,\dd\eta
    \ge
    e^{-n\tau}\int_{J_\tau}\bar f_i(\eta)\,\dd\eta.
\]
Thus
\[
    \liminf_{n\to\infty}\frac1n\log Z_n\ge -\tau.
\]
Letting \(\tau\downarrow0\) gives
\[
    \liminf_{n\to\infty}\frac1n\log Z_n\ge0.
\]

For the upper bound on \(Z_n\), for all large \(n\) we have
\(\bar m_n\in A'(\Theta)\).  Define the pseudo-MLE
\[
    \theta_n:=(A')^{-1}(\bar m_n),
    \qquad
    \eta_n:=\beta^{-1}(\theta_n).
\]
Then $\theta_n\to\theta_i$ and $\eta_n\to \eta_i$.
Moreover, using \(\bar m_n=A'(\theta_n)\),
\[
\begin{aligned}
    L_n(\eta_n)-L_n(\eta)
    &=
    (\theta_n-\beta(\eta))A'(\theta_n)
    -
    A(\theta_n)
    +
    A(\beta(\eta))                                                \\
    &=
    D(s^{-1}(\eta_n),s^{-1}(\eta)).
\end{aligned}
\]
Therefore, by assumption \eqref{eq:ckl-under-simple-upper},
\[
    L_n(\eta)
    \le
    L_n(\eta_n)-\widetilde{c}\,(\eta-\eta_n)^2.
\]
Since
\[
    L_n(\eta_n)
    =
    D(s^{-1}(\eta_n),\mu_i)
    \to0,
\]
boundedness of \(\bar f_i\) gives
\[
\begin{aligned}
    Z_n
    &\le
    \|\bar f_i\|_\infty
    e^{nL_n(\eta_n)}
    \int_{\mathcal H}e^{-n\widetilde{c} (\eta-\eta_n)^2}\,\dd\eta        \\
    &\le
    \|\bar f_i\|_\infty
    e^{nL_n(\eta_n)}
    \sqrt{\frac{\pi}{n\widetilde{c}}}.
\end{aligned}
\]
Hence
\[
    \limsup_{n\to\infty}\frac1n\log Z_n\le0.
\]
This proves \eqref{eq:one-arm-denominator-rate-simple-upper}.

Now fix a compact interval \(K\Subset\mathcal H\).  Since \(\bar m_n\to m_i\),
\[
    L_n(\eta)\to -D(\mu_i,s^{-1}(\eta))=-I_i(\eta)
\]
uniformly on \(K\).  Hence
\[
\begin{aligned}
    \limsup_{n\to\infty}
    \frac1n\log
    \int_K \bar f_i(\eta)e^{nL_n(\eta)}\,\dd\eta
    &\le
    -\inf_{\eta\in K}I_i(\eta).
\end{aligned}
\]
Combining this numerator bound with
\eqref{eq:one-arm-denominator-rate-simple-upper} gives the compact-interval
upper bound.

Next let \(G\subseteq\mathcal H\) be a nonempty open interval.  Choose
\(\eta_0\in G\) such that
\[
    I_i(\eta_0)\le \inf_{\eta\in G}I_i(\eta)+\tau.
\]
Choose a compact interval \(J\Subset G\) containing \(\eta_0\) and small enough
that
\[
    \sup_{\eta\in J}I_i(\eta)\le I_i(\eta_0)+\tau.
\]
By uniform convergence of \(L_n\) on \(J\) and strict positivity of
\(\bar f_i\) on \(J\),
\[
\begin{aligned}
    \liminf_{n\to\infty}
    \frac1n\log
    \int_G \bar f_i(\eta)e^{nL_n(\eta)}\,\dd\eta
    &\ge
    \liminf_{n\to\infty}
    \frac1n\log
    \int_J \bar f_i(\eta)e^{nL_n(\eta)}\,\dd\eta                  \\
    &\ge
    -\sup_{\eta\in J}I_i(\eta)
    \ge
    -\inf_{\eta\in G}I_i(\eta)-2\tau.
\end{aligned}
\]
Using again the denominator rate
\eqref{eq:one-arm-denominator-rate-simple-upper}, and then letting
\(\tau\downarrow0\), gives the open-interval lower bound.

It remains to prove the one-sided estimates.  Consider \(a<\eta_i\).  Since
\(\eta_n\to \eta_i\), eventually \(a<\eta_n\).  The function $\theta\mapsto \theta\bar m_n-A(\theta)$
is strictly concave and maximized at \(\theta_n\).  Since \(\beta\) is strictly
increasing, \(L_n(\eta)\) is increasing on \(\{\eta<\eta_n\}\).  Therefore, for all large
\(n\),
\[
    \sup_{\eta\le a}L_n(\eta)=L_n(a).
\]
Thus
\[
    \varPi_{i,n}(\eta\le a)
    \le
    \frac{e^{nL_n(a)}}{Z_n},
\]
because \(\int_{\{\eta\le a\}}\bar f_i(\eta)\,\dd\eta\le1\).
Since $L_n(a)\to -I_i(a)$
and \(n^{-1}\log Z_n\to0\), we obtain
\[
    \limsup_{n\to\infty}
    \frac1n\log\varPi_{i,n}(\eta\le a)
    \le
    -I_i(a).
\]
The proof for \(b>\eta_i\) is identical: eventually \(b>\eta_n\), \(L_n\) is
decreasing on \(\{\eta>\eta_n\}\), and hence
\[
    \sup_{\eta\ge b}L_n(\eta)=L_n(b).
\]
This gives the upper-tail estimate.

\end{proof}

The following lemma shows that, if arms \(i\) and \(j\)
receive asymptotic sampling proportions \(q_i\) and \(q_j\), then the posterior
probability of incorrectly ordering them decays with exponent
\[
    \inf_{z\in[\mu_j,\mu_i]}
    \{q_iD(\mu_i,z)+q_jD(\mu_j,z)\}.
\]

\begin{lemma}[Posterior Ordering Large Deviations]
\label{lem:posterior-ordering-simple-upper}
Fix two arms \(i\) and \(j\) with \(\mu_i>\mu_j\).  Suppose that
\[
    \frac{N_i(T)}{T}\to q_i>0,
    \qquad
    \frac{N_j(T)}{T}\to q_j>0.
\]
Then almost surely
\begin{equation} \label{eq:behavior_of_Pi}
    -\frac1T
    \log
    \varPi_T(\mu_j\ge \mu_i)
    \to
    \inf_{z\in[\mu_j,\mu_i]}
    \{q_iD(\mu_i,z)+q_jD(\mu_j,z)\}.
\end{equation}
\end{lemma}

\begin{proof}
Write
\[
    \eta_i=s(\mu_i),
    \qquad
    \eta_j=s(\mu_j),
    \qquad
    \eta_i>\eta_j.
\]
Recall from the proof of Lemma~\ref{lem:one-arm-laplace-simple-upper} that,
conditional on the realized history, the arm-wise posterior of
\(\eta_i=s(\mu_i)\) is
\[
    \varPi_{i,T}(\dd\eta)
    =
    \frac{
        \bar f_i(\eta)
        \exp\{\beta(\eta)S_i(T)-N_i(T)A(\beta(\eta))\}
    }{
        Z_{i,T}
    }\,\dd\eta,
\]
where
\[
    Z_{i,T}
    :=
    \int_{\mathcal H}
        \bar f_i(u)
        \exp\{\beta(u)S_i(T)-N_i(T)A(\beta(u))\}
    \,\dd u .
\]
Because the prior is a product prior and the likelihood factors arm by arm, the
full posterior factorizes:
\[
    \varPi_T(\dd\eta_1,\ldots,\dd\eta_K)
    =
    \bigotimes_{\ell=1}^K \varPi_{\ell,T}(\dd\eta_\ell).
\]
Consequently, let \(U\sim\varPi_{i,T}\) and \(V\sim\varPi_{j,T}\) denote the posterior
information coordinate qualities (arm parameter) of arms \(i\) and \(j\), then \(U\) and \(V\)
are independent under the posterior, and for Borel sets \(B,C\subseteq\mathcal H\),
\[
    \varPi_T(U\in B,\ V\in C)
    =
    \varPi_{i,T}(B)\varPi_{j,T}(C).
\]
In particular,
\[
    \varPi_T(\mu_j\ge \mu_i)
    =
    \varPi_T(\eta_j\ge \eta_i)
    =
    \varPi_T(V\ge U).
\]
Define
\[
    I_i(u):=D(\mu_i,s^{-1}(u)),
    \qquad
    I_j(v):=D(\mu_j,s^{-1}(v)),
\]
and
\[
    C_{ij}
    :=
    \inf_{w\in[\eta_j,\eta_i]}
    \{q_iI_i(w)+q_jI_j(w)\}.
\]
The one-arm bounds on $\log \varPi_{i,n}$ of Lemma~\ref{lem:one-arm-laplace-simple-upper} hold
along every subsequence \(n\to\infty\) on the arm-wise strong-law event.
Therefore they also hold along the random subsequences
\(N_i(T)\) and \(N_j(T)\).  Multiplying the one-arm exponents by
\(N_i(T)/T\to q_i\) and \(N_j(T)/T\to q_j\) gives the corresponding
\(T\)-speed bounds: For intervals \(B,C\),
\[
\limsup_{T\to\infty}
\frac1T\log\{\varPi_{i,T}(B)\varPi_{j,T}(C)\}
\le
-q_i\inf_{u\in B}I_i(u)-q_j\inf_{v\in C}I_j(v)
\]
when \(B,C\) are compact, and the corresponding lower bound when \(B,C\) are
open.

We first prove the lower bound for \eqref{eq:behavior_of_Pi}.  Let \(w^\star\in[\eta_j,\eta_i]\) be a minimizer of
\[
    w\mapsto q_iI_i(w)+q_jI_j(w).
\]
If \(w^\star\in(\eta_j,\eta_i)\), choose \(\delta>0\) so small that
\[
    (w^\star-\delta,w^\star)
    \times
    (w^\star,w^\star+\delta)
    \subseteq
    \{(u,v):v\ge u\}.
\]
Then, by posterior product decomposition and the open-interval lower bounds in
Lemma~\ref{lem:one-arm-laplace-simple-upper},
\[
\begin{aligned}
    \liminf_{T\to\infty}
    \frac1T\log \varPi_T(V\ge U)
    &\ge
    -q_i\inf_{u\in(w^\star-\delta,w^\star)}I_i(u)
    -q_j\inf_{v\in(w^\star,w^\star+\delta)}I_j(v).
\end{aligned}
\]
Letting \(\delta\downarrow0\) gives
\[
    \liminf_{T\to\infty}
    \frac1T\log \varPi_T(V\ge U)
    \ge
    -C_{ij}.
\]
If \(w^\star=\eta_j\), use instead the rectangle $U\in(\eta_j-\delta,\eta_j)$ and $V\in(\eta_j,\eta_j+\delta)$.
Since
\[
    \inf_{u\in(\eta_j-\delta,\eta_j)} I_i(u)\to I_i(\eta_j),
    \qquad
    \inf_{v\in(\eta_j,\eta_j+\delta)} I_j(v)\to I_j(\eta_j)=0,
\]
the same lower bound follows.  

\noindent Similarly if \(w^\star=\eta_i\), use $U\in(\eta_i-\delta,\eta_i)$ and $V\in(\eta_i,\eta_i+\delta)$.
Then
\[
    \inf_{u\in(\eta_i-\delta,\eta_i)} I_i(u)\to I_i(\eta_i)=0,
    \qquad
    \inf_{v\in(\eta_i,\eta_i+\delta)} I_j(v)\to I_j(\eta_i),
\]
and the same lower bound again follows.  Thus
\[
    \liminf_{T\to\infty}
    \frac1T\log \varPi_T(V\ge U)
    \ge
    -C_{ij}.
\]

For the upper bound, decompose
\[
    \{V\ge U\}
    \subseteq
    E_-\cup E_0\cup E_+,
\]
where $E_-:=\{U\le \eta_j\}$, $E_+:=\{V\ge \eta_i\}$ and $E_0:=\{\eta_j\le U\le V\le \eta_i\}$.

The one-sided estimates in Lemma~\ref{lem:one-arm-laplace-simple-upper} give
\[
    \limsup_{T\to\infty}
    \frac1T\log\varPi_T(E_-)
    \le
    -q_i I_i(\eta_j)
    =
    -q_iD(\mu_i,\mu_j),
\]
and
\[
    \limsup_{T\to\infty}
    \frac1T\log\varPi_T(E_+)
    \le
    -q_j I_j(\eta_i)
    =
    -q_jD(\mu_j,\mu_i).
\]
Since \(C_{ij}\le q_iD(\mu_i,\mu_j)\), by taking \(w=\eta_j\), and
\(C_{ij}\le q_jD(\mu_j,\mu_i)\), by taking \(w=\eta_i\), both endpoint events have
exponent at least \(C_{ij}\).

It remains to control \(E_0\).  The set $K_0:=\{(u,v):\eta_j\le u\le v\le \eta_i\}$ is compact.  Fix \(\tau>0\).  By continuity of $(u,v)\mapsto q_iI_i(u)+q_jI_j(v)$
on the compact square \([\eta_j,\eta_i]^2\), there exist finitely many compact
rectangles
\[
    R_m=I_m\times J_m,\qquad m=1,\ldots,M,
\]
whose union covers \(K_0\), such that each \(R_m\subseteq[\eta_j,\eta_i]^2\) and
\[
    \inf_{(u,v)\in R_m}
    \{q_iI_i(u)+q_jI_j(v)\}
    \ge
    \inf_{(u,v)\in K_0}
    \{q_iI_i(u)+q_jI_j(v)\}
    -\tau .
\]
By posterior product decomposition,
\[
    \varPi_T(E_0)
    \le
    \sum_{m=1}^M
    \varPi_{i,T}(I_m)\varPi_{j,T}(J_m),
\]
where \(\varPi_{i,T}\) and \(\varPi_{j,T}\) are the corresponding arm-wise posterior
marginals.  Applying the compact-interval upper bound from
Lemma~\ref{lem:one-arm-laplace-simple-upper} along the subsequences
\(N_i(T)\) and \(N_j(T)\), we get, for each \(m\),
\[
\begin{aligned}
    \limsup_{T\to\infty}
    \frac1T\log\{\varPi_{i,T}(I_m)\varPi_{j,T}(J_m)\}
    &\le
    -q_i\inf_{u\in I_m}I_i(u)
    -q_j\inf_{v\in J_m}I_j(v)  \\
    &=
    -\inf_{(u,v)\in R_m}
    \{q_iI_i(u)+q_jI_j(v)\}.
\end{aligned}
\]
Since the cover is finite, the union bound yields
\[
\begin{aligned}
    \limsup_{T\to\infty}
    \frac1T\log\varPi_T(E_0)
    &\le
    -
    \inf_{(u,v)\in K_0}
    \{q_iI_i(u)+q_jI_j(v)\}
    +\tau .
\end{aligned}
\]
Letting \(\tau\downarrow0\), we obtain
\[
    \limsup_{T\to\infty}
    \frac1T\log\varPi_T(E_0)
    \le
    -
    \inf_{\eta_j\le u\le v\le \eta_i}
    \{q_iI_i(u)+q_jI_j(v)\}.
\]

We now reduce this infimum to the diagonal ($u=v$). Note the derivative $$
\frac{\partial }{\partial z} D(\mu,z) = \theta'(z)\left[ A'(\theta(z))-A'(\theta(\mu)) \right],
$$
is negative for $z<\mu$ and positive for $z>\mu$. Hence, on the interval
\([\eta_j,\eta_i]\), the map \(u\mapsto I_i(u)\) is non-increasing whereas \(v\mapsto I_j(v)\) is non-decreasing.
Therefore, if \(u\le v\), then for every \(w\in[u,v]\), $I_i(w)\le I_i(u)$ and $I_j(w)\le I_j(v)$.

Thus every feasible off-diagonal point can be moved to a diagonal point with no
larger objective.  Hence
\[
    \inf_{\eta_j\le u\le v\le \eta_i}
    \{q_iI_i(u)+q_jI_j(v)\}
    =
    \inf_{w\in[\eta_j,\eta_i]}
    \{q_iI_i(w)+q_jI_j(w)\}
    =
    C_{ij}.
\]
Combining the bounds for \(E_-\), \(E_0\), and \(E_+\) by a union bound gives
\[
    \limsup_{T\to\infty}
    \frac1T\log \varPi_T(V\ge U)
    \le
    -C_{ij}.
\]
Together with the lower bound,
\[
    -\frac1T\log\varPi_T(V\ge U)\to C_{ij}.
\]
Returning to \(z=s^{-1}(w)\) proves the claimed formula.

\end{proof}

\subsubsection{Asymptotic Behavior of the Certificate $R_T$}

Recall we defined
\begin{equation}
    C^\dagger(\mu)
    :=
    \min_{j\neq b^\star}
    \inf_{z\in[\mu_j,\mu_{b^\star}]}
    \left\{
        p_{b^\star}^\dagger(\eta) D(\mu_{b^\star},z)
        +
        p_j^\dagger(\eta) D(\mu_j,z)
    \right\}.
\end{equation}

\begin{lemma}[Terminal Certificate Limit]
\label{lem:RT-limit-simple-upper}
Under Algorithm~\ref{alg:icpgws-simple-upper},
\[
    \frac{R_T}{T}
    \to
    C^\dagger(\mu)
    \qquad\text{a.s.}
\]
\end{lemma}

\begin{proof}
By Lemma~\ref{lem:center-consistency-simple-upper}, $\widehat b_T=b^\star$ eventually almost surely.  By Lemma~\ref{lem:sampling-proportions-simple-upper},
\[
    \frac{N_i(T)}{T}\to p_i^\dagger(\eta)
    \qquad\text{for every }i.
\]
Since \(p_i^\dagger(\eta)>0\) for every arm, applying
Lemma~\ref{lem:posterior-ordering-simple-upper} with \(i=b^\star\) and each
\(j\neq b^\star\) gives
\[
    -\frac1T
    \log
    \varPi_T(\mu_j\ge \mu_{b^\star})
    \to
    \inf_{z\in[\mu_j,\mu_{b^\star}]}
    \left\{
        p_{b^\star}^\dagger(\eta) D(\mu_{b^\star},z)
        +
        p_j^\dagger(\eta) D(\mu_j,z)
    \right\}.
\]
Taking the minimum over the finite set \(j\neq b^\star\) yields
\[
    \frac{R_T}{T}
    =
    \min_{j\neq b^\star}
    \left\{
        -\frac1T\log
        \varPi_T(\mu_j\ge \mu_{b^\star})
    \right\}
    \to
    C^\dagger(\mu). 
\]

\end{proof}

\subsubsection{Lower Tail of the Limiting Certificate.}

We now prove that \(C^\dagger(\mu)\) has the same small lower tail as
\(\Gamma_s^2/8\).

\begin{lemma}[Top-Two Gap and Triple-Gap Expansions]
\label{lem:top-gap-triple-simple-upper}
Under Assumption~\ref{ass:regular-info-prior},
\[
    \mathbb P(\Gamma_s\le\varepsilon)
    =
    \kappa\varepsilon+o(\varepsilon),
    \qquad
    \varepsilon\downarrow0.
\]
Moreover, define 
$\Gamma_{3,s}:=\eta_{(1)}-\eta_{(3)}$ 
for \(K\ge3\), and \(\Gamma_{3,s}:=+\infty\) for \(K=2\), then
\[
    \mathbb P(\Gamma_{3,s}\le\varepsilon)
    =
    O(\varepsilon^2).
\]
\end{lemma}

\begin{proof}
For the top-two gap expansion, condition on arm \(i\) having information coordinate
value \(u\).  The event that arm \(i\) is best and that the top-two gap is at most
\(\varepsilon\) has conditional probability
\[
    \prod_{k\neq i}\bar F_k(u)
    -
    \prod_{k\neq i}\bar F_k(u-\varepsilon),
\]
where \(\bar F_k(t)\) is understood to be \(0\) below the left endpoint of
\(\mathcal H\) and \(1\) above the right endpoint.  Hence
\[
    \mathbb P(\Gamma_s\le\varepsilon)
    =
    \sum_{i=1}^K
    \int_{\mathcal H}
        \bar f_i(u)
        \left[
            \prod_{k\neq i}\bar F_k(u)
            -
            \prod_{k\neq i}\bar F_k(u-\varepsilon)
        \right]\dd u.
\]
Let
\[
    G_i(u):=\prod_{k\neq i}\bar F_k(u).
\]
Then
\[
    G_i'(u)
    =
    \sum_{j\neq i}
    \bar f_j(u)\prod_{k\neq i,j}\bar F_k(u)
\]
where the derivative exists, and \(G_i\) is Lipschitz because the densities are
bounded.  Therefore
\[
    \frac{G_i(u)-G_i(u-\varepsilon)}{\varepsilon}
    \to
    G_i'(u)
\]
for almost every \(u\), and the quotient is uniformly bounded.  Dominated
convergence gives
\[
\begin{aligned}
    \frac{\mathbb P(\Gamma_s\le\varepsilon)}{\varepsilon}
    \to
    \sum_{i=1}^K
    \sum_{j\neq i}
    \int_{\mathcal H}
        \bar f_i(u)\bar f_j(u)
        \prod_{k\neq i,j}\bar F_k(u)\,\dd u  
    =\kappa .
\end{aligned}
\]

For the triple-gap bound, the case \(K=2\) is trivial.  Assume \(K\ge3\), and
let
\[
    M_{\bar{f}}:=\max_{i\in[K]}\|\bar f_i\|_\infty<\infty.
\]
If \(\Gamma_{3,s}\le\varepsilon\), then there exist three distinct arms
\(i,j,k\) whose information coordinate values all lie in an interval of length
\(\varepsilon\) below the largest of the three.  Hence, by a union bound over
ordered triples,
\[
\begin{aligned}
    \mathbb P(\Gamma_{3,s}\le\varepsilon)
    &\le
    \sum_{i,j,k\ {\rm distinct}}
    \int_{\mathcal H}
        \bar f_i(u)
        \left[
            \int_{u-\varepsilon}^{u}\bar f_j(v)\,\dd v
        \right]
        \left[
            \int_{u-\varepsilon}^{u}\bar f_k(w)\,\dd w
        \right]\dd u  \\
    &\le
    K(K-1)(K-2)M_{\bar{f}}^2\varepsilon^2. 
\end{aligned}
\]

\end{proof}

Define the ordered top-pair tie manifold
\[
    \mathcal T^\uparrow
    :=
    \biguplus_{i\neq j}
    \left\{
        (i,j,u,\zeta_{-ij}):
        u\in\mathcal H,\ \zeta_k\in\mathcal H,\ \zeta_k<u\ \forall k\neq i,j
    \right\}.
\]
Here \(i\) denotes the best arm and \(j\) denotes the second-best arm.  For a
parameter vector \(\eta\) with unique best and second-best arms, define
\[
    \xi^\uparrow(\eta)
    :=
    \bigl(b^\star,b^\star_2,\eta_{b^\star},(\eta_k)_{k\neq b^\star,b^\star_2}\bigr),
\]
where $b^\star$ is the best arm index, and \(b^\star_2\) is the second-best arm index. Note that it does not include $\eta_{b^\star_2}$, which is meant to be recovered by $\Gamma_s$.

On the component \((i,j)\), define the finite measure
\[
    \lambda(\dd\xi)
    :=
    \bar f_i(u)\bar f_j(u)
    \prod_{k\neq i,j}\bar f_k(\zeta_k)\,\dd u\,\dd \zeta_{-ij}.
\]
Then
    $\lambda(\mathcal T^\uparrow)=\kappa$.

\begin{lemma}[Compact Localization of Top-Two Gap Density]
\label{lem:compact-localization-simple-upper}
For every \(\rho>0\), there exists a compact set
\[
    \mathcal K_\rho\subset\mathcal T^\uparrow
\]
such that, for every fixed \(R<\infty\),
\[
    \mathbb P\{\Gamma_s\le R\varepsilon,\,
    \xi^\uparrow(\eta)\in\mathcal K_\rho\}
    \ge
    (\kappa-\rho)R\varepsilon+o(\varepsilon),
\]
and
\[
    \mathbb P\{\Gamma_s\le R\varepsilon,\,
    \xi^\uparrow(\eta)\notin\mathcal K_\rho\}
    \le
    \rho R\varepsilon+o(\varepsilon),
\]
as \(\varepsilon\downarrow0\).
\end{lemma}

\begin{proof}
Since \(\mathcal T^\uparrow\) is a finite disjoint union of open subsets of
Euclidean spaces and \(\lambda\) is a finite Borel measure with density with
respect to Lebesgue measure on each component, \(\lambda\) is a finite Radon
measure.  In particular, it is inner regular.  Choose a compact set $\mathcal K_\rho\subset\mathcal T^\uparrow$
such that $\lambda(\mathcal K_\rho)\ge \kappa-\rho$.

Because \(\mathcal K_\rho\) is compact and lies inside the open ordered tie
manifold, there exists \(\ell_\rho>0\) such that, for every
\((i,j,u,\zeta_{-ij})\in\mathcal K_\rho\), $ \zeta_k\le u-\ell_\rho$ for all $k\neq i,j$.

For all sufficiently small \(\varepsilon\), uniformly over
\(0\le r\le R\varepsilon\), the second-arm value \(u-r\) remains above all
background values on \(\mathcal K_\rho\).  Hence, restricted to
\(\mathcal K_\rho\),
\[
\begin{aligned}
    &\mathbb P\{\Gamma_s\le R\varepsilon,\,
    \xi^\uparrow(\eta)\in\mathcal K_\rho\}                         \\
    &\quad=
    \sum_{i\neq j}
    \int
        \mathbbm{1}_{\{(i,j,u,\zeta_{-ij})\in\mathcal K_\rho\}}
        \bar f_i(u)
        \prod_{k\neq i,j}\bar f_k(\zeta_k)
        \left[
            \int_{u-R\varepsilon}^{u}\bar f_j(v)\,\dd v
        \right]
    \dd u\,\dd \zeta_{-ij}.
\end{aligned}
\]
By continuity of \(\bar f_j\) on the compact projection of
\(\mathcal K_\rho\),
\[
    \int_{u-R\varepsilon}^{u}\bar f_j(v)\,\dd v
    =
    R\varepsilon\,\bar f_j(u)+o(\varepsilon)
\]
uniformly over \(\mathcal K_\rho\).  Therefore
\[
    \mathbb P\{\Gamma_s\le R\varepsilon,\,
    \xi^\uparrow(\eta)\in\mathcal K_\rho\}
    =
    R\varepsilon\,\lambda(\mathcal K_\rho)+o(\varepsilon)
    \ge
    (\kappa-\rho)R\varepsilon+o(\varepsilon).
\]
The complement bound follows by subtracting this lower bound from the total
top-two gap expansion:
\[
    \mathbb P(\Gamma_s\le R\varepsilon)
    =
    \kappa R\varepsilon+o(\varepsilon). 
\]

\end{proof}

\begin{lemma}[Small Certificate Implies Small Information Gap]
\label{lem:nonlocal-exclusion-simple-upper}
There exists a constant \(R^\star<\infty\), depending only on \(K\) and
\(\widetilde{c}\), such that
\[
    C^\dagger(\mu)\le u
    \quad\Longrightarrow\quad
    \Gamma_s\le R^\star\sqrt u.
\]
\end{lemma}

\begin{proof}
Fix \(j\neq b^\star\), write
$\Delta:=\Delta_{j,s}=\eta_{b^\star}-\eta_j$,
    and
    $a:=w_j^\dagger(\eta)$.
For \(z\in[\mu_j,\mu_{b^\star}]\), write
    $r:=s(z)-\eta_j\in[0,\Delta]$.
By \eqref{eq:ckl-under-simple-upper},
\[
\begin{aligned}
    D(\mu_{b^\star},z)+aD(\mu_j,z)
    &\ge
    \widetilde{c}\{(\Delta-r)^2+a r^2\}.
\end{aligned}
\]
If \(a=1\), then
\[
    \inf_{0\le r\le\Delta}
    \widetilde{c}\{(\Delta-r)^2+r^2\}
    =
    \frac{\widetilde{c}}{2}\Delta^2
    \ge
    \frac{\widetilde{c}}{2}\Gamma_s^2.
\]
If \(a<1\), then by \eqref{eq:simple-limiting-weights},
    $a=\frac{\Gamma_s^2}{4\widetilde{c}\,\Delta^2}$.
Thus
\[
\begin{aligned}
    \inf_{0\le r\le\Delta}
    \widetilde{c}\{(\Delta-r)^2+a r^2\}
    =
    \widetilde{c}\frac{a}{1+a}\Delta^2  
    \ge
    \frac{\widetilde{c} a}{2}\Delta^2
    =
    \frac{\Gamma_s^2}{8},
\end{aligned}
\]
where we used \(a<1\).  Therefore, for every challenger \(j\),
\[
    \inf_{z\in[\mu_j,\mu_{b^\star}]}
    \{D(\mu_{b^\star},z)+w_j^\dagger(\eta) D(\mu_j,z)\}
    \ge
    c_{\rm gap}\Gamma_s^2,
\]
where $c_{\rm gap}:=
    \min\left\{\frac{\widetilde{c}}{2},\frac18\right\}>0$.

Since \(W^\dagger(\eta)\le K\), \eqref{eq:Cdagger-unnormalized-simple-upper} gives $C^\dagger(\mu)
    \ge
    \frac{c_{\rm gap}}{K}\Gamma_s^2$.
Hence \(C^\dagger(\mu)\le u\) implies
\[
    \Gamma_s\le \sqrt{\frac K{c_{\rm gap}}}\sqrt u.
\]
Taking $R^\star:=\sqrt{K/c_{\rm gap}}$ gives the desired result.

\end{proof}

\begin{lemma}[Local Expansion of the Limiting Certificate]
\label{lem:Cdagger-local-expansion-simple-upper}
Fix a compact set \(\mathcal K\subset\mathcal T^\uparrow\).  Uniformly over
parameter vectors satisfying
    $\xi^\uparrow(\eta)\in\mathcal K$
    and
    $\Gamma_s\downarrow0$,
we have
\[
    C^\dagger(\mu)
    =
    \frac{\Gamma_s^2}{8}\{1+o(1)\}.
\]
\end{lemma}

\begin{proof}
Let \(b^\star_2\) denote the second-best arm and put $\gamma:=\Gamma_s$.

Because \(\mathcal K\) is compactly contained in
\(\mathcal T^\uparrow\), there exists \(\ell>0\) such that, whenever
\(\xi^\uparrow(\eta)\in\mathcal K\),
\[
    \eta_{b^\star}-\eta_k\ge \ell,
    \qquad k\notin\{b^\star,b^\star_2\},
\]
for all sufficiently small \(\gamma\).  Also, all relevant
information coordinate values remain in a compact subset of \(\mathcal H\).

The KL divergence has the uniform local Fisher coordinate expansion
\begin{equation}
\label{eq:local-kl-simple-upper}
    D(s^{-1}(\eta),s^{-1}(v))
    =
    \frac12(\eta-v)^2+O(|\eta-v|^3),
\end{equation}
for \(\eta,v\) in any compact subset of \(\mathcal H\).  This follows from a
Taylor expansion of the KL divergence and the definition of \(s\).

For the top two arms, $w_{b^\star}^\dagger(\eta)=1$ and $w_{b^\star_2}^\dagger(\eta)
    =
    \min\left\{
        1,\frac{\gamma^2}{4\widetilde{c}\,\gamma^2}
    \right\}
    =
    1,$
Hence using
\eqref{eq:local-kl-simple-upper},
\[
\begin{aligned}
    &\inf_{z\in[\mu_{b^\star_2},\mu_{b^\star}]}
    \{D(\mu_{b^\star},z)+D(\mu_{b^\star_2},z)\}                      \\
    &\qquad=
    \inf_{w\in[\eta_{b^\star}-\gamma,\eta_{b^\star}]}
    \left\{
        \frac12(\eta_{b^\star}-w)^2
        +
        \frac12(w-\eta_{b^\star}+\gamma)^2
    \right\}
    +
    o(\gamma^2)                                             \\
    &\qquad=
    \frac{\gamma^2}{4}+o(\gamma^2).
\end{aligned}
\]

For any other challenger \(k\notin\{b^\star,b^\star_2\}\), its information gap
\(\Delta_{k,s}:=\eta_{b^\star}-\eta_k\) is bounded below by \(\ell\).  Therefore $w_k^\dagger(\eta)
    =
    \frac{\gamma^2}{4\widetilde{c}\,\Delta_{k,s}^2}
    =
    O(\gamma^2)$,
for all sufficiently small \(\gamma\).  Hence
\[
    W^\dagger(\eta)
    =
    1+w_{b^\star_2}^\dagger(\eta)+\sum_{k\notin\{b^\star,b^\star_2\}}w_k^\dagger(\eta)
    =
    2+O(\gamma^2).
\]
Furthermore, by the same coercivity calculation as in
Lemma~\ref{lem:nonlocal-exclusion-simple-upper},
\[
\begin{aligned}
    \inf_{z\in[\mu_k,\mu_{b^\star}]}
    \{D(\mu_{b^\star},z)+w_k^\dagger(\eta) D(\mu_k,z)\}
    \ge
    \widetilde{c}
    \frac{w_k^\dagger(\eta)}{1+w_k^\dagger(\eta)}
    \Delta_{k,s}^2
    =
    \frac{\gamma^2}{4}\{1+o(1)\}.
\end{aligned}
\]
Thus
\[
    \min_{j\neq b^\star}
    \inf_{z\in[\mu_j,\mu_{b^\star}]}
    \{D(\mu_{b^\star},z)+w_j^\dagger(\eta) D(\mu_j,z)\}
    =
    \frac{\gamma^2}{4}\{1+o(1)\}.
\]
Combining this with $W^\dagger(\eta)=2+o(1)$ in \eqref{eq:Cdagger-unnormalized-simple-upper} gives
\[
    C^\dagger(\mu)
    =
    \frac{\gamma^2}{8}\{1+o(1)\}.
\]
The preceding estimates are uniform on \(\mathcal K\), so the proof is complete.

\end{proof}

\begin{lemma}[Small Lower Tail of the Limiting Certificate]
\label{lem:Cdagger-tail-simple-upper}
As \(u\downarrow0\),
\[
    \mathbb P(C^\dagger(\mu)\le u)
    =
    \kappa\sqrt{8u}+o(\sqrt u).
\]
\end{lemma}

\begin{proof}
We prove matching lower and upper bounds.

For the lower bound, fix \(\rho>0\) and \(\tau\in(0,1)\).  Choose
\(\mathcal K_\rho\) as in
Lemma~\ref{lem:compact-localization-simple-upper}.  By the uniform local
expansion in Lemma~\ref{lem:Cdagger-local-expansion-simple-upper}, for all
sufficiently small \(u\),
\[
    \left\{
        \Gamma_s\le (1-\tau)\sqrt{8u},
        \ \xi^\uparrow(\eta)\in\mathcal K_\rho
    \right\}
    \subseteq
    \{C^\dagger(\mu)\le u\}.
\]
Indeed, on \(\mathcal K_\rho\),
\[
    C^\dagger(\mu)
    =
    \frac{\Gamma_s^2}{8}\{1+o(1)\},
\]
uniformly as \(\Gamma_s\downarrow0\), and the factor
\((1-\tau)^2\) leaves enough slack.  Therefore, using
Lemma~\ref{lem:compact-localization-simple-upper} with
\(\varepsilon=\sqrt u\) and \(R=(1-\tau)\sqrt8\),
\[
\begin{aligned}
    \mathbb P(C^\dagger(\mu)\le u)
    &\ge
    \mathbb P\left(
        \Gamma_s\le (1-\tau)\sqrt{8u},
        \ \xi^\uparrow(\eta)\in\mathcal K_\rho
    \right)  \\
    &\ge
    (\kappa-\rho)(1-\tau)\sqrt{8u}
    +
    o(\sqrt u).
\end{aligned}
\]
Letting \(\rho\downarrow0\) and then \(\tau\downarrow0\) yields
\[
    \liminf_{u\downarrow0}
    \frac{\mathbb P(C^\dagger(\mu)\le u)}{\sqrt u}
    \ge
    \kappa\sqrt8.
\]

For the upper bound, Lemma~\ref{lem:nonlocal-exclusion-simple-upper} gives
\[
    \{C^\dagger(\mu)\le u\}
    \subseteq
    \{\Gamma_s\le R^\star\sqrt u\}.
\]
Fix \(\rho>0\), and again use the compact set \(\mathcal K_\rho\).  On
\[
    \{C^\dagger(\mu)\le u,\ \xi^\uparrow(\eta)\in\mathcal K_\rho\},
\]
Lemma~\ref{lem:nonlocal-exclusion-simple-upper} implies
\(\Gamma_s\to0\) uniformly as \(u\downarrow0\).  Hence the uniform local
expansion gives, for all sufficiently small \(u\),
\[
    C^\dagger(\mu)
    \ge
    (1-\tau)\frac{\Gamma_s^2}{8}.
\]
Consequently, on the same event,
\[
    \Gamma_s
    \le
    \sqrt{\frac{8u}{1-\tau}}
    \le
    (1+\tau)\sqrt{8u},
\]
after decreasing \(u\) if necessary and replacing \(\tau\) by a smaller slack
constant.  Therefore
\[
\begin{aligned}
    \mathbb P(C^\dagger(\mu)\le u)
    \le
    \mathbb P\{\Gamma_s\le (1+\tau)\sqrt{8u}\} 
    +
    \mathbb P\{\Gamma_s\le R^\star\sqrt u,\,
    \xi^\uparrow(\eta)\notin\mathcal K_\rho\}.
\end{aligned}
\]
By Lemmas~\ref{lem:top-gap-triple-simple-upper} and
\ref{lem:compact-localization-simple-upper},
\[
    \mathbb P(C^\dagger(\mu)\le u)
    \le
    \kappa(1+\tau)\sqrt{8u}
    +
    \rho R^\star\sqrt u
    +
    o(\sqrt u).
\]
Letting \(u\downarrow0\), then \(\rho\downarrow0\), and finally
\(\tau\downarrow0\), gives
\[
    \limsup_{u\downarrow0}
    \frac{\mathbb P(C^\dagger(\mu)\le u)}{\sqrt u}
    \le
    \kappa\sqrt8.
\]
Combining the two bounds proves the claim.

\end{proof}

\subsubsection{Lower Bound on the Calibration Threshold}

\begin{lemma}[Lower Bound on the Calibration Threshold]
\label{lem:threshold-simple-upper}
For every \(\varepsilon\in(0,1)\), there exists \(\alpha_0>0\) such that, for
every \(0<\alpha\le\alpha_0\), there exists
\(T_0(\alpha,\varepsilon)<\infty\) such that, for all
\(T\ge T_0(\alpha,\varepsilon)\),
\[
    r_{T,\alpha}
    \ge
    (1-\varepsilon)\frac{\alpha^2T}{8\kappa^2}.
\]
\end{lemma}

\begin{proof}
Set $c:=\frac{1}{8\kappa^2}$ and $u_{\alpha,\varepsilon}:=(1-\varepsilon)c\alpha^2$.
By Lemma~\ref{lem:Cdagger-tail-simple-upper},
\[
\begin{aligned}
    \mathbb P(C^\dagger(\mu)\le u_{\alpha,\varepsilon})
    &=
    \kappa\sqrt{8(1-\varepsilon)c}\,\alpha+o(\alpha)  \\
    &=
    \sqrt{1-\varepsilon}\,\alpha+o(\alpha).
\end{aligned}
\]
Therefore, for all sufficiently small \(\alpha\),
    $\mathbb P(C^\dagger(\mu)\le u_{\alpha,\varepsilon})<\alpha$.
By Lemma~\ref{lem:RT-limit-simple-upper},
\[
    \frac{R_T}{T}\to C^\dagger(\mu)
    \qquad\text{a.s.}
\]
Thus \(R_T/T\) converges to \(C^\dagger(\mu)\) in distribution.  Applying the
Portmanteau theorem to the closed set $(-\infty,u_{\alpha,\varepsilon}]$
gives
\[
    \limsup_{T\to\infty}
    \mathbb P\left(
        \frac{R_T}{T}\le u_{\alpha,\varepsilon}
    \right)
    \le
    \mathbb P(C^\dagger(\mu)\le u_{\alpha,\varepsilon})
    <
    \alpha.
\]
Hence, for all sufficiently large \(T\),
    $\mathbb P(R_T\le u_{\alpha,\varepsilon}T)<\alpha$.
By the definition of \(r_{T,\alpha}\) as the lower \(\alpha\)-quantile of
\(R_T\),
\[
    r_{T,\alpha}
    \ge
    u_{\alpha,\varepsilon}T
    =
    (1-\varepsilon)\frac{\alpha^2T}{8\kappa^2}. 
\]

\end{proof}

\subsubsection{Proof of Theorem~\ref{thm:simple-transfer-upper}.}

\begin{proof}
We first prove abstention calibration.  By the definition of the lower
\(\alpha\)-quantile,
\[
    \mathbb P(R_T<r_{T,\alpha})
    \le
    \alpha
    \le
    \mathbb P(R_T\le r_{T,\alpha}).
\]
If \(\mathbb P(R_T=r_{T,\alpha})>0\), the boundary randomization in
\eqref{eq:icpgws-terminal-rule-simple-upper} gives
\[
\begin{aligned}
    \mathbb P(\widehat a_T=?)
    =
    \mathbb P(R_T<r_{T,\alpha})
    +
    \theta_{T,\alpha}
    \mathbb P(R_T=r_{T,\alpha}) 
    =
    \alpha.
\end{aligned}
\]
If \(\mathbb P(R_T=r_{T,\alpha})=0\), then
\[
    \mathbb P(R_T<r_{T,\alpha})
    =
    \mathbb P(R_T\le r_{T,\alpha})
    =
    \alpha,
\]
by the defining property of the lower quantile and the absence of an atom at
the quantile.  Since \(\theta_{T,\alpha}=0\) in this case, the same conclusion
holds.  Therefore
\[
    \mathcal A_T (\textsc{IC-PGWS}(\alpha))=\alpha.
\]

For the undetected error probability,
Lemma~\ref{lem:posterior-certification-simple-upper} gives
\[
    \mathcal E_T (\textsc{IC-PGWS}(\alpha))
    \le
    (K-1)e^{-r_{T,\alpha}}.
\]
By Lemma~\ref{lem:threshold-simple-upper}, for every
\(\varepsilon\in(0,1)\), all sufficiently small \(\alpha\), and all sufficiently
large \(T\),
\[
    r_{T,\alpha}
    \ge
    (1-\varepsilon)\frac{\alpha^2T}{8\kappa^2}.
\]
Hence
\[
    \mathcal E_T(\textsc{IC-PGWS}(\alpha))
    \le
    (K-1)
    \exp\left\{
        -(1-\varepsilon)\frac{\alpha^2T}{8\kappa^2}
    \right\}.
\]
Taking logarithms, dividing by \(\alpha^2T\), taking
\(T\to\infty\), then \(\alpha\downarrow0\), and finally
\(\varepsilon\downarrow0\), yields
\[
    \limsup_{\alpha\downarrow0}
    \limsup_{T\to\infty}
    \frac{1}{\alpha^2T}
    \log \mathcal E_T(\textsc{IC-PGWS}(\alpha))
    \le
    -\frac{1}{8\kappa^2}. 
\]

\end{proof}

\section*{Acknowledgments}
The authors thank Junwen Yang (National University of Singapore) for helpful discussions during the initial phase of this work.

\newpage
\bibliographystyle{ims}
\bibliography{main_bib}

\end{document}